\documentclass[11pt,natbib]{article}

\usepackage{geometry}
\usepackage{amsmath}
\usepackage{amssymb}
\usepackage{dsfont}
\usepackage{url}
\usepackage{subcaption}
\usepackage{titling}
\usepackage{mathrsfs}
\usepackage{float} 
\usepackage{graphicx}
\allowdisplaybreaks
\usepackage{bibunits}
\usepackage{csquotes}
\usepackage{minitoc}
\usepackage{minitoc}
\usepackage{tocloft}
\usepackage{enumitem}
\usepackage{algorithm, algorithmicx, algpseudocode}
\cftsetindents{sec}{0em}{2.75em}
\cftsetindents{subsec}{2.75em}{3.5em}

\usepackage{amsthm}

\makeatletter
\newtheorem*{rep@theorem}{\rep@title}
\newcommand{\newreptheorem}[2]{%
\newenvironment{rep#1}[1]{%
 \def\rep@title{#2 \ref{##1}}%
 \begin{rep@theorem}}%
 {\end{rep@theorem}}}
\makeatother

\newtheorem{theorem}{Theorem}
\newreptheorem{theorem}{Theorem}
\newtheorem{lemma}{Lemma}
\newreptheorem{lemma}{Lemma}
\newtheorem*{theorem*}{Theorem}

\newtheorem{corollary}{Corollary}
\newtheorem{remark}{Remark}
\newtheorem{assumption}{Assumption}
\newtheorem{case}{Case}

\usepackage{natbib}

\usepackage{tikz}
\usetikzlibrary{shapes,decorations,arrows,calc,arrows.meta,fit,positioning}
\tikzset{
    -Latex,auto,node distance =1 cm and 1 cm,semithick,
    state/.style ={ellipse, draw, minimum width = 0.7 cm},
    point/.style = {circle, draw, inner sep=0.04cm,fill,node contents={}},
    bidirected/.style={Latex-Latex,dashed},
    el/.style = {inner sep=2pt, align=left, sloped}
}

\newtheorem*{assumption*}{\assumptionnumber}
\providecommand{\assumptionnumber}{}


\usepackage{xr}
\usepackage[hypertexnames=false]{hyperref}
\hypersetup{
           breaklinks=true,   
           colorlinks=true,   
           pdfusetitle=true,  
        }

\begin{document}

\doparttoc
\faketableofcontents
\begin{bibunit}[custom]

\title{Holdout sets for safe predictive model updating}

\author{Sami Haidar-Wehbe$^{1,\perp}$, Samuel R. Emerson$^{1,\perp}$, Louis J. M. Aslett$^{1,2}$, James Liley$^{1}$ \\[4pt] 
\textit{1: Department of Mathematical Sciences, 
Durham University,
UK} \\
\textit{2: Alan Turing Institute, London, UK} \\
\textit{$\perp$: Equal contribution} \\[8pt]
Correspondence to \{louis.aslett, james.liley\}@durham.ac.uk}

\maketitle

\begin{abstract}
{Predictive risk scores for adverse outcomes are increasingly crucial in guiding health interventions}. {Such scores may need to be periodically updated due to change in the distributions they model.} However, directly updating risk scores used to guide intervention can lead to biased risk estimates. 
To address this, we propose updating using a `holdout set' --- a subset of the population that does not receive interventions guided by the risk score. Balancing the holdout set size is essential to ensure good performance of the updated risk score whilst minimising the number of held out samples.
{We prove that this approach reduces adverse outcome frequency to an asymptotically optimal level and argue that often there is no competitive alternative.}
We describe conditions under which an optimal holdout size (OHS) can be readily identified, and introduce parametric and semi-parametric algorithms for OHS estimation.
{
We apply our methods to the ASPRE risk score for pre-eclampsia to recommend a plan for updating it in the presence of change in the underlying data distribution. We show that, in order to minimise the number of pre-eclampsia cases over time, this is best achieved using a holdout set of around 10,000 individuals.
}

\end{abstract}

\section{Introduction}
\label{sec:introduction}

Risk scores estimate the probability of an event $Y$ given predictors $X$. Their use has become routine in medical practice~\citep{topol2019}, where $Y$ is typically a binary random variable representing an adverse event incidence and $X$ various clinical observations. Once calculated, risk scores may be used to guide interventions, perhaps modifying $X$, with the aim of decreasing the probability of an adverse event. {
The ASPRE score~\citep{akolekar13}, with which we will be working, evaluates risk of pre-eclampsia (\textsc{pre}), a hypertensive complication of pregnancy, using predictors derived from ultrasound scans in early pregnancy, and can be used to prioritise prescription of aspirin (and other interventions) to at-risk pregnancies}. 

Risk scores are typically developed by regressing observations of $Y$ on $X$. {Should the conditional distribution $\left(Y|X\right)$ subsequently change,  or `drift', then risk estimates may become biased~\citep{tsymbal04}}. This can happen naturally over time, meaning that risk scores typically need to be updated periodically to maintain accuracy. {In the complex settings in which such scores are typically used, interventions in response to risk scores may be multifaceted, and hence impractical to record.}

Updating of the risk score will involve obtaining new observations of $(X,Y)$, but crucially the distribution of $(X,Y)$ may also have changed due to the effect of the risk score itself: that is, high predicted risk of an adverse event may trigger intervention to reduce that risk. The effect of such interventions may be impractical to infer or measure, and indeed the fact that intervention took place may be unrecorded. {In the example above, this means individuals prescribed aspirin in response to higher ASPRE scores should have lower \textsc{PRE} risk than they would have if ASPRE was not used}. Should a new risk score be fitted to observed $(X,Y)$, the effect of hypertension on risk would be underestimated. This bias is worsened by heavier intervention resulting in risk scores becoming `victims of their own success’~\citep{lenert19}. This framework of directly updating a risk score on an `intervened' population has been termed `repeated risk minimisation'~\citep{perdomo20} in the context when such bias is accounted for, or `na\"{i}ve updating'~\citep{liley21updating} when it is not.

In \citet{liley21updating} we briefly noted that this bias could be avoided by splitting the population on which the score can be used into an `intervention' set and a `holdout' set, with an updated model trained on the latter. In this work, we formally develop this proposal for practical use in real-world predictive risk score updating, prove its suitability, {and apply it to the ASPRE score}. In particular, we address a vital tension in the choice of an optimal holdout size (OHS) for the holdout set: for the risk score to be accurate, the holdout set should not be too small; but any samples in the holdout set will not benefit from risk scores, so nor should it be too large. The holdout set must be actively generated; it is not sufficient to simply update a model using samples who received a risk score but were untreated.

We begin by introducing a motivating example in Section~\ref{sec:motivating_example}, reviewing relevant literature in Section~\ref{sec:litrev}.
Our first question is whether a hold-out set is worth the cost, as opposed to simply continuing to use the existing score with degraded performance, or updating na\"{i}vely.
We set out the problem and notation precisely in Section~\ref{sec:theory}.
In Section~\ref{sec:holdout_set_motivation}, we then develop theory which proves that under certain simplifying assumptions, as long as drift and intervention effects occur, the cost of holding out samples is generally justified and that the holdout set approach outperforms common alternatives.
We then turn attention to the problem of selecting holdout set size in Section \ref{sec:cost_specification}, constructing an optimisation problem to find an OHS.
Therein we also set out why several apparent alternatives are not competitive.
In Section~\ref{sec:ohs_estimation}, we describe two algorithms for OHS estimation, using a parametric model, and using Bayesian emulation. 
In Section~\ref{sec:simulations}, we support our findings with numerical demonstrations and resolve our motivating example by applying our methods to a risk score for pre-eclampsia (\textsc{pre}) to estimate an OHS for updating it.

\subsection{Motivating example}
\label{sec:motivating_example}

{We consider the ASPRE score~\citep{akolekar13} for evaluating risk of pre-eclampsia (\textsc{pre}), a hypertensive complication of pregnancy, on the basis of predictors derived from ultrasound scans in early pregnancy (we will not differentiate early- and late-stage \textsc{pre})}. Although treatable, \textsc{pre} confers a serious risk to both the fetus and the mother. The risk of \textsc{pre} is lowered by treatment with aspirin through the second and third trimesters~\citep{rolnik17}, but aspirin therapy itself confers a slight risk, contraindicating universal treatment, and suggesting prescription of aspirin only if the risk of \textsc{pre} is sufficiently high or other indications are present~\citep{acog16}. The ASPRE score was developed to aid clinicians in estimating \textsc{pre} risk and has been shown to be useful in prioritising patients for aspirin therapy~\citep{rolnik17b}. 
{Our aim is to develop a plan for updating the ASPRE predictive risk score, in such a way as to minimise the expected number of \textsc{pre} cases per unit time. To update the ASPRE model, we presume covariate and outcome data is available on a set of pregnancies in a given time period.
}

Due to changing population demographics, we anticipate that the influence of risk factors is likely to change over time, and the ASPRE score will predict individual \textsc{pre} risk less accurately over time {(as compared to an optimal predictor), reducing the usefulness of the score in identifying high-risk pregnancies and the capacity of healthcare practitioners to anticipate and avoid~\textsc{pre} cases. }

{
We illustrate why updating the ASPRE model is difficult using a somewhat exaggerated effect. Let us informally denote \textsc{PRE} incidence as $Y$, and ASPRE score covariates as $X$, and an indicator $A$ for whether a patient is treated with aspirin. Suppose that we consider patients with particular set of covariates $X=x$, and that under normal healthcare (in the absence of an ASPRE score) such pregnancies have a~\textsc{pre} risk $P(Y|X=x)=60\%$ (where some individuals are treated pre-emptively with aspirin).  Suppose that if we were to alter clinical care to treat all (or almost all) of these patients with aspirin, the~\textsc{pre} risk would drop to 2\% (approximately the baseline rate). We write this as a~\emph{counterfactual} $P_{A\gets 1}(Y|X=x)=2\%$: that is, we force treatment $A=1$.
}

{
Should we `naively' re-fit the ASPRE score, we would learn that individuals with covariates $x$ had a risk of approximately $2\%$. Such a risk score woud not generally warrant healthcare practitioners to focus additional attention on such patients, and may even serve as a false reassurance. Patients with covariates $x$ now face a risk of (say) 70\% of \textsc{pre}, which is worse than having no risk score at all. 
}
%

We propose avoiding this problem by maintaining a holdout set. For patients in this set, no ASPRE score would be calculated at first scan, and treatment would be according to best practice in the absence of a risk score. An updated ASPRE score can then be fitted to data from these patients. Patients in this holdout set go without the benefit of the ASPRE score, leading to a less accurate allocation of prophylactic treatment (aspirin) and consequently a higher risk of \textsc{pre}~\citep{rolnik17b}. However, an inappropriately small holdout set would lead to an inaccurate updated model, reducing the benefit of future use of the score. 
{We make a simplifying assumption that ASPRE is only used for decisions on aspirin therapy, although in practice it could be used more generally}.

\subsection{Review of related work}
\label{sec:litrev}

{In applied statistics, a vast amount of effort has been expended designing predictive scores in healthcare.} Widespread collection of electronic health records has spurred development of new diagnostic and prognostic risk scores~\citep{cook15, liley21medRxiv}, which can allow detection of patterns too complex for humans to discover. Examples of such scores in widespread use include: EuroSCORE II, which predicts mortality risk at hospital discharge following cardiac surgery~\citep{nashef12}; and the STS risk score from the United States, predicting risk of postoperative mortality~\citep{shahian18}. Many such scores have demonstrable efficacy in clinical trials and in-vivo~\citep{chalmers13, wallace14,hippisley17}.

An important general concern with these scores is continued accuracy of predictions. A 2011 review found that risk scores for hospital readmission perform poorly and highlighted issues with design of their trials \citep{kansagara11}. More recently, an analysis of a sepsis response score used during the COVID pandemic found increasing risk overestimation over time~\citep{finlayson20}. Various efforts have been made to standardise procedures in risk score estimation to address these issues~\citep{collins15}. {A critical practical aspect of development of such scores is to determine how they must be updated: indeed, we claim that whenever a risk score is deployed for use, a plan should be made for its continued development, and it is to this general applied area that our work contributes}.

Several algorithms have been developed to update models with new data in the presence of drift~\citep{lu18}, which ideally leads to the best possible model performance after every update. However, adaptation of model updating to avoid na\"{i}ve updating-induced bias requires explicit causal reasoning~\citep{sperrin19} and often further data collection~\citep{liley21stacked}. In a seminal paper, \cite{perdomo20} analyse asymptotic behaviour of repeated na\"{i}ve updating, giving necessary and sufficient conditions under which successive predictions end up converging to a stable setting where they essentially predict their own effect. Other approaches to optimise a general loss function by modulating parameters of the risk score are developed in~\cite{mendler20,drusvyatskiy20,li21} and~\cite{izzo21}. These approaches seek to minimise a `performative' loss to the population in the presence of an arbitrary risk score, whilst our approach seeks to target risk scores which reliably estimate the same quantity, namely $P(Y \mid X)$ in a `native' system prior to risk score deployment. Our approach is well-suited to settings where the performative loss is essentially intractable, requiring cost estimates of risk scores only in limited settings.

We note that `stability' is not necessarily desirable in terms of the distribution of interventions: in the QRISK3 setting, if an individual is at untreated risk of 50\% and treated risk of 10\%, with treatment distributed proportionally to assessed risk, a `stable' risk score would assess risk as e.g. 30\%, prompting a milder intervention than actual untreated risk would suggest, after which true risk remains at 30\%, regardless of treatment cost.

We found no applied or theoretical literature directly addressing the focus in this paper: determining how large a holdout set should be. Similar problems do arise in clinical trial design: \cite{stallard17} estimate the optimal size of clinical trial groups for a rare disease in which individuals not in the trial stand to gain more than those in it, using a Bayesian decision-theoretic approach accounting for benefit to future patients in the population. Our problem is related to computation of a minimal training size for a clinical prediction model~\cite{riley20}, but rather than being limited by financial cost of obtaining training samples, we are limited by a cost to all individuals in the holdout set, allowing specification of an explicit tradeoff for larger sample sizes.

{In previous work~\citep{liley21updating} we proposed, in addition to a holdout set, that the problems of updating a risk score in the presence of interventions could broadly be managed by complete causal modelling of the intervention or by explicitly specifying what interventions should be made. In applied work~\citep{liley21medRxiv} we proposed a practical alternative in which we updated a risk score as a maximum of the existing risk score and a refitted score. We are concerned in this work with the setting in which none of these options are usable, which we consider to be a common setting. We detail in Supplement~\ref{supp_sec:natural_holdout} why causal modelling is of limited use when we cannot record an intervention or when we deterministically plan an intervention, and in Supplement~\ref{supp_sec:best_of_two} why our approach of using the maximum of two scores is not generalisable. }

OHS estimation requires quantification of expected material costs when using risk scores trained to holdout data sets of various sizes: that is, the cost of reduced accuracy from limiting the OHS, as well as the cost due to individuals in the holdout set not benefiting from a risk score. Such costs depend on the error in risk predictions. The relation of predictive error to training set size is well studied and known as the `learning curve', which can sometimes be accurately parameterised \citep{amari93}. A recent review paper suggests a power-law is accurate for simple models~\citep{viering21}. 

\subsection{Legal and ethical considerations}
\label{sec:legal_ethical}

{
Use of a holdout set appears ethically tenuous. However, we argue that in many circumstances it is unethical \emph{not} to use a holdout set, in that other options lead to worse outcomes. 
Essentially, use of a holdout set limits costs incurred from inaccuracies in prediction to individuals in that set, whereas all alternatives lead to risk score inaccuracy across the entire population.
We formalise these arguments in Section~\ref{sec:holdout_set_motivation}, in particular showing that the updating paradox described above is inevitable for risk scores on complex systems intended to guide interventions.
}

{Contributions in this area are important due to rapidly evolving legislation. Currently, the European Union treat each update of a risk model as a separate risk score requiring re-approval, but in the United States a proactive approach is taken with a `total-life cycle' paradigm which allows practitioners to update risk models as necessary without requesting approval~\citep{fda19}. This approach could allow updating-induced biases to go undetected, and highlights the need for safe updating methods in risk score deployment.  The use of holdout sets as examined in this work offers one potential solution.
}

\section{Problem description} 
\label{sec:theory}

\renewcommand\theassumption{A\arabic{assumption}}

\subsection{General setup}
\label{sec:general_setup}

We presume a random process $X_t$, $t \in \mathds{R}^+$, representing the covariates of a single sample at time $t$. We let $X_t$ have distribution $\mu_t$ where $\mu_t$ has constant support $\mathscr{X}$. 

To demonstrate why we opt to use a holdout set approach, we define two functions dependent on $t$. We define $f_t(x)$ 
as the probability of an event if no risk score is in place, and $g_t(x;\rho)$ 
as the probability of an event when a risk score $\rho$ is used to guide decisions. 
{We allow $\rho$ to be an arbitrary risk score (that is, $\rho \in R=\{r:\mathscr{X} \to [0,1]\}$ and $g:(\mathscr{X} \times R) \to [0,1]$) but will generally take it to be a risk score fitted to samples encountered prior to or at time $t$.} 
We will not explicitly consider the occurence of adverse events as random variables (for the moment); rather we will simply consider $f_t$ and $g_t$ as functions (noting that $g_t(x,\rho)$ depends on the \emph{function} $\rho$ rather than only the value $\rho(x)$). We will generally omit the argument $\rho$ from $g_t$, as it will usually be clear. We will use $f,g$ to informally denote the sets of functions $\{f_t,t \in \mathbb{R}^+\}$, $\{g_t,t \in \mathbb{R}^+\}$ respectively.

We wish to estimate $f_t$ rather than $g_t$, since it gives a risk of the event in question under standard practice: that is, \emph{without} already using a risk score. An agent may opt to intervene in addition to standard practice if $f_t(x)$ is high.

We thus aim to generate risk scores $\rho_{\cdot}:\mathscr{X} \to [0,1]$ which estimate $f_t$ for $t$ on some interval. Risk scores will be fitted to samples observed over a time period $(e-s,e]$ for $s\leq e$. 
Denote by
\begin{equation}
\rho^{n,e,s}_{h}(\cdot) \label{eq:rho_n_def}
\end{equation}
a risk score fitted to $n$ samples of $(X,Y)$ with $t\sim U(e-s,e)$, $X\sim \mu_t$, and $Y\sim\textrm{Bern}(h(X))$ where $U$ is a uniform distribution, $\textrm{Bern}(\cdot)$ a Bernoulli distribution, and $h$ is {either $f_t(\cdot)$ or $g_t(\cdot;\rho)$ for some $\rho$}. We will measure deviation of a risk score $\rho$ from a function $f$ using (essentially) mean-square generalisation error:
\begin{align}
\xi_t^2(\rho,f) &= \mathbb{E}_{X_t\sim \mu_t} \left\{\left(\rho(X_t)-f(X_t)\right)^2\right\} \, .
\end{align}
We will consider four options for how best to decide on a series of risk scores to be used over a time period $[0,T]$. We will use the term `epoch' to mean periods of time $(0,\delta]$, $(\delta,2\delta]$, $\dots$, $(e,e+\delta]$ with $e \in \mathbb{N}\delta$ during which a particular risk score is used and a new risk score is (potentially) fitted. From Section~\ref{sec:cost_specification} onwards, we will generally take $\delta=1$ which we will be able to do without loss of generality.
\begin{itemize}[leftmargin=1cm]
\item[\textbf{Holdout:}] Amongst samples encountered while $t\in (e-s,e]$,  $e \in \mathbb{N}\delta$, we withhold $n_\star$ randomly chosen samples from attaining risk scores, and fit a risk score to them. We thus obtain a risk score $\rho_{f}^{n_\star,e,s}$ for use on the epoch $(e,e+\delta]$. We call this approach the \emph{holdout} strategy.
\item[\textbf{No-update:}] We use a risk score $\rho_0=\rho^{n,0,0}_{f_0}$, fitted to a number $n$ of samples at time 0, and continue using this score throughout the period. We call this the \emph{no-update} strategy.
\item[\textbf{Na\"{i}ve update:}] Whenever $t\in (e-s,e]$,  $e \in \mathbb{N}\delta$, we fit a risk score to as many samples as possible (say $n'$) on whom a risk score is already used. We thus attain a risk score $\rho_{g_t}^{n',e,s}$ for use on the epoch $(e,e+\delta]$. We call this approach the \emph{na\"{i}ve update} strategy. 
\end{itemize}

We also consider the performance of an unspecified alternative which is `less than asymptotically perfect', in that for some period of time the risk score is (on average) a slightly biased estimator of $f_t$. 
\begin{itemize}[leftmargin=1cm]
\item[\textbf{Alternative:}] We consider an arbitrary alternative giving rise to a risk score $\rho_t$ at time $t$ for which for some value $T$:
 \begin{equation}
 \lim_{N_{\rho} \to \infty} \left(\mathbb{E}_{t \sim U[0,T], D}\left\{\xi_t^2(\rho_t,f_t)\right\}\right) = b > 0 \, , \label{eq:sublinear_condition}
 \end{equation}
where $D$ denotes any information used to fit risk scores $\rho_t$, $N_{\rho}$ is the number of samples used to fit $\rho$, and $b$ does not depend on any parameters of the strategy (e.g., update frequency $\delta$), though may depend on $T$\footnote{We will make assumptions under which the no-update strategy is of the alternative-strategy type. We differentiate them here to illustrate examples of such alternative strategies.}.
\end{itemize}
We will finally consider an `oracle' option. Rather than perfect information (e.g. perfect knowledge of $f_t$ for all $t$), our oracle can observe $f_t$ acting on a number of samples:
\begin{itemize}[leftmargin=1cm]
\item[\textbf{Oracle:}] We presume that whenever $t\in [e-s,e)$,  $e \in \mathbb{N}\delta$, we fit a risk score to $n$ samples of $(X,Y)$, for which $X \sim \mu_t$, $Y \sim \textrm{Bern}(f_t(X))$, and use the risk score $\rho_{f_t}^{n,e,s}$ on the epoch $(e,e+\delta]$. We call this approach the \emph{oracle} strategy.
\end{itemize}

Our strategy for `holdout' updating is illustrated with $\delta=1$ as a causal graph in Figure~\ref{fig:epoch}.  The ellipses $\{X_0\}$, $\{Y_0\}$ correspond to sets of observations of $(X,Y)$ with $X\sim\mu_0$ and $(Y|X) \sim \textrm{Bern}(f_0(x))$ in epoch 0, representing initial training data, to which a risk score $\rho_0$ is fitted, where $\rho_0 \approx f_0$.

We use the shorthand $\{X_e^i\}$ (`$i$' for intervention) to mean a set of samples from $X_t$ with $t \in (e,e+1]$, representing the set of samples on which the risk score is used, and $\{X_e^h\}$ (`$h$' for holdout) to mean sets of samples from $X_t$ with $t$ as close to $e+1$ as possible, representing the set of `holdout' samples on which the next risk score $\rho_e$ is fitted. We define corresponding sets $\{Y_e^i\}$ and $\{Y_e^h\}$ representing observations: for $x_t \in \{X_e^i\}$, we have the corresponding $y_t$ in  $Y_e^i$ distributed as $y_t \sim \textrm{Bern}(g_t(x_t,\rho_{\lfloor t \rfloor}))$ and for $x_t \in \{X_e^h\}$, we have the corresponding $y_t$ in $Y_e^h$ distributed as $y_t\sim \textrm{Bern}(f_t(x_t))$.

Under a `native' setting prior to deployment of a risk score,  $\{X_0\}$ and $\{Y_0\}$ have a single causal link, modelled by risk score $\rho_0$ (leftmost epoch of Figure~\ref{fig:epoch}). Once $\rho_0$ is in use in the intervention set in epoch 1 (ellipses $\{X_1^i\}$, $\{Y_1^i\}$), a second causal pathway through $\rho_0$ is established from $\{X_1^i\}$ to $\{Y_1^i\}$, but there remains only one causal pathway from $\{X_1^h\}$ to $\{Y_1^h\}$ in the holdout set (middle epoch of Figure~\ref{fig:epoch}). The quantities in the shaded area do not causally depend on any quantities outside the shaded area; we will consider them together in Section~\ref{sec:cost_specification}. The updating process can be continued rightwards ($\rho_1$, $\rho_2$, $\dots$).

\begin{figure}[h] 
\centering
\includegraphics[width=0.75\textwidth,clip,trim=0cm 0cm 0cm 0cm]{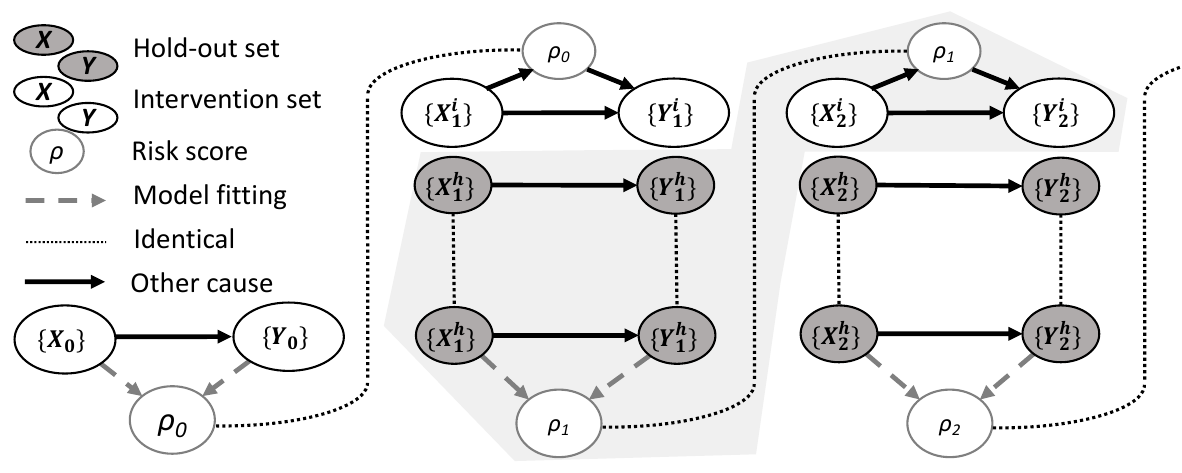}
\caption{Dynamics of a risk model, $\rho$, across three epochs under holdout set updating. Each column corresponds to an epoch, denoted by subscripts $0,1,2$, representing consecutive intervals of continuous time $(0,1], (1,2],$ and $(2,3]$ respectively. Ellipses containing $X$ or $Y$ correspond to covariates and outcomes respectively, with superscripts $i$ and $h$ denoting the partition into intervention and holdout sets respectively.}
\label{fig:epoch}
\end{figure}

\section{Motivation for use of a holdout set}
\label{sec:holdout_set_motivation}

We argue that use of a holdout set approach is generally necessary for tolerable total population costs when updating a risk score. In this sense, we argue that it is an ethical imperative to consider use of a holdout set in this context.  We will now make and justify a series of assumptions (which will only be used directly in this section). We will discuss robustness to these assumptions in Section~\ref{sec:robustness_to_assumptions_a}.

The motivation for updating a risk score is generally that the true risk $f_t$ changes continuously with time $t$ in a random way. 
If $f_t$ did not change with $t$ there would be no need to update a well-fitted risk score. 
Lipschitz-continuous change in the distribution of covariates and the risk function underlies the assumption that, for a risk score to be useful {(regardless of updating strategy)}, it must remain reasonably accurate for a period of time after it is fitted.

\begin{assumption}[Non-negligible drift]
\label{asm:f0_difference}
{For some $T>0$ we have }
$\int_0^T \xi_t^2(f_t,f_0) dt>0$ 
\end{assumption}

\begin{assumption}[Continuous drift in population distribution]
\label{asm:d_drift}
{$\mu_t$ is uniformly Lipschitz continuous in $t$ with respect to total variation distance.}
\end{assumption}

\begin{assumption}[Continuous drift in target function]
\label{asm:ft_lipschitz}
For all $x$, $f_t(x)$ is $\alpha_1$-Lipschitz continuous in $t$ (where $\alpha_1$ does not depend on $x$)
\end{assumption}

{We implicitly assume that drift in $f_t$ is unpredictable: if $f_t$ changed in a \emph{predictable} way, we could simply update our estimate of $f_t$ with time.}

We presume that adding more samples to a training set for a risk score improves its error at an $O(n^{-1})$ rate (typical of, for instance, linear models). In subsequent sections, we will consider instead an arbitrary `learning curve'.

\begin{assumption}[Expected prediction error]
\label{asm:risk_score_convergence}
For $h \in \{f, g\}$, denoting 
\begin{equation}
H(x)=\mathbb{E}_{t \sim U(e-s,e)} \{h_t(x)\} \nonumber
\end{equation}
we have 
\begin{equation}
\mathbb{E}\left\{\xi_u^2(\rho_h^{n,e,s},H)\right\}=O(n^{-1}) \nonumber
\end{equation}
for any time $u$, where the expectation is over the data used to fit $\rho_h^{n,e,s}$.
\end{assumption}

Since we assume prediction of an adverse event, we want to lower the chance of this event as much as possible. 
We define a cost per sample in terms of the reduction in the probability of an adverse event which can be achieved by intervention: that is, how much lower is $g_t$ (probability of an adverse event when using a risk score to guide decisions) than $f_t$ (probability of adverse event). For the moment, we are unconcerned with the scale of this cost, so we assume a unit coefficient for this difference. 
We shift this cost so that it is zero when we are doing `as well as possible'; that is, when $\rho$ is exactly $f$. 

\begin{assumption}[Costs]
\label{asm:costs}
For a sample encountered at time $t$, when a risk score $\rho$ is in use, we define the expected total cost, {$c_t(\rho)$}, as:
{
\begin{equation}
c_t(\rho)=\mathbb{E}_{X \sim \mu_t} \left\{ g_t(X,\rho) - g_t(X,f_t) \right\} \, ;\nonumber
\end{equation}
that is, the change in expected difference in the probability of the outcome when using an estimate $\rho_t$ of $f_t$ and with perfect knowledge of $f_t$. We make the assumption that:
\begin{equation}
k_1 = \mathbb{E}_{X \sim \mu_t} \left\{ f_t(X) - g_t(X,f_t) \right\} > 0 \, .
\end{equation}
}
\end{assumption}
{The value of $k_1$ can be considered as the cost per sample when no risk score is used (since the risk of outcome is $f_t(X)$ in this case), so the assertion that $k_1>0$ is tantamount to a `potentially useful' risk score: with perfect knowledge of $f_t$, it is possible to reduce the expected risk of the outcome below $f_t$ itself. }
%

We now make a key assumption in presuming that the risk score is `useful', in that our propensity to make $g_t$ smaller than $f_t$ depends directly on the accuracy of the score. Taken together, \ref{asm:costs} and \ref{asm:usefulness} constrain the relationship between $f_t, g_t,$ and $\rho$, helping to simplify the statement of our initial results, but we will show that we can substantially weaken this assumption in Section~\ref{sec:robustness_to_assumptions_a}, and discuss our reasons for making this type of assumption in Section~\ref{sec:practicalities}. 

\begin{assumption}[Usefulness of risk score]
\label{asm:usefulness}
If a risk score $\rho_t$ is in use at time $t$, we have  
\begin{equation}
c_t(\rho_t) = k_2\,\xi_t^2(\rho_t,f_t) \nonumber
\end{equation}
for some constant $k_2>0$.
\end{assumption}

We note that it immediately follows that cost is minimised at $0$ if and only if $\rho_t=f_t$ holds $\mu_t$-almost everywhere{, and that we cannot find a `better' risk score than $f_t$ itself} (our cost may be considered an `excess cost' over a setting with essentially perfect knowledge of $f_t$), and we allow that a sufficiently inaccurate risk score may incur a higher cost than no risk score, which could potentially occur in real settings. We give an example of a class of functions $g_t$ satisfying Assumptions~\ref{asm:costs} and~\ref{asm:usefulness} in Supplement~\ref{supp_sec:relaxation_usefulness}. Finally, we make simplifying assumptions of independence of samples at different times and a constant population size.

\begin{assumption}[Population size and holdout number]
\label{asm:pop_size}
The times at which samples are observed occur uniformly randomly over time (that is, from a Poisson process) with a mean of $N$ samples per unit of time. When using the holdout strategy, we use a constant number $n_\star$ of samples in each epoch (if there are fewer than $n_\star$ samples available in $(e-s,e]$, we use no risk score for the subsequent epoch $(e,e+\delta]$, but this will be rare). {Samples of $X_t$ need not be independent, although in many cases they will be nearly so: for the ASPRE score, for instance, treatment decisions will generally be made only once per pregnancy.}
\end{assumption}

We define the `cost of a strategy' $C_{\textrm{strat}}[t_1,t_2]$ as the cost accrued {for all samples encountered} over time interval $[t_1,t_2]$ using a given strategy `strat' (which may be: $(h)$: holdout; $(0)$: no-update; $(n)$: na\"{i}ve update; $(a)$: alternative or $(o)$: oracle). {Given Assumptions~\ref{asm:pop_size} and \ref{asm:usefulness}, we have 
\begin{equation}
\frac{1}{N(t_2-t_1)}\mathbb{E}\{C_{\textrm{strat}}[t_1,t_2]\} = 
\mathbb{E}_{t \sim U[t_1,t_2],\rho_t}\{c_t(\rho_t)\} = \mathbb{E}_{t \sim U[t_1,t_2],\rho_t}\{k_2\,\xi_t^2(\rho_t,f_t)\} \, ,\nonumber 
\end{equation}
where the (potentially random) function $\rho_t$ is determined by the strategy.} We now state three results, the first of which establishes the growth rate of the holdout strategy over time.

\begin{theorem}
\label{thm:holdout_asymptotic}
Suppose we use a holdout set with size $n_\star=\Theta(N^a)$, with $0<a<1$, an $s<1$ of size $s=\Theta(N^{a + \epsilon-1})$ for some $\epsilon$ with $0<\epsilon<1-a$, and an update frequency $\delta \leq 1$ which  may vary with $N$.
{Under Assumptions~\ref{asm:ft_lipschitz}, \ref{asm:risk_score_convergence}, \ref{asm:usefulness} and \ref{asm:pop_size},}
we have
\begin{equation}
\mathbb{E}\left\{C_{(h)}[0,T]\right\} = \delta N T k_2 (2\alpha_1 + \alpha_1^2) +  \delta^{-1} O(N^a) + O(N^{a + \epsilon}) +  O(N^{1-a}) \, .\nonumber 
\end{equation}

\end{theorem}

%
For fixed $\delta$, this suggests an optimal holdout set size $n_\star=\Theta(N^{1/2})$. If we choose $\delta=O(N^{-b})$ {with $0<b<1-(a+\epsilon)$ (given that we must have $s < \delta < 1$)} then  
\begin{equation}
\mathbb{E}\left\{C_{(h)}[0,T]\right\} =  O(N^{a+b}) + O(N^{a + \epsilon}) +  O(N^{1-a}) +  O(N^{1-b}) \, , \label{eq:delta_variation}
\end{equation}
and we achieve an optimal sublinear asymptotic growth rate of $O(N^{\frac{2}{3}})$ if $a=\frac{1}{3}$, $b=\frac{1}{3}$.

We now consider the costs of the no-update and na\"{i}ve-update strategies. We will show that for either strategy, costs must grow at least linearly in population size (thereby also showing that the updating paradox arises inevitably from Assumptions~\ref{asm:costs}, \ref{asm:usefulness}):

\begin{theorem}
\label{thm:alternative_asymptotic}
Suppose we choose $s$ such that $s \to 0$ as $N \to \infty$. Under Assumptions~\ref{asm:f0_difference}--\ref{asm:pop_size}, {for 
sufficiently small $\delta$}, we have for $(\textrm{strat}) \in \{(0),(n),(a)\}$:
\begin{equation}
\mathbb{E}\left\{C_{(\textrm{strat})}[0,T]\right\} = \Omega(N) \, .
\end{equation}
\end{theorem}

We note that Theorem~\ref{thm:alternative_asymptotic} immediately implies a dominance of the holdout-set strategy, since we cannot attain sublinear cost growth with the no-update, na\"{i}ve-update, or alternative strategies. 

Finally, we show that for fixed $\delta$, the holdout strategy is essentially optimal in that its cost is asymptotically similar to that of the oracle strategy. This is not true for the no-update, alternative or na\"{i}ve-update strategies, for which total cost arbitrarily exceeds that of the oracle.

\begin{theorem}
\label{thm:oracle_comparison}
Consider use of each strategy in parallel with an `oracle' procedure. Under Assumptions~\ref{asm:f0_difference}--\ref{asm:pop_size}, {with sufficiently small fixed $\delta$,} 
holdout set size $n_\star=\Theta(N^{1/2})$, and $s<1$ of size $s=\Theta(N^{\epsilon-1/2})$ for some $\epsilon$ with $0<\epsilon<1/2$, we have for $(\textrm{strat}) \in \{(0),(n),(a)\}$:
\begin{equation}
\lim_{N \to \infty}\left( \frac{\mathbb{E}\left\{C_{(h)}[0,T]\right\}}{\mathbb{E}\left\{C_{(o)}[0,T]\right\}}\right) = 1\textrm{, and } 
\lim_{N \to \infty}\left(\frac{\mathbb{E}\left\{C_{(\textrm{strat})}[0,T]\right\}}{\mathbb{E}\left\{C_{(o)}[0,T]\right\}}\right) > 1 \hspace{5pt} \left(=\Omega(\delta^{-2})\right) \, . \nonumber
\end{equation}
\end{theorem}

We prove Theorems~\ref{thm:holdout_asymptotic}, \ref{thm:alternative_asymptotic} and \ref{thm:oracle_comparison} in Supplement~\ref{supp_sec:proofs_holdout_dominance}. Restrictions on $\delta$ in Theorems~\ref{thm:alternative_asymptotic}, \ref{thm:oracle_comparison} are required because for large $\delta$ {(relative to Lipschitz constants in Assumptions~\ref{asm:d_drift} and \ref{asm:f0_difference})} too much drift occurs between updates to guarantee sustained performance for the time a risk score is in use.

In general, our results show that any strategy for which alternative strategy condition~\eqref{eq:sublinear_condition} holds leads to higher costs than the holdout set approach. We claim that it is essentially impossible for a strategy to evade condition~\eqref{eq:sublinear_condition} (that is, be able to arbitrarily closely approximate $f_t$ at almost all $t \in [0,T]$) without the use of a holdout set, and hence have costs competitive to the holdout-set strategy.

In our setting, without a holdout set, the only information available when we wish to update the model is a set of samples of $X_t$ and $\textrm{Bern}\{g_t(X_t,\rho_t)\}$ for some $\rho_t$, where $g_t(\cdot,\rho_t) \neq f_t$ and $\rho_t \neq f_t$ (if $f_t=g_t$, the risk score is prompting no change in behaviour, which generally contradicts its usefulness). With only this information, we cannot hope to infer $f_t$ without error. To see this, suppose we have some risk score $\rho_t$ in place over an epoch and observe the function $g_t(\cdot,\rho_t)$. All we know about $f_t$ comes from Assumptions~\ref{asm:costs} and \ref{asm:usefulness}, that is:
\begin{equation}
k_2\xi_t^2(f_t,\rho_t) = k_1 - \mathbb{E}_{X \sim \mu_t}\{f_t(X)-g_t(X,\rho_t)\} \, ,
\end{equation}
from which $f_t$ is not identifiable (see Supplement~\ref{supp_sec:nonidentifiability} for details and an explicit example).

We may alternatively try to estimate $f_t$ using our knowledge of $f_u$ with $u<t$ (we will say $u=0$). If $f_0=f_t$, there is no drift, and the risk score need not be updated. To accurately infer $f_t$ from $f_0$, we would need to know exactly how $f_t$ changes with $t$, whereas typically drift in $f_t$ is random. 

It is more conceivable that $f_t$ could be inferred from $g_t$ using additional information (for instance, records of interventions). However, risk scores are usually used in complex settings (e.g. medicine, finance or law), and decisions are nuanced, so the extent to which a decision is due to a risk score is hard to quantify, even if decisions are explicitly recorded. In the most general setting where we have no additional information, we claim a holdout set provides the most principled statistical approach.

{
In the context of the ASPRE score, Assumption~\ref{asm:f0_difference} corresponds to the claim that over time, the risk of PRE given ultrasound-derived predictors changes in a non-negligible way and Assumptions~\ref{asm:d_drift} and \ref{asm:ft_lipschitz} assert that this change happens gradually. Assumption~\ref{asm:costs} and~\ref{asm:usefulness} state that our cost (number of PRE cases per unit time) should worsen according to the accuracy of our prediction of PRE risk. Assumption~\ref{asm:pop_size} states that we encounter pregnancies at approximately an even rate over time. 
}

{
Presuming that we may not update the score arbitrarily frequently, Theorem~\ref{thm:holdout_asymptotic} indicates we should choose a holdout set size scaling roughly as $\sqrt{N}$ to optimise costs, and Theorems~\ref{thm:alternative_asymptotic} and~\ref{thm:oracle_comparison} show that our costs, should we do so, will asymptotically be no worse than if we incurred no costs in the holdout set, but would be substantially worse should we update `na\"{i}vely' without one, or not update the model at all.
}

\subsection{Robustness} 
\label{sec:robustness_to_assumptions_a}

We briefly discuss the robustness of the findings in Theorems~\ref{thm:holdout_asymptotic}, \ref{thm:alternative_asymptotic} and \ref{thm:oracle_comparison} to Assumptions~\ref{asm:f0_difference}--\ref{asm:pop_size}. We first note that we can weaken Assumption~\ref{asm:usefulness} to 
\begin{equation}
k_2^l \xi^{b_l} \leq c_t \leq k_2^u \xi^{b_u} \, ,\label{eq:relaxed_usefulness}
\end{equation}
where $\xi=\xi_t^2(\rho_t,f_t)$, for some constants $k_2^l,k_2^u>0$ and $0<b_u \leq 1 \leq b_l$, whilst retaining the correctness of Theorem~\ref{thm:alternative_asymptotic}, requiring only minor modifications to Theorems~\ref{thm:holdout_asymptotic}, \ref{thm:oracle_comparison}, and Equation~\ref{eq:delta_variation} (Supplement~\ref{supp_sec:relaxation_usefulness}). {Although we assume that $k_1$ and $k_2$ (or $k_2^l,k_2^u$) are constant over time, we may relax this to their being bounded-below over time by positive constants. }

{In the absence of drift (assumption~\ref{asm:f0_difference}), the no-update strategy is preferable to the holdout-set update, though they remain equivalent as $N \to \infty$ in the sense of Theorem~\ref{thm:oracle_comparison}. However, the naive-update strategy remains suboptimal in this case (Supplement~\ref{supp_sec:f0diff}). }

We do not explicitly mention the possibility of `latent' covariates which influence the risk of outcome $f_t$. Rather, we may simply assume that risk scores are considered as marginals over latent covariates, which we discuss in greater detail as part of Supplement~\ref{supp_sec:natural_holdout}. Additionally, we do not assume that drift is independent of our actions: that is, the choice of who to intervene on may affect $\mu_t$ in the future. However, we do not explicitly model any long-term effects of intervention. {We presume that the intervention is not able to be recorded; we briefly discuss possibilities if the intervention is recorded (including use of a treatment indicator as a covariate in the predictive model) in Supplement~\ref{supp_sec:natural_holdout}.}

If $f_t$ or $\mu_t$ are not Lipschitz continuous for some $t$ (Assumptions~\ref{asm:ft_lipschitz}, \ref{asm:d_drift} respectively), then a risk score approximating $f_{t-\epsilon}$ will not necessarily approximate $f_{t+\epsilon}$, even for small $\epsilon$. Practically, this means that performance of a risk score fitted prior to $t$ cannot be guaranteed after $t$. In practice, this could readily occur (for instance, a financial risk score fitted prior to an unexpected disaster), but such an event would nonetheless be better-managed using the holdout-set strategy than others, since a holdout set could be used to `reset' the risk score after $t$. 

Heuristically, when using a na\"{i}ve update strategy, a less-biased risk score (that is, for which $\rho_t$ is more similar to $f_t$) will initially result in lower costs during the first deployment epoch (by a mechanism analogous to Assumption~\ref{asm:usefulness}), but then suffer higher costs after updating (since lower bias induces a larger difference between $f_t$ and $g_t$ by Assumption~\ref{asm:costs}, worsening the consequences of using $g_t$ to approximate $f_t$). Under Assumption~\ref{asm:f0_difference}, a non-updated model will accrue greater costs over time. We make the case that holdout sets enable good estimation of $f_t$ for most of the sample population at all $t$. 

These heuristics generally remain true when using a metric of similarity between $\rho_t$ and $f_t$ other than $\xi_t^2(f_t,\rho_t)$, or when Assumptions~\ref{asm:risk_score_convergence}, \ref{asm:usefulness}, \ref{asm:costs} and \ref{asm:pop_size} only roughly hold. In such cases, the holdout set approach will still tend to be lower-cost than other approaches for sufficiently large $N$, although the precise statements of the Theorems may not hold. For example, if Assumption~\ref{asm:risk_score_convergence} fails perhaps due to model mis-specification so that accuracy decreases to a positive minimum, then costs will grow linearly with the holdout set strategy. However, it will be closer to the oracle strategy than no-update, na\"{i}ve-update or alternative strategies. {Although we specify a constant size (up to order) for the holdout set in Theorems~\ref{thm:holdout_asymptotic} and \ref{thm:oracle_comparison}, in practice lower costs may be attainable with variable holdout set sizes.}

Nonetheless, we do acknowledge that our formulation is quite prescriptive. We roughly model a medical setting in which a physician sees a patient, assesses them, then intervenes and sends them away, only observing some outcome later. There are of course related settings for which this formulation is inappropriate, and we once again must defer to heuristic arguments in such situations.

\subsection{Note on ethics}
\label{sec:ethics}

In medical settings (such as our motivating example) the use of holdout sets appears ethically tenuous, due to differential treatment of samples in and out of the holdout set. We argue, however, that the use of holdout sets in these settings should still be considered. Our main reason is the absence of a viable alternative: as above, in many settings it is not possible to attain costs comparable to an idealised oracle strategy without a holdout set. We recapitulate that, even in the event that risk-score guided interventions are recorded, it is \emph{not} sufficient to use a `natural' holdout set by simply considering individuals who received a risk score, but were untreated, and it is not necessarily possible to infer $f_t$  (Supplement~\ref{supp_sec:natural_holdout}). The ethical questions surrounding holdout sets are discussed in more depth in~\cite{chislett24}.

We do note an important subtle assumption is that $k_1$ is finite (Assumption~\ref{asm:costs}). Settings in which it is unacceptable for even one sample to not have a risk score correspond to an infinite $k_1$, and in such settings we must make do with a risk score for which performance is not close to optimal. By contrast, there are settings for which the absence of a risk score is less serious: for instance, if the outcome in question is not life-threatening and for which existing best practice is often adequate for identifying cases (e.g., tooth decay~\citep{zukanovic13} or minor sexually transmitted infections~\citep{kranzer21}). 
In such settings the use of a holdout set is more ethically acceptable, since the cost to any given individual (even in the holdout set) is low.

Lastly, we underscore the importance of considering the adoption of a holdout set, when other ethical concerns do not take precedence. In situations where a holdout set is not feasible, and withholding any treatment is deemed unacceptable, the ability to effectively respond to a drifting ground truth is compromised. Consequently, updates would ultimately yield sub-optimal risk scores for the entire population.

\section{Choosing the size of a holdout set}
\label{sec:cost_specification}

\setcounter{assumption}{0}
\renewcommand\theassumption{B\arabic{assumption}}

For the remainder of this paper, we will be concerned with choosing an optimal size for the holdout set. We begin by somewhat simplifying our formulation, with a focus on total cost, and no longer using Assumptions~\ref{asm:f0_difference}--\ref{asm:pop_size}. We will retain $k_1$ and $k_2$ as having roughly their existing meanings.

{
Returning to Figure~\ref{fig:epoch}, our goal at time $e$ is to nominate a size $n$ of the holdout set $\left(\{X_e^h\},\{Y_e^h\}\right)$, comprising samples encountered during $[e,e+1)$, to be used to fit a risk score to be used during $[e+1,e+2)$. During $[e+1,e+2)$, the risk score is used only on the intervention set $\left(\{X_{e+1}^i\},\{Y_{e+1}^i\}\right)$, whose size will depend on the choice of holdout set size at the next update time $e+1$. To estimate the costs incurred on this intervention set, we take its size to be $N-n$: the same as the size of the intervention set in $[e,e+1)$, with $N$ samples encountered in total in $[e,e+1)$. We do not, however, intend that the same holdout size be used for successive epochs in general. 
}

We denote the $n$ samples in the `holdout' set $\left(\{X_e^h\},\{Y_e^h\}\right)$ as $D_n$ (where `$D$' indicates `data'). Since we choose $\left(\{X_e^h\},\{Y_e^h\}\right)$ to be as close in time to $\left(\{X_{e+1}^i\},\{Y_{e+1}^i\}\right)$ as possible, we will presume that
\begin{align}
X_e^h &\sim \mu_{e+1} \, , \hspace{20pt} & \mathbb{E}(Y_e^h|X_e^h) &= f_{e+1}(X_e^h) \, ,\nonumber \\
X_{e+1}^i &\sim \mu_{e+1} \, , \hspace{20pt} &\mathbb{E}(Y_{e+1}^i|X_{e+1}^i) &= g_{e+1}(X_{e+1}^i,\rho_e) \, .\nonumber 
\end{align}
A risk score $\rho_e$ is fitted to $D_n$ which approximates $f_{e+1}(x)$ and is used in the intervention set $\left(\{X_{e+1}^i\},\{Y_{e+1}^i\}\right)$. We will presume that samples in $D_n$ are pairwise independent, as are samples in $\left(\{X_{e+1}^i\},\{Y_{e+1}^i\}\right)$, although samples in the latter depend on $D_n$ through the fitted risk score.

{
We define $C_1(X)$ and $C_2(X;D_n)$ as random variables associated with the total `cost' of an observation with covariates $X$ in the holdout set in epoch $e$ and intervention set in epoch $e+1$ respectively. 
Although the most natural cost (as in Assumption~\ref{asm:costs}) may be the number of adverse events, we take the `cost' in this case to represent only a quantity we aim to minimise through our use of the risk score. The cost could, for example, encompass the costs of managing adverse events and the costs of administering interventions.}

{
The value of $C_2(X; D_n)$ depends on $D_{n}$ only through the risk score fitted to $D_n$. We define the expected cost per observation in the holdout and intervention sets, respectively:
\begin{equation}
    k_1 = \mathbb{E}_{X \sim \mu_{e+1},C_1} \left\{C_1(X)\right\} \, , \hspace{20pt}
    k_2(n) = \mathbb{E}_{X \sim \mu_{e+1},C_2}\left[\mathbb{E}_{D_{n}} \left\{ C_2(X; D_n)\right\}\right] \, .\label{eq:k2def}
\end{equation}
Subscripted values $C_1$ and $C_2$ indicate variance in $C_1(X)$, $C_2(X; D_n)$ independent of $X,D_{n}$. We now express total expected cost $\ell$ across all samples as a function of holdout set size $n$:
\begin{equation}
\ell(n) = \underbrace{k_1 n}_{\text{Holdout set}} + \underbrace{k_2(n)(N-n)}_{\text{Intervention set}} \, . \label{eq:loss_specification}
\end{equation}
As for $C_1, C_2$, 
the meaning of $\ell$ is contextual dependent on the application; for instance, in QRISK3 it may mean total number of deaths for a fixed healthcare budget.
}

\subsection{Sufficient conditions for existence of an OHS}
\label{sec:assumptions}

In this Section, we consider conditions under which the cost $\ell(n)$ can be readily optimised. We discuss estimation of $N$, $k_1$ and $k_2(\cdot)$ in Section~\ref{sec:ohs_estimation}. We begin with the following assumptions:
\begin{assumption}\label{item:k1_indep_pi}
$k_1$ does not depend on $n$: in a medical context, this means for example that treatment plans and outcomes for patients without risk scores do not depend on the number of such patients.
\end{assumption}
\begin{assumption}\label{item:k2_decrease}
$k_2(n)$ is monotonically decreasing in $n$: the more data available to train the risk score, the greater its clinical utility.
\end{assumption}
\begin{assumption}\label{item:k1_l_k2}
There exists $M \in (0,N)$ such that $n\geq M \Leftrightarrow k_2(n) \leq k_1$: a good enough risk score will lead to better patient outcomes than baseline treatment, and a poor enough risk score fitted to small amounts of data leads to worse expected outcomes than baseline treatment.
\end{assumption}
\begin{assumption}\label{item:k2_2nd_der}
$\mathbb{E}\left\{ k_2(i+1)-k_2(i)\right\} > \mathbb{E}\left\{k_2(j+1)-k_2(j)\right\}$ for $1 \leq i < j \leq N-1$, with expectations over training data: the `learning curve' for our risk score is convex; there are diminishing returns in the cost per patient from adding more samples to the training data.
\end{assumption}

We may extend the domain of $k_2(\cdot)$, $\ell(\cdot)$ to the real interval $[0,N)$ such that both functions are smooth; and $k_2'(n)< 0$, given Assumption~\ref{item:k2_decrease},  $k_2''(n) > 0$, given Assumption~\ref{item:k2_2nd_der}.
This leads to the following result, that there exists an optimal size for the holdout set minimizing the expected total cost. The proof is given in Supplement~\ref{apx:thm1proof}. 

\begin{theorem}
\label{thm:ohs_exists}

Suppose Assumptions~\ref{item:k1_indep_pi}--\ref{item:k2_2nd_der} hold. Then there exists an OHS $N_\star \in \{1, \dots, N-1\}$ with $N \in \mathbb{N}$, such that:
$\ell(i) \geq \ell(j) \textrm{ for } 0 < i < j < N_\star$ and 
$\ell(i) \leq \ell(j) \textrm{ for } N_\star < i < j < N$
\end{theorem}

{In the context of the ASPRE score, Assumption~\ref{item:k1_indep_pi} indicates that if an ASPRE score is unavailable for a given patient, their care is independent of the number of people used in training the ASPRE score. Assumptions~\ref{item:k2_decrease} and~\ref{item:k2_2nd_der} state that when training the score, the usefulness of the score (that is, the frequency of \textsc{pre} in patients on whom we use the score) improves at a diminishing rate as we increase the number of training samples. Assumption~\ref{item:k1_l_k2} states that a sufficiently good risk score is better for patients than no risk score at all. Theorem~\ref{thm:ohs_exists} then states that, to minimise overall cost of \textsc{pre} cases, a holdout set of some non-zero size should be used to update the score. }

We note that the OHS always exceeds the minimal training sample size required to match baseline treatment.

\begin{corollary}
\label{cor:linear_improvement}
The value of $N_\star$ always exceeds the value of $M$ in Assumption~\ref{item:k1_l_k2}, since if $N'<M$ we have
$\ell(N') = k_1 N' + k_2(N')(N-N') 
> k_1 N' + k_1(N-N')
= k_1 N 
\geq \ell(N_\star)$ 

\end{corollary}

Consequently Assumption~\ref{item:k2_2nd_der} may be relaxed for $i,j < M$; we need only be concerned with the behaviour of $k_2(n)$ at realistically large values of $n$, rather than $n \in \{1,\dots,M\}$. We also have
\begin{equation}
\lim_{N \to \infty} \lim_{n \to N} \ell'(n) = \lim_{N \to \infty} \lim_{n \to N} \left\{ k_1-k_2(n) +(N-n)k_2'(n)\right\} =
k_1 - \lim_{n \to \infty} k_2(n) \, , \nonumber
\end{equation}
and since $k_1>k_2(n)>0$ for large $n$, we have that expected total costs $\ell(n)$ are increasing, but bounded by the per observation expected cost of baseline treatment $k_1$.

{Let us suppose that the holdout set used to fit $\rho_e$ is encountered at time $e+1$, and that we have Assumption~\ref{asm:usefulness} and a stronger form of Assumption~\ref{asm:risk_score_convergence} in that $\xi_{e+1}^2(\rho_{e},f_{e+1}) =c_2 n^{-1} = O(n^{-1})$ for some constant $c_2$. The cost for a single sample with distribution $\mu_{e+1}$ (as in the intervention set), by Assumption~\ref{asm:usefulness} is given by $c_{e+1} = k_2(n)=k_2 \xi_{e+1}^2(\rho_e,f_{e+1}) = k_2 c_2 n^{-1}$ and $\ell(n) = k_1 n + k_2 c_2 n^{-1}(N-n)$. This is minimised at $n=\sqrt{k_2 c_2 N/k_1}=\Theta(N^{1/2})$, in line with the fixed-$\delta$ theoretical OHS in Section~\ref{sec:holdout_set_motivation}}.

\subsection{Robustness to Assumptions~\ref{item:k1_indep_pi}--\ref{item:k2_2nd_der}}
\label{sec:robustness_to_assumptions}

The applicability of Assumptions~\ref{item:k1_indep_pi}--\ref{item:k2_2nd_der} in real world settings requires careful consideration. We address violations of Assumptions~\ref{item:k2_decrease} and~\ref{item:k2_2nd_der} in Section~\ref{sec:emulation} and Assumptions~\ref{item:k1_indep_pi} and~\ref{item:k1_l_k2} here.

Assumption~\ref{item:k1_indep_pi} is fundamental to the success of the holdout set concept. It may be violated if, for instance, agents who can make interventions learn the behaviour of a risk score and apply this to samples with no score. However, such violations are not of serious concern: if we presume that such changes in agents endure over time, then they can be considered as simply contributing to drift, which need not be independent of holdout set size. 

If ethically appropriate, Assumption~\ref{item:k1_indep_pi} could be assured by partitioning agents to manage only samples in holdout sets or only in intervention sets (e.g., cluster randomisation). This requires assuming that changes in agent behaviour as above \emph{do not} endure until the following epoch.

If Assumption~\ref{item:k1_l_k2} fails because $k_2(0) \leq k_1$, we may show a weaker result (the presence of a non-trivial but potentially non-unique minimum loss) by replacing Assumptions~\ref{item:k2_decrease}, \ref{item:k1_l_k2} and~\ref{item:k2_2nd_der} with: 

\begin{assumption}\label{item:frac_imp}
There exists an $0 < M < N$ such that $\frac{N-M}{N}(k_1 - k_2(M)) > k_1 - k_2(0)$
\end{assumption}

This assumption is 
essentially stating 
that at some point the risk score will greatly outperform a risk score built with no data. 
This leads to the result 
 that:
\begin{theorem}\label{thm:ohs_exists_weak}
Suppose Assumptions \ref{item:k1_indep_pi} and
\ref{item:frac_imp} hold, and $k_2(0) \leq k_1$. Then there exists an $N_\star \in \{1, \dots, N-1\}$ such that:
$\ell(i) \geq \ell(N_\star) \textrm{ for } i \in \{1,\dots,N-1\}$ and
$\ell(i) > \ell(N_\star) \textrm{ for } i \in \{0,N\}$
\end{theorem}
In a setting in which $k_1 > k_2(0)$ but one or more of Assumptions \ref{item:k2_decrease}--\ref{item:k2_2nd_der} do not hold, we have
\begin{theorem}\label{thm:ohs_cor}
Suppose Assumption~\ref{item:k1_indep_pi} holds, $k_1 < k_2(0)$ and there exists $0<M<N$ such that $k_2(M) < k_1$. Then there exists an $N_\star \in \{1, \dots, N-1\}$ such that:
$\ell(i) \geq \ell(N_\star) \textrm{ for } i \in \{1,\dots,N-1\}$ and
$\ell(i) > \ell(N_\star) \textrm{ for } i \in \{0,N\}$.
\end{theorem}
Both results are proved in Supplement~\ref{apx:thm1proof}. 
In the setting where $k_1=k_2(0)$ and Assumption~\ref{item:k1_indep_pi} does not hold, it may sometimes be reasonable to assign samples in the holdout set 
risk scores based on no data (for example risk scores generated entirely from expert opinion) and blind agents to holdout/intervention status. Under this setting we may have greater assurance of Assumption~\ref{item:k1_indep_pi}.

\section{Estimation of OHS}
\label{sec:ohs_estimation}

\subsection{Estimation of $k_2(n)$}
\label{sec:practicalities}

We are aiming to find a holdout set size which minimises costs during an epoch $e \geq 1$, $t \in [e,e+1)$ (noting that the holdout set will be used late in the epoch when $t \approx e+1$). This choice must be made during epoch $e-1$ (when $t<e$). We take it that we have the following:
\begin{enumerate}
\item An approximate number of samples on which the model will be used or refitted; \label{item:duration}
\item A cohort of samples $(X,Y)$ with $X \sim \mu_{e-\epsilon} \approx \mu_e$, $Y|X \sim f_{e-\epsilon}(X) \approx f_e(X)$, with $\epsilon$ small. \label{item:samples_holdout}
\end{enumerate}
In~\ref{item:samples_holdout}, the samples are from a holdout set if $e>1$, or from initial training data if $e-1=0$.
We aim to estimate the cost function $\ell(n)$ at $t \in [e,e+1)$. Our approach is to estimate $\ell(n)$ for $t=e-\epsilon$ and assume that the OHS is approximately conserved from $t=e-\epsilon$ to $t=e+1$, though in reality drift may occur in $\ell(n)$.

At time $t=e-\epsilon$, we need to estimate constants $N$, $k_1$, and the function $k_2(\cdot)$. The constants $N$ and $k_1$ are straightforward: $N$, the total number of samples on which a predictive score can be fitted or used, will usually be known or specified (item~\ref{item:duration} above); and $k_1$, the average cost per sample under baseline behaviour without a score, can be estimated from observed costs in the cohort in item~\ref{item:samples_holdout} above.
The function $k_2(\cdot)$ (Equation~\ref{eq:k2def}) is more difficult to estimate, as it involves quantifying costs of hypothetical risk scores. We may tractably estimate $k_2(\cdot)$ by assuming that
\begin{equation}
\mathbb{E}_{X\sim \mu_e, C_2} \left\{ C_2(X;D_n)\right\}=\mathscr{L}\{\textrm{err}(\rho_{D_n})\} \, ,
\end{equation}
where $\rho_{D_n}$ is a risk score fitted to $D_n$, $\textrm{err}(\cdot)$ is a measure of error, and $\mathscr{L}$ is some nondecreasing function. 
We claim that in general circumstances we may take $\mathscr{L}(\cdot)$ to be linear and $\textrm{err}(\cdot)$ to be expected mean-squared error (MSE) or a similar general loss. We derive this heuristically in Supplement~\ref{supp_sec:k2_estimation} and derive expressions for $k_2(n)$ directly in a specific case in Section~\ref{sec:aspre}. Once $\mathscr{L}$ is known, this allows $k_2(n)$ to be estimated readily by establishing the `learning curve' of the risk score using  item~\ref{item:samples_holdout} above. 

Some direct estimates of $k_2(n)$ are necessary to determine $\mathscr{L}$. One option is to designate subcohorts of the intervention set in epoch $e-1$ to receive risk scores fitted to smaller subsamples of available training data, allowing direct observation of the costs of such risk scores.
While simple, this approach may be ethically tenuous and expensive. 
Other options include estimating the function $\mathscr{L}$ through expert opinion or other outside information. 

In summary, we recommend that $k_2(n)$ is estimated by jointly making a small number of estimates during epoch $e-1$, either directly or indirectly, to establish $\mathscr{L}$, and thereafter estimated by evaluating the error of a risk score fitted to $n$ samples using the set in item~\ref{item:samples_holdout} and transforming it according to the estimated $\mathscr{L}$.

\subsection{Parametric estimation of OHS}
\label{sec:parametric}

A natural algorithm for estimating the OHS is immediately suggested by Theorem~\ref{thm:ohs_exists}: assume $k_2$ is known up to parameters $\theta$, and estimate $N$, $k_1$ and $\theta$ to estimate the OHS. Parameters $\theta$ of $k_2$ may be estimated from observations of pairs $\{n,k_2(n)\}$, potentially with error in $k_2(n)$. To minimize the number of estimates of $k_2(n)$ we iteratively add observations $(n,k_2(n))$ to an existing set of observations so as to greedily reduce expected error in the resultant OHS estimate.

We suggest a routine parametric algorithm (Algorithm~\ref{alg:parametric_approximate}) with estimation of asymptotic confidence intervals. Full details of theory, proofs and algorithms are given in Supplement~\ref{apx:parametric}.

\begin{algorithm}[h]
\begin{algorithmic}[1]
\State $\mathbf{n},\mathbf{k_2}, \boldsymbol{\sigma}^2 \gets$ some initial values $\mathbf{n}=\{n_1,\dots,n_m\}$ with $(\boldsymbol{k_2})_i\approx k_2(n_i)$, $(\boldsymbol{\sigma}^2)_i = \textrm{var}(\hat{k}_2(n_i))$\;
\While{$|\mathbf{n}| < $ total iterations}
\State Find $\tilde{n}$ which minimises expected OHS confidence interval width (\S\ref{apx:parametric}, eq.\ref{eq:parametric_nextn}), and add to $\mathbf{n}$ \;
\State Estimate $\hat{k}_2(\tilde{n}) \approx k_2(\tilde{n})$ \;
\State $\mathbf{n} \gets \{\mathbf{n} \cup \tilde{n}\}$, $\mathbf{k_2} \gets \left\{ \mathbf{k_2} \cup \hat{k}_2(\tilde{n})\right\}$, $\boldsymbol{\sigma}^2 \gets \left\{\boldsymbol{\sigma}^2 \cup \textrm{var}\left\{\hat{k}_2(\tilde{n})\right\}\right\}$ 
\EndWhile
\State \Return Re-estimate OHS $n_\star^\text{final}$ from $\mathbf{n},\mathbf{k_2},\boldsymbol{\sigma}$
\caption{Parametric OHS estimation overview}
\label{alg:parametric_approximate}
\end{algorithmic}
\end{algorithm}

We use the shorthands $\Theta=(N,k_1,\theta)$ and $\Theta_0=\mathbb{E}(\Theta)$. Consistency of Algorithm~\ref{alg:parametric_approximate} depends on whether $\mathbf{n}$ eventually contains enough elements of sufficient multiplicity to estimate $\Theta_0$ consistently. 
Sampling some positive proportion of values of $\mathbf{n}$ randomly from $\{1,\dots,N\}$ guarantees that the multiplicity of all $n \in \mathbf{n}$ almost surely eventually exceeds any finite value, readily ensuring consistency. Finite-sample bias of $n_\star^\text{final}$ depends on $\nabla_{\Theta} n_\star$ and the variance of $\Theta$. 
See Supplementary Figure~\ref{supp_fig:partials_nstar} for typical forms of $\nabla_{\Theta} n_\star$.

\subsection{Semi-parametric (emulation) estimation of OHS} 
\label{sec:emulation}

Parametrization of $k_2(n)$ may be inappropriate if the learning curve of the risk score or the relation between the learning curve and $k_2(n)$ (from Section~\ref{sec:practicalities}) are complex~\citep{viering21}. We propose a second algorithm which is less reliant on assuming a parametric form for $k_2(n)$, using Bayesian optimisation~\citep{brochu2010tutorial}. 
We quantify the uncertainty in $k_2(n)$ through the construction of a Gaussian process emulator of $\ell$ . The prior mean function for this emulator takes a particular parametric form, but crucially can deviate from this prior function with the addition of data.

We take the $n$ corresponding to the minimum cost function value (for the evaluated points) to be our OHS estimate.  Values of $n$ at which to estimate $\ell(n)$ are selected using an `expected improvement' function $EI(\cdot)$, whereby if $EI(n)>\tau$ we roughly expect the minimum cost to decrease by at least $\tau$ from adding another estimate of $\ell(n)$ to our data. This also provides a natural stopping criterion.
An outline procedure is given in Algorithm~\ref{alg:emulation_approximate}. Further algorithm details and proofs of consistency are in Supplement~\ref{apx:emulation}. 

\begin{algorithm}[h]
\begin{algorithmic}[1]
\State $\mathbf{n},\mathbf{d} \gets$ some initial values $\mathbf{n}=\{n_1,\dots,n_m\}$ with $d_i=d(n_i) \approx \ell(n_i)$ \;
\State Estimate mean and variance of Gaussian process $\ell(n)$ and function $EI(n)$\; 
\While{$\max_{n \in \{1,\dots, N\}}\{EI(n)\} > \tau$}
\State $\tilde{n} \gets \arg \max_{n \in \{1,\dots, N\}} EI(n)$ \;
\State Estimate $d(\tilde{n}) \approx k_2(\tilde{n})$ \;
\State $\mathbf{n} \gets (\mathbf{n} \cup \tilde{n})$; $\mathbf{d} \gets (\mathbf{d} \cup d(\tilde{n}))$ \;
\State Re-estimate mean and variance of Gaussian process $\ell(n)$ and function $EI(n)$  \;
\EndWhile
\State \Return $n_\star^\text{final} = \arg \min_{n \in \{1,\dots, N\}\cap\mathbf{n}} \left\{\frac{1}{|\{j: n_j = n\}|}\sum_{j: n_j = n}{d}_j\right\}$
\caption{Emulation OHS estimation; minimum cost improvement $\tau$ }
\label{alg:emulation_approximate}
\end{algorithmic}
\end{algorithm}

\section{Simulations}
\label{sec:simulations}

\subsection{Simulation showing dominance of holdout set approach}
\label{sec:holdout_dominance_figure}

We briefly illustrate the theory described in Section~\ref{sec:general_setup} using simulated data, similar to our motivating example, which satisfiess Assumptions~\ref{asm:f0_difference}--\ref{asm:costs}, \ref{asm:pop_size} and the weaker form of Assumption~\ref{asm:usefulness} (Equation~\ref{eq:relaxed_usefulness}) with $\delta=10$, $s=1$, {no drift in $\mu_t$}, $\alpha_1 \approx 0.32$, $k_1 = 0.023$, $k_2^l = 0.038$, $b_u=b_l=1$ and $k_2^u = 0.22$ 
(details in Supplement~\ref{supp_sec:holdout_dominance_simulation}). Figure~\ref{fig:holdout_dominance} shows total costs accrued per sample during unit time periods of no-update (`none'), na\"{i}ve-update (`naive') and holdout-update (`H.S') at two holdout sizes over time. As drift occurs in $f_t$, the costs associated with the no-update strategy grow due to increasingly poor approximation of $f_t$, and the costs of the na\"{i}ve-update strategy increase dramatically due to intervention effects. The total costs of the holdout-set approaches remain low. Choice of the holdout set size aims to balance increased costs due to non-intervention in the holdout set (the `spikes') against inaccuracy in fitted scores after drift. We demonstrate the natural emergence of an OHS in a simulated context in Supplement~\ref{supp_sec:sim_example}. 

\begin{figure}[h] 
\centering
\includegraphics[width=0.75\textwidth,clip,trim=0cm 0cm 0cm 0cm]
  {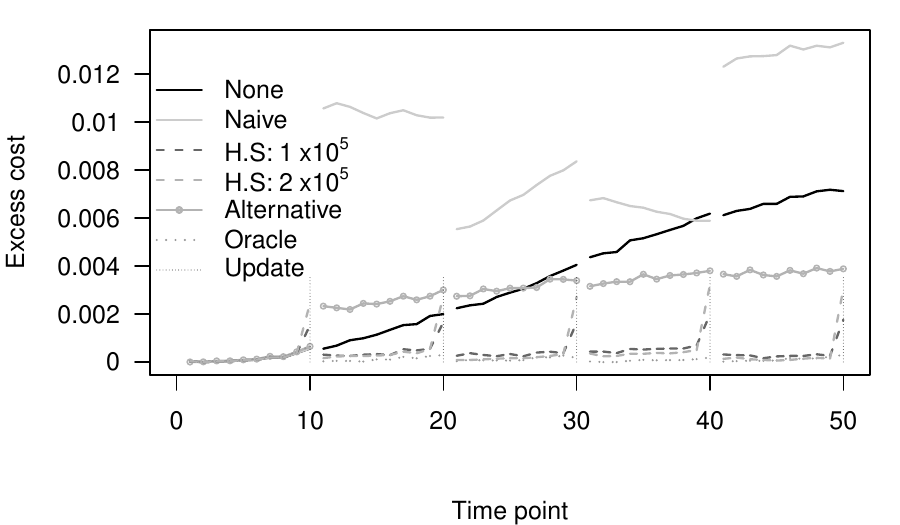}
\caption{Cost per sample per unit of time for no-update, na\"{i}ve-update, holdout-update, {oracle, and alternative-update (using a treatment indicator as a covariate)}  strategies for a simulated example with population drift in risk. {The costs associated with na\"{i}ve updating will decrease during an epoch if $f_t$ drifts towards the fitted risk score (that is, $\xi_t^2(f_t,\rho_t)$ decreases with $t$) and increase if it drifts away (that is, $\xi_t^2(f_t,\rho_t)$ increases with $t$).} {Supplementary Figure~\ref{supp_fig:cumcost} shows the cumulative cost over time}.}
\label{fig:holdout_dominance}
\end{figure}

\subsection{Comparison of parametric and emulation algorithms}
\label{sec:comparison_parametric_emulation}

In this section, we give circumstances in which one of Algorithms~\ref{alg:parametric_approximate} or \ref{alg:emulation_approximate} may be preferable to the other. We consider two versions of the function $k_2(n)$:
\begin{equation}
k_2^p(n) = a n^{-b} + c \hspace{10pt} \mbox{,   } \hspace{10pt}
k_2^{np}(n) = a n^{-b} + c + \frac{10^4}{\sqrt{2\pi}} \exp\left(-\frac{1}{2}\left(\frac{n- 4 \times 10^4}{8 \times 10^3}\right)^2\right) \,;\nonumber
\end{equation}
where: `p'/`np' denote `parametric assumptions satisfied/not satisfied', and $\theta=(a,b,c)=(10000, 1.2, 0.2)$.  We assume $N$ and $k_1$ are known to be $10^5$ and $0.4$ respectively. For emulation, we use a kernel width $\zeta=5000$ and variance $\sigma_u^2$ of $10^7$.

The function $k_2^{np}(n)$ exhibits `double-descent' behaviour (Supplementary Figures~\ref{supp_fig:k2_versions}, \ref{supp_fig:cost_versions}), which is possible for learning curves~\citep{viering21} but violates Assumptions~\ref{item:k2_decrease}, \ref{item:k2_2nd_der}. 

\sloppy
We firstly show the distribution of estimates of OHS using both algorithms when $k_2$ takes either form above. To fit $k_2$, we use 200 randomly chosen values from $\{1, \dots, N\}$ for $\mathbf{n}$, with values $\mathbf{k_2}$ independently sampled as $(\mathbf{k_2})_i \sim N\{k_2(n_i),\sigma_i^2\}$, where $\sigma_i\stackrel{\text{iid}}{\sim} U(0.001,0.02)$. Supplementary Figure~\ref{supp_fig:param_emul_comp} shows the distributions and medians of OHS estimates using the parametric and emulation algorithms in settings with parametric assumptions either satisfied or not. 

The results confirm expectations that the parametric OHS estimate is empirically unbiased and has less variance than the emulation estimate when parametric assumptions are satisfied, but is biased when they are not. Variance of OHS estimates using the emulation method is lower when parametric assumptions are not satisfied, because the true cost function has a sharper minimum in that case (see Supplementary Figure~\ref{supp_fig:cost_versions}). Since the cost function is `flat' around the minimum in the setting where parametric assumptions are satisfied (Supplementary Figure~\ref{supp_fig:cost_versions}), the consequences of the high variance of the semi-parametric (emulation) estimator are minimal, as the cost is similar across a range of values near the OHS.

We next examine the convergence rates of OHS estimates when sampling the `next' value of $n$, $\tilde{n}$,  greedily, using Equation~\ref{eq:parametric_nextn} for Algorithm~\ref{alg:parametric} or $EI$ for Algorithm~\ref{alg:emulation}, versus simply randomly selecting $\tilde{n}$ uniformly in $\{1,\dots,N\}$. This is shown in Figure~\ref{fig:nextpoint_comparison}, which depicts medians and OHS estimates at various sizes of $|\mathbf{n}|$ under the different methods for selecting $\tilde{n}$. 

Convergence is faster when next points are picked greedily rather than randomly and when using parametric estimates (though these are biased and inconsistent for $k_2^{np}$). This is highlighted by the smaller panels which show the root mean-square error between the total cost at the estimated optimal sizes and the total cost at the true OHS. 
In particular, observe that the non-parametric method shows bifurcation, detecting both local minima in the double descent setting, whilst the parametric method converges to a mid-point which is far from optimal in terms of total costs. All code can be found in \url{https://github.com/jamesliley/OptHoldoutSize_pipelines}.

  \begin{figure}[h] 

  \begin{subfigure}{\textwidth}
  \begin{center}  \begin{subfigure}{0.4\textwidth}
   \includegraphics[width=\textwidth,clip,trim=0cm 1cm 0cm 0cm]
      {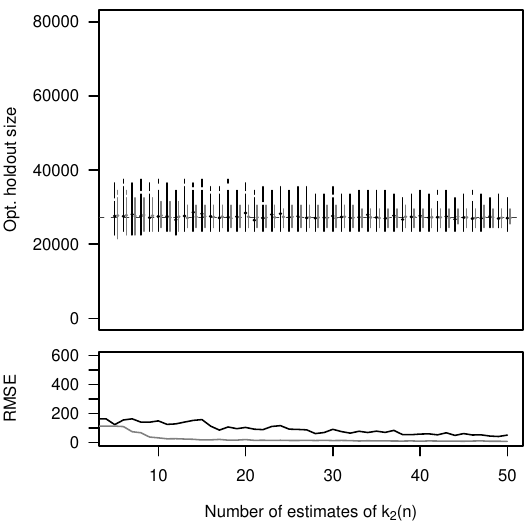}
  \end{subfigure}  
    \begin{subfigure}{0.38\textwidth}
   \includegraphics[width=\textwidth,clip,trim=0.5cm 1cm 0cm 0cm]
      {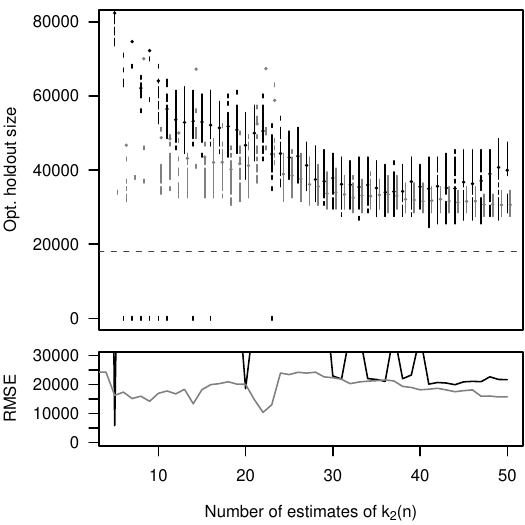}
  \end{subfigure}  
  \end{center}
  \end{subfigure}
  
  \begin{subfigure}{\textwidth}
  \begin{center}
    \begin{subfigure}{0.4\textwidth}
   \includegraphics[width=\textwidth,clip,trim=0cm 0cm 0cm 0cm]
      {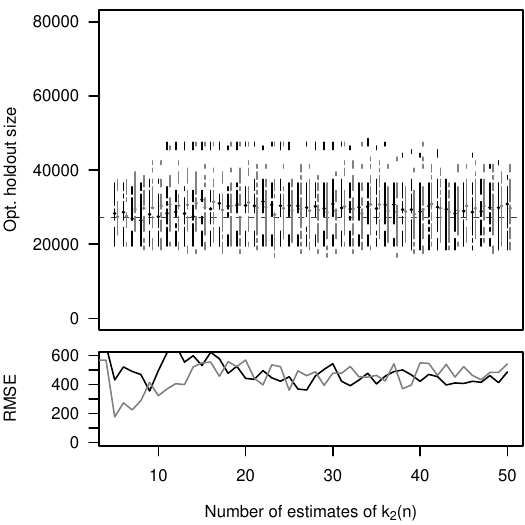}
  \end{subfigure}  
    \begin{subfigure}{0.38\textwidth}
   \includegraphics[width=\textwidth,clip,trim=0.5cm 0cm 0cm 0cm]
      {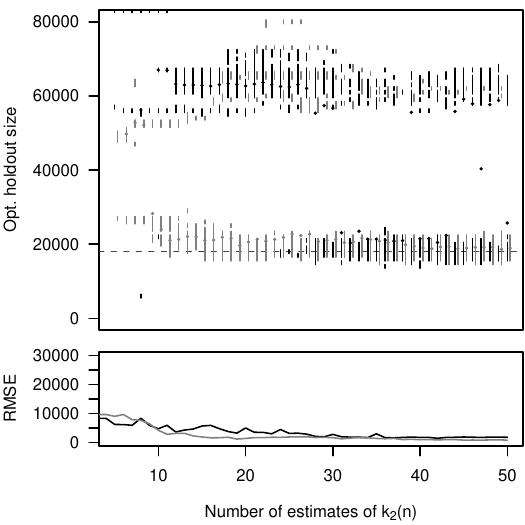}
  \end{subfigure}  
  \end{center}
  \end{subfigure}
  \caption{Convergence rates with parametric (top) and emulation (bottom) algorithms, using either a random (black) or greedy (gray) methods to select the next value, $\mathbf{n}$, with parametric assumptions satisfied (left) or unsatisfied (right). Simulations were run for 200 datasets from each underlying model. In larger panels, horizontal lines show true optimal holdout set (OHS) size; the OHS results from all simulation runs are discretised on a 1000-resolution grid, with vertical lines indicating OHS values that occurred in at least 2.5\% of simulations (i.e., 5 occurrences). Smaller panels show root mean-square error between total costs from simulations and minimal total cost under random/greedy methods. Note variable axis scaling under the two models. }
  \label{fig:nextpoint_comparison}
\end{figure}

\section{{Application to ASPRE}}
\label{sec:aspre}



We now return to %
{
our main motivating example. 
We are now in a position to address our main aim of developing an updating strategy for the ASPRE risk score. All code can be found in \url{https://github.com/jamesliley/OptHoldoutSize_pipelines}.
}

Supposing ASPRE is to be refitted every five years, the intervention set should include all individuals in the subsequent years before the model is refitted, and all individuals not used in the next refitting procedure. Suppose we refit ASPRE for use in a population of 5 million individuals, from which we have approximately 80,000 new pregnancies per year. 
The incidence of pregnancy per year is now
$(8 \times 10^4)/(5 \times 10^6) = 1/125$
so we have
$N \approx 5 \times 8 \times 10^4 = 400000$ (SE $\approx 1500$). We must now estimate $k_1$ and $k_2(\cdot)$ from published data. Although this method is not especially generalisable, $k_1$ and $k_2(\cdot)$ will generally be more easily estimable given raw data, which is not publicly available.

We presume a simple clinical action in which a fixed proportion $\pi$ of individuals at the highest assessed \textsc{pre} risk are treated with aspirin. We assume $\pi=10\% \approx 2707/25797$, the proportion of individuals assigned to the treatment group in~\cite{rolnik17b} due to having an estimated risk of \textsc{pre} $>1\%$. We assume that if untreated with aspirin, a {proportion $\pi_0$} of individuals designated to be `low-risk' (lowest 90\%) will develop \textsc{pre}, as will a proportion $\pi_1$ of individuals designated high-risk. 

To estimate $\pi_0$, $\pi_1$ and ultimately $k_1$, we considered the study reported in~\cite{ogorman17} assessing sensitivity and specificity of NICE and ACOG guidelines in assessing \textsc{pre} risk. In this study, 8775 indivduals were assessed, amongst which 239 developed \textsc{pre}, for an overall incidence of $239/8875 \approx 0.027$. We estimated the performance of a `baseline' estimator of \textsc{pre} risk (that is, in the absence of any ASPRE score) by linearly interpolating the points corresponding to `ACOG aspirin', `NICE' and `ACOG' on ROC curves in Figure 1 of that paper. On this basis, a baseline estimator identifying the 10\% of individuals at highest \textsc{pre} risk (approximately 800) would correspond to the point $(x,y)$ on the interpolated ROC curve with 
$239x + (8775-239)y=0.1 \times 8875$ 
%
which occurs at roughly a 20\% detection (true positive) rate and a 10\% false positive rate, close to that of the NICE guidelines. 

Since few women in the study were treated with aspirin, we assume that \textsc{pre} rates in the highest-10\% and lowest-10\% risk groups assessed by baseline risk (NICE) are untreated risk (that is, if not treated with aspirin). At the inferred true and false positive rates, we would expect a \textsc{pre} rate $\pi_0$ amongst the 10\% of women designated highest-risk by the NICE guidelines and $\pi_1$ amongst the 90\% designated lower risk, where 
\begin{align}
\pi_0 &\approx \frac{(1-\textrm{TPR})\times\textrm{(Num. \textsc{pre})}}{\textrm{Num. negative}} = \frac{0.8 \times 239}{0.9 \times 8875} \approx 0.024 \, ,
\nonumber \\ 
\pi_1 &\approx \frac{\textrm{TPR}\times\textrm{(Num. \textsc{pre})}}{\textrm{Num. positive}} = \frac{0.2 \times 239}{0.1 \times 8875} \approx 0.054 \, ,\nonumber
\end{align}
with standard errors 
$SE(\pi_1) \approx 
0.0076$ and $SE(\pi_0) \approx 
0.0017$. 
We denote by $\alpha$ the relative reduction in \textsc{pre} risk with aspirin treatment. Aspirin reduces \textsc{pre} risk to approximately $63\%$ (SE 0.09) of untreated risk ~\citep{rolnik17} so we take $\alpha=1-0.63 = 0.37$. Now, treating errors in $\pi_0$, $\pi_1$ and $\alpha$ as pairwise independent, we have 
\begin{equation}
k_1 = \pi_0(1-\pi) + \pi_1 \pi \alpha \approx 0.0235 \, ,
\end{equation}
with $SE(k_1) = SE\left(\pi_0 (1-\pi) + \pi_1 \pi \alpha\right) \approx 0.0016$. We estimate the population prevalence $\pi_\text{\textsc{pre}}$ of untreated \textsc{pre} as the frequency observed in the original ASPRE data: $\pi_\text{\textsc{pre}}=1426/57974 \approx 2.4\%$. Note that, although this is approximately equal to $\pi_0$, they are different quantities: $\pi_0$ is the population frequency of \textsc{pre} amongst individuals at the lowest 90\% risk by NICE guidelines. 

Denoting $\pi_1(n)$ as the untreated risk of \textsc{pre} in the top 10\% of individuals according to an ASPRE score trained on $n$ individuals (and $\pi_0(n)$ correspondingly), we note that it is equal to the sensitivity (or TPR) of the risk score at the level where proportion $\pi$ of individuals are designated high-risk. Thus for any training set size $n$, we have 
$\pi_0(n)=(\pi_\text{\textsc{pre}}-\pi \pi_1(n))/(1-\pi)$
%
so the average cost to an individual in the intervention set may be expressed in terms of $\pi_1(n)$:
\begin{equation}
k_2(n) = \pi_0(n)(1-\pi) + \pi_1(n)\pi\alpha = \pi_\text{\textsc{pre}}-\pi \pi_1(n) (1-\alpha) \, . \nonumber
\end{equation}
We denote `cost' as simply the number of cases of \textsc{pre} in a population, so total expected cost per individual under `baseline' treatment (clinical actions without the aid of a risk model) is
\begin{equation}
k_1 = \pi_0(1-\pi) + \pi_1 \pi \alpha \approx 0.02 \, ,
\end{equation}
with standard error approximately 0.001. Note that this is not equal to the untreated \textsc{pre} risk in the population, since some proportion of individuals are treated pre-emptively.

The data used to fit the initial ASPRE model could be used to estimate $\pi_1(n)$ and hence $k_2(n)$ for potential model updates. 
We do not have access to this dataset, but demonstrate estimation of a learning curve on synthetic data designed to resemble it. In this case, $k_2(n)$ is easy and fast to estimate, and is well-approximated by a power law, so we would favour use of Algorithm~\ref{alg:parametric}. In order to mimic a real example where such estimation is time consuming or costly we use both algorithms and restrict ourselves to use only $|\mathbf{n}|=120$ values of $n$, determined using either Algorithm~\ref{alg:parametric} or~\ref{alg:emulation}. For both algorithms, we assumed a power-law form $k_2\{n;\theta=(a,b,c)\}=a n^{-b}+c$. 

{Using the parametric algorithm, we found an OHS of 12684 (90\% CI 10811-14556), with minimum cost (expected cases over five years) of 8177. Using the emulation algorithm, we found an OHS of 13313 with an expected cost of 8164, with holdout sizes of 9210-17619 having a probability $>0.1$ of cost $<8164$.} Figure~\ref{fig:aspre} shows estimated cost functions, OHSs, and error using the two algorithms.
{
From our parametric approximation of $k_2(n)$, we estimate that if we use the suggested holdout set, the cost to a sample in the subsequent intervention set (that is, the PRE risk for a given pregnancy) is $k_2(12684)=2.034\%$. This compares to a risk of $k_1=2.354\%$ in the current holdout set, and an overall risk of $8177/(5 \times 10^6)=2.044\%$ across both sets. 
}

{
We illustrate the practical mean risk per patient in Figure~\ref{fig:risk_per_patient}. It is evident that the risk for patients in the holdout set ($k_1$, equivalent to use of a risk score fitted to 100 individuals), is only slightly higher than the risk to individuals in the intervention set, which is very close to the lowest possible risk and the average risk to all patients under the holdout set updating strategy (equivalent to a risk score fitted to around 6000 individuals).
}

\begin{figure}[h] 
\centering
\begin{subfigure}{0.24\textwidth}
\includegraphics[width=\textwidth,clip,trim=0cm 0.5cm 0cm 2cm]
    {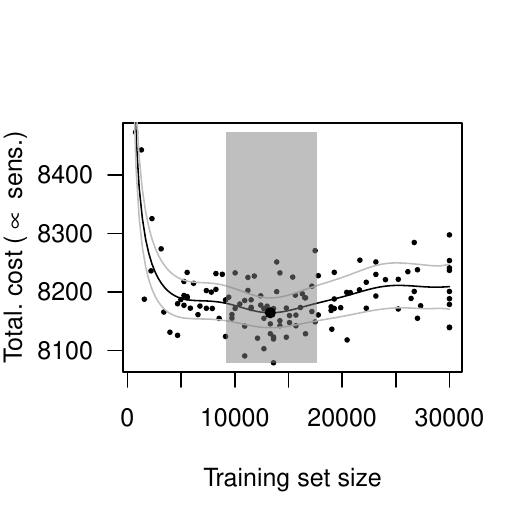}
\caption{}
\label{fig:aspre_parametric}
\end{subfigure}
\begin{subfigure}{0.24\textwidth}
\includegraphics[width=\textwidth,clip,trim=0cm 0.5cm 0cm 2cm]
    {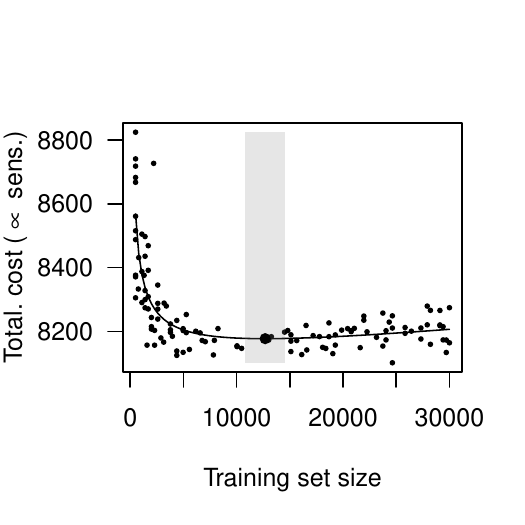}
\caption{}
\label{fig:aspre_emulation}
\end{subfigure}
\begin{subfigure}{0.24\textwidth}
\includegraphics[width=\textwidth,clip,trim=0cm 0.5cm 0cm 2cm]
    {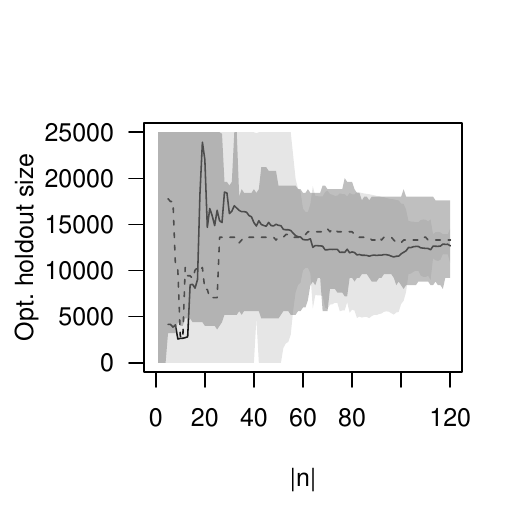}
\caption{}
\label{fig:aspre_track}
\end{subfigure}
\begin{subfigure}{0.24\textwidth}
\includegraphics[width=\textwidth,clip,trim=0cm 0.5cm 0cm 2cm]
    {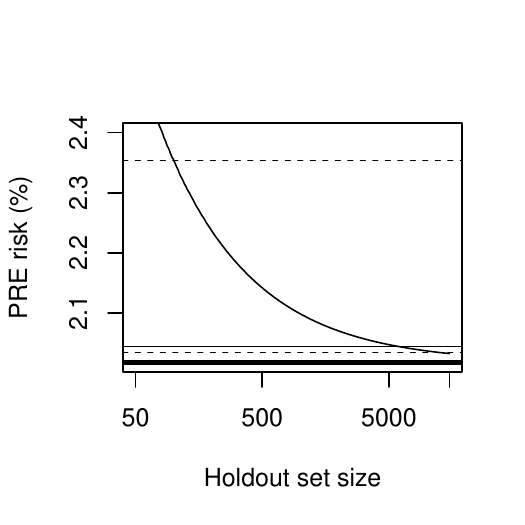}
\caption{}
\label{fig:risk_per_patient}
\end{subfigure}
\caption{
Estimation of cost functions (black lines), OHS (black dots), error using parametric (middle left) and emulation (leftmost; outer lines $\mu(n) + 3\sqrt{\Psi(n)}$) algorithms, change in estimated OHS and error with number of sample points $|\mathbf{n}|$ (middle right; solid lines: parametric, dashed lines: emulation), and overview of PRE risk per patient in final updating strategy (rightmost; solid black curve $k_2(n)$; heavy black line minimum possible risk $\min_n k_2(n)$; upper dashed line risk in holdout set; lower dashed line risk in intervention set; horizontal narrow solid line overall average risk). Note that the `best' points (black dots) to optimize parametric estimation are spread-out to estimate $\theta$ well, but for emulation they are clustered for accurate local approximation. Error measures for OHS in parametric and emulation algorithms (red/blue shaded respectively) have different meanings and are not comparable. 
}
\label{fig:aspre}
\end{figure}

{In summary, we recommend that the ASPRE score use a holdout set of around 12,000 pregnancies in order to train an updated model, under the use assumptions outlined above. In keeping with our discussions in Section~\ref{sec:holdout_set_motivation} and \ref{sec:cost_specification}, we recommend that these samples are held-out late in the five-year period during which a given iteration of the model is in use. For held-out samples, the decision of whether to prescribe aspirin should be based on the best estimate of medical practitioners. }

\section{Concluding remarks}

In this work we propose the use of a holdout set to safely update predictive models, and describe considerations in determining the optimal size of such a set. We establish theoretical properties of this optimal size under common conditions, and develop two algorithms for estimating it, evaluating their use in both a toy simulation and a real-life motivated simulation. The holdout set approach comprises a practical and simple approach to an important problem in practical applied statistical modelling and machine learning, which will be increasingly important as risk scores start to be used ever more routinely to prompt intervention in real-world applications. 

An appealing alternative for managing the effects of intervention without holdout sets is to attempt to explicitly infer parameters of an underlying causal structure~\citep{alaa18,sperrin19}. However, this approach cannot evade the difficulties we describe in Section~\ref{sec:holdout_set_motivation}: either we must be able to observe more detail (for instance, values of covariates after interventions which is often impractical) or make simplifying assumptions (such as absence of drift). In such relaxed settings, non-holdout-set options may allow a lower asymptotic and finite-sample cost than holdout set use. However, they will have higher asymptotic costs than holdout-set use in the more general setting considered here. A second potential non-holdout set option is to explicitly specify interventions which are to be made in response to risk scores, rather than leave them up to end-users~\citep{ochs19,liley21stacked}, but in typical complex settings in which we wish to use risk scores, we may wish to retain the autonomy of end users in decision-making. Indeed, prescriptive risk score based interventions would cause a model to fall under medical device regulation in the United Kingdom, so many risk scores there are developed for information only.

{
We believe that considerations for updating predictive models are under-studied in applied statistics. It is well-recognised that predictive models generally need to be updated~\cite{hippisley17,kansagara11,wallace14}, but often this is planned to be done by simply refitting a predictive model to observed covariates and outcome data, corresponding to what we term `na\"{i}ve updating'. We note that the problem with na\"{i}ve updating is not mentioned in the standard TRIPOD guideline~\citep{collins15}. In a separate paper detailing the updating of a predictive risk score for emergency admissions in Scotland~\citep{liley21medRxiv}, we recognised the shortcomings of na\"{i}ve updating, but in the absence of mature literature on safe updating methods or prior planning on managing updates, we were restricted in our choice of method. We proposed updating the risk score to the \emph{maximum} of a previous risk score and a refitted risk score. This option will tend to lead to \emph{overestimation} of risk, while avoiding \emph{underestimation} of risk, and was satisfactory for a single update, though untenable in the longer term (Supplement~\ref{supp_sec:best_of_two}).  
}

Our methods may be extended in several ways. We do not consider the possibility of combining information over training iterations, which could reduce the number of training samples needed for a given prediction error. {Nonetheless, retention of information is only useful up to a point: the main aim of using a holdout set is to correct for drift in $f_t$ between times $e-1$ and $e$, which is described by the difference $f_{e}(x)-f_{e-1}(x)$ and not assisted by even perfect knowledge of $f_{e-1}(x)$.} It is also possible that information from the intervention set could be used alongside the holdout set to partially infer the effect of interventions. We presume a setting in which a risk score is periodically updated, but continual or `online' updating is also used~\citep{delange21} and is also susceptible to intervention effects. A holdout set approach may also be usable in this setting. {An implicit assumption of our work is that drift in $\mu_t$, $f_t$ and $g_t$ are not influenced by our choice of strategy. Effectively, this means that the risk score affects the underlying system only through $g_t$, and does not affect the covariates or outcomes of samples for whom the risk score is not used. In complex settings such as medicine or finance, this assumption may be generally reasonable, but its relaxation is an important avenue for future research. For the ASPRE score, this corresponds to an expectation that treatment decisions made on the basis of the score will not causally influence the overall PRE incidence in the population (that is, $f_t$) or the characteristics of ultrasound scans or demographics (that is, $\mu_t$). 
}

Our simulations and theoretical findings show several non-obvious properties of the optimal holdout set size. We note that if the mean square error of the risk score decreases as $1/N$ (Assumption~\ref{asm:risk_score_convergence}), then the OHS increases as $N^{1/2}$, with the immediate consequence that the OHS is a vanishing proportion of $N$, for large enough $N$. Moreover, in practical settings, the true OHS is fairly small, with rapidly diminishing returns to increasing risk score accuracy (Section~\ref{sec:aspre}).
Interestingly, as demonstrated in Supplement~\ref{supp_sec:sim_example}, a more accurate risk score does not necessarily lead to lower OHS size.
Given Corollary~\ref{cor:linear_improvement} (and as seen in Figure~\ref{fig:aspre_track}), it is generally better to err on the higher end of the optimal holdout set size, since cost increases at most linearly, whereas it can increase faster for smaller holdout sets. 

We propose two algorithms for estimating an optimal size for a holdout set. The parametric method is simple and converges rapidly, but the use of a Gaussian process emulator requires fewer assumptions. We advise use of the emulator method if the risk score is fitted using complex methods which may not lead to readily parametrisable $k_2$. A reasonable option if parametrisability of $k_2$ is uncertain is to use both estimation algorithms, and favour the emulation method if results disagree.

We strongly suggest planning an updating strategy for a risk model \emph{before} it is deployed. This work illustrates one strategy in this direction and we hope stimulates both use of and extensions of such methods for safe predictive score updating.

\section{Acknowledgments}

The authors would like to thank the anonymous referees, an Associate Editor and the Editor for their constructive comments that improved the quality of this paper.

SH's contributions arose from an MSc dissertation for the MISCADA programme at Durham University.
We thank Catalina Vallejos, Sebastian Vollmer and Bilal Mateen for helpful discussion.

\section{Funding}

LJMA was partially supported by a Health Programme Fellowship at the Alan Turing Institute.
JL and LJMA were partially supported by Wave 1 of The UKRI Strategic Priorities Fund under the EPSRC Grant EP/T001569/1, particularly the `Health' theme within that grant and The Alan Turing Institute; and by Health Data Research UK, an initiative funded by UKRI, Department of Health and Social Care (England), the devolved administrations, and leading medical research charities;
SRE was funded by the EPSRC doctoral training partnership at Durham University, grant reference EP/R513039/1.

\section{Index of supplementary material}

\begin{itemize}
\item \textbf{Supplement~\ref{supp_sec:general_notation}-\ref{supp_sec:pre_estimate}}

Contains proofs of Theorems \ref{thm:holdout_asymptotic}, \ref{thm:alternative_asymptotic}, \ref{thm:oracle_comparison}, \ref{thm:ohs_exists}, \ref{thm:ohs_exists_weak}, and \ref{thm:ohs_cor}, discussion of estimates of $k_2(n)$, details of algorithms, details of simulations, and details of the analysis of the ASPRE risk score.

\item \textbf{Supplement~\ref{supp_sec:supplementary_figures}}

Supplementary Figures

\item \textbf{Supplement S11}

R scripts to reproduce figures and other output. Available at \url{https://github.com/jamesliley/OptHoldoutSize_pipelines}

\end{itemize}

\end{bibunit}

\clearpage

\clearpage

\begin{bibunit}[custom]

\title{Holdout sets for predictive model updating \\
Supplementary materials}
\maketitle

\setcounter{section}{0}
\setcounter{page}{1}
\renewcommand\thesection{S\arabic{section}}
\renewcommand\thetheorem{S\arabic{theorem}}
\renewcommand\thelemma{S\arabic{lemma}}
\renewcommand\thecorollary{S\arabic{corollary}}

\renewcommand\thefigure{S\arabic{section}.\arabic{subsection}.\arabic{figure}}

\textcolor{white}{\part{}} 
\clearpage
\parttoc 

\clearpage

\section{General notation}
\label{supp_sec:general_notation}

This supplement requires several sets of notation, which will be introduced as needed. However, we note some commonalities. Throughout this document we will take $X$ to generically mean `covariates' and $Y$ to mean `outcomes'.  The subscript $t$ will be taken to mean `time' in a continuous sense, and the subscript $e$ to mean `epoch', referring to consecutive episodes of time. The superscript $h$ will correspond to the holdout set, and $i$ to the intervention set. The number $N$ will refer to the total number of samples on which a risk score may be trained or used during an epoch, and $n$ to denote a holdout set size, usually taken to be variable. We denote the standard normal PDF and CDF by $\phi(\cdot)$, $\Phi(\cdot)$ respectively, the Bernoulli distribution with parameter $p$ as $\textrm{Bern}(p)$, and the Poisson distribution with parameter $\lambda$ as $\textrm{Pois}(\lambda)$.

\clearpage

\section{Proofs of Theorems \ref{thm:holdout_asymptotic}, \ref{thm:alternative_asymptotic} and \ref{thm:oracle_comparison}}
\label{supp_sec:proofs_holdout_dominance}

\subsection{Theorem~\ref{thm:holdout_asymptotic}}

\begin{reptheorem}{thm:holdout_asymptotic}
Suppose we use a holdout set with size $n_\star=\Theta(N^a)$, with $0<a<1$, an $s<1$ of size $s=\Theta(N^{a + \epsilon-1})$ for some $\epsilon$ with $0<\epsilon<1-a$, and an update frequency $\delta \leq 1$ which  may vary with $N$.
Under assumptions~\ref{asm:ft_lipschitz}, \ref{asm:risk_score_convergence}, \ref{asm:usefulness} and \ref{asm:pop_size},
we have
\begin{equation}
\mathbb{E}\left\{C_{(h)}[0,T]\right\} = \delta N T k_2 (2\alpha_1 + \alpha_1^2) +  \delta^{-1} O(N^a) + O(N^{a + \epsilon}) +  O(N^{1-a}) \, .\nonumber 
\end{equation}

\end{reptheorem}

\begin{proof}
We begin by establishing an inequality on the quantity
\begin{equation}
\mathbb{E}\{\xi_t^2(\rho_{f}^{n_*,e,s},f_t)\} \, ,
\end{equation}
where $t-e<\delta$ and the expectation is over the data used to fit $\rho_f^{n,e,s}$. 

As in assumption~\ref{asm:risk_score_convergence}, denote
$F(x)=\mathbb{E}_{t \sim U(e-s,e)} \{f_t(x)\}$ and note that, by assumption \ref{asm:ft_lipschitz}, we have 
\begin{equation}
f_e(x) - \alpha_1 s \leq F(x) \leq f_e(x) + \alpha_1 s \, . \label{eq:fe_bounds_lipschitz}
\end{equation}
To establish an upper bound, we use {(in order) assumptions~\ref{asm:ft_lipschitz}, 
inequality~\ref{eq:fe_bounds_lipschitz}, and assumption \ref{asm:risk_score_convergence}.} Taking these, along with noting that $|\rho(x)-f(x)|\leq 1$ for any $x,\rho$, we have:
\begin{align}
\mathbb{E}\left\{\xi_t^2(\rho_f^{n_*,e,s},f_t)\right\} &= \mathbb{E}\left\{\int \left(\rho_f^{n_*,e,s}(x) - f_t(x)\right)^2 d \mu_t\right\} \nonumber \\
&\leq \mathbb{E}\left\{\int \left(\left|\rho_f^{n_*,e,s}(x) - f_e(x)\right| + |f_t(x)-f_e(x)|\right)^2 d \mu_t\right\} \nonumber \\
&\leq \mathbb{E}\left\{\xi_t^2\left(\rho_f^{n_*,e,s},f_e\right)  + 2 \int \left|f_t(x) - f_e(x)\right| d \mu_t + \xi_t^2(f_t,f_e) \right\} \nonumber \\
&\leq \mathbb{E}\left\{\int \left(\rho_f^{n_*,e,s}(x) - f_e(x)\right)^2 d \mu_t  \right\} + 2\alpha_1 (t-e) + \alpha_1^2(t-e)^2 \nonumber \\
&\leq \mathbb{E}\left\{\int \left(\left|\rho_f^{n_*,e,s}(x) -F(x)\right| + \left|F(x) - f_e(x)\right|\right)^2 d \mu_t  \right\}  + 2\delta\alpha_1 + \delta^2\alpha_1^2 \nonumber \\
&\leq \mathbb{E}\left\{\int \left(\left|\rho_f^{n_*,e,s}(x) -F(x)\right| + \alpha_1 s \right)^2 d \mu_t  \right\}  + \delta(2\alpha_1 + \alpha_1^2) \nonumber \\
&\leq \mathbb{E}\left\{\xi_t^2 \left(\rho_f^{n_*,e,s},F\right) \right\} + (2\alpha_1 + \alpha_1^2)(\delta + s) \nonumber \\
&= O(n_*^{-1}) + (2\alpha_1 + \alpha_1^2)(\delta+s) \, . \label{eq:err_cost}
\end{align}
In any period $(e-\delta,e]$, the probability $p_*$ of at least $n_*$ samples being encountered in $(e-s,e)$, given $\frac{n_*}{Ns} = \Theta(N^{-\epsilon}) \to 0$ and $Ns = \Theta(N^{1-a + \epsilon}) \to \infty$, satisfies the (weak) condition:
\begin{equation}
p_* = P\left(\textrm{Pois}(Ns) \geq n_*\right) = 1-O(N^{-2}) \, .\label{eq:poisson} \end{equation}

Consider costs accrued in the period $(e-\delta,e]$ with $e>\delta$, under the holdout strategy. We encounter some total number of samples $n\sim \textrm{Pois}(N\delta)$. We assign at most $n_*$ samples to the holdout set, with a total cost of at most $k_1 n_*$ (as a consequence of Assumption A5). The remaining samples each accrue a cost proportional to~\ref{eq:err_cost}, as long as at least $n_*$ samples were observed in $(e-\delta-s,e]$. We thus have
\begin{align}
\mathbb{E}\left\{C_{(h)}[e-\delta,e]\right\} &\leq \underbrace{k_1 n_*}_{\substack{\text{Cost for held} \\ \text{-out samples}}} + \mathbb{E}\{n\} \bigg( \underbrace{(1-p_*) k_1}_{\substack{\text{Cost if $<n_*$}\\ \text{samples in} \\ \text{$(e-1-s,e]$}}} +  \underbrace{p_* \mathbb{E} \left\{ k_2\xi_t^2(\rho_f^{n_*,e,s},f_t)\right\}}_{\substack{\text{Cost for non-held-} \\ \text{out samples otherwise}}}\bigg) \label{eq:holdout_cost_expansion} \\
&\leq k_1 n_* + N\delta \left(O(N^{-2}) +  k_2\mathbb{E}\left\{\xi_t^2(\rho_f^{n_*,e,s},f_t)\right\}\right) \nonumber \\
&\leq k_1 n_* + O(\delta N^{-1}) +  \delta N O(n_*^{-1}) + \delta N k_2 (2\alpha_1 + \alpha_1^2)(\delta+s) \nonumber \\
&\leq \Theta(N^{a}) + \delta O(N^{1-a}) + \delta \Theta(Ns) + \delta^2 N k_2 (2\alpha_1 + \alpha_1^2)) \nonumber \\
&= O(N^a) + \delta O(N^{a + \epsilon}) + \delta O(N^{1-a}) + \delta^2 N k_2 (2\alpha_1 + \alpha_1^2)) \, . \label{eq:holdout_asymptotic_cost}
\end{align}
The total cost accrued over the $T/\delta$ total epochs is thus
\begin{align}
\mathbb{E}\left\{C_{(h)}[0,T]\right\} &= \frac{1}{\delta} \mathbb{E}\left\{C_{(h)}[e-\delta,e]\right\}\nonumber \\
&= \delta^{-1} O(N^a) + O(N^{a + \epsilon}) +  O(N^{1-a}) + \delta N T k_2 (2\alpha_1 + \alpha_1^2))  \label{eq:holdout_asymptotic_cost_total} \\
&= \delta^{-1} O(N^a) + O(N^{a + \epsilon}) +  O(N^{1-a}) + \delta O(N)  \, ,
\end{align}
where the bounds in $O(\cdot)$ depend only on $k_1,k_2,\alpha_1,T$. {We note that this result holds for \emph{any} fixed $T$ rather than only the $T$ in assumption~\ref{asm:f0_difference}.}
%

\end{proof}

\clearpage

\subsection{Theorem~\ref{thm:alternative_asymptotic}}

\begin{reptheorem}{thm:alternative_asymptotic}
Suppose we choose $s$ such that $s \to 0$ as $N \to \infty$. Under assumptions~\ref{asm:f0_difference}-\ref{asm:pop_size}, for 
sufficiently small $\delta$, we have for $(\textrm{strat}) \in \{(0),(n),(a)\}$:
\begin{equation}
\mathbb{E}\left\{C_{\textrm{(\textrm{strat})}}[0,T]\right\} = \Omega(N) \, .
\end{equation}
\end{reptheorem}

\begin{proof}

We will begin with a simple lemma which we will use repeatedly:
\begin{lemma}
\label{lem:lt_difference}
If, for all $x$, we have $0\leq f(x),g(x),h(x) \leq 1$, then
\begin{equation}
\xi_t^2(f,g) \geq \xi_t^2(f,h) + \xi_t^2(h,g) -2\sqrt{\xi_t^2(h,g)} \, .\nonumber
\end{equation}

\end{lemma}

\begin{proof}
We have
\begin{align}
\xi_t^2(f,g) &= \int \left(f(x)-g(x)\right) d\mu_t \nonumber \\
&\geq \int \left(|f(x)-h(x)|-|h(x)-g(x)|\right) d\mu_t \nonumber \\
&\geq \int \left(f(x)-h(x)\right)^2 d\mu_t + \int \left(h(x)-g(x)\right)^2 d\mu_t - 2\int |f(x)-h(x)||h(x)-g(x)| d\mu_t\nonumber \\
&\geq \xi_t^2(f,h) + \xi_t^2(h,g) - 2\int|h(x)-g(x)| d\mu_t\label{eq:step1} \\
&\geq \xi_t^2(f,h) + \xi_t^2(h,g) - 2\sqrt{\int \left(h(x)-g(x)\right)^2 d\mu_t} \label{eq:step2} \\
&\geq \xi_t^2(f,h) + \xi_t^2(g,h) - 2\sqrt{\xi_t^2(h,g)} \, , \nonumber
\end{align}
using the fact that $|f(x)-h(x)| \leq 1$ at step~\ref{eq:step1} and the Cauchy-Schwarz inequality at step~\ref{eq:step2}.
\end{proof}

We secondly prove a short lemma to show that we need not consider the no-update strategy separately from the alternative strategy:

\begin{lemma}
\label{lem:0isa}
Strategy $(0)$ is a special case of strategy $(a)$.

\end{lemma}

\begin{proof}

We note that from assumption~\ref{asm:f0_difference} {we have
\begin{equation}
\mathbb{E}_{t \sim U[0,T]}\{\xi_t^2(f_0,f_t)\} = \int_0^T\frac{1}{T} \xi_t^2(f_0,f_t)dt > 0 \, .\nonumber 
\end{equation}
}
As $N \to \infty$, since $s \to 0$, we have from assumption~\ref{asm:risk_score_convergence}:
\begin{equation}
\mathbb{E}\left\{\xi_0^2\left(\rho_f^{N,0,s},f_0\right)\right\} \to 0 \, ,
\end{equation}
so 
we have (using Lemma~\ref{lem:lt_difference})
\begin{align}
\lim_{N \to \infty} &\left( \mathbb{E}_{t \sim U[0,T]}\left\{\xi_t^2(\rho_f^{N,0,s},f_t)\right\} \right) 
\geq \lim_{N \to \infty} \left( \mathbb{E}_{t \sim U[0,T]}\left\{\xi_t^2(f_0,f_t) - 2\sqrt{\xi_t^2(\rho_f^{N,0,s},f_0)} \right\} \right) \nonumber \\
&= \mathbb{E}_{t \sim U[0,T]}\{\xi_t^2(f_0,f_t)\} \nonumber \\
&> 0 \, ,
\end{align}
where the expectation is also over data used to fit $\rho_f^{N,0,s}$. Hence the no-update strategy is in the `alternative' class of strategies, with $\rho_t=\rho_f^{N,0,s}$ for all $t$. 

\end{proof}

Recalling our definition of $b$ as
\begin{equation}
 \lim_{N \to \infty} \left(\mathbb{E}_{t \sim U[0,T], D}\left\{\xi_t^2(\rho_t,f_t)\right\}\right) = b > 0 \, , \nonumber 
\end{equation}
we simply choose any $\epsilon$ with $0<\epsilon<b$, so for large enough $N$, we have
\begin{align}
\mathbb{E}\left\{C_{(a)}[0,T]\right\} &= \mathbb{E}\{\textrm{Pois}(NT)\}\cdot k_2 \mathbb{E}_{t \sim U(0,T),D} \left\{\mathbb{E}\left\{\xi_t^2(\rho_t,f_t)\right\}\right\} \label{eq:alternative_expansion_1} \\
&\geq k_2 N T (b-\epsilon) \nonumber \\ 
&= \Omega(N) \label{eq:series_a_omega_n} 
\end{align}
as required, establishing the theorem for $\textrm{strat}=(0)$ and $\textrm{strat}=(a)$.

We now establish the rate of growth of the costs of the naive update strategy.
The essential idea is
\begin{enumerate}
\item We consider two consecutive time periods $[(i-1)\delta,i\delta)$ and $[i\delta,(i+1)\delta)$, with $i\geq 2$.
\item We introduce an `index' value $\Delta_0$ which is the similarity of the risk score $\rho_g^{N,(i-1)\delta,s}$ used during period $[(i-1)\delta,i\delta)$ to the function $f_{(i-1)\delta}$ governing risk at the start of that period. 
\item Given assumption~\ref{asm:ft_lipschitz}, we establish that $f_{(i-1)\delta-s}$ is not very different from $f_t$ with $t \in [(i-1)\delta,i\delta)$, so $\Delta_0$ is similar to the difference between the risk score and the function $f_t$ throughout the period $[(i-1)\delta,i\delta)$. We conclude that $\Delta_0$ governs the total cost accrued during time period $[(i-1)\delta,i\delta)$, in that the larger $\Delta_0$, the larger the total cost.
\item We then consider the similarity between $g$ and $f$ during the period $[i\delta-s,i\delta)$ during which the risk score $\rho_g^{N,i\delta,s}$ is fitted for use during period $[i\delta,(i+1)\delta)$.  Since the `cost' (the difference between $f_t$ and $g_t$) and `inaccuracy' (the difference between the risk score and $f_t$) are related by assumptions~\ref{asm:costs}, \ref{asm:usefulness}, the difference between $f_t$ and $g_t$ is also governed by $\Delta_0$, with a larger $\Delta_0$ corresponding to a \emph{smaller} difference. 
\item Since $\rho_g^{N,i\delta,s}$ (the risk score for use in time period $[i\delta, (i+1)\delta]$ is fitted to $g_t$ with $t \in [i\delta-s,i\delta)$, the similarity between $f_t$ and $g_t$ in this period is also the similarity between $f_t$ and the new risk score. We establish that $f_t$ is similar in time period $[i\delta-s,i\delta]$ and in time period $[i\delta,(i+1)\delta]$, so the difference between $f_t$ and $g_t$, and hence the difference between $f_t$ and the risk score is largely conserved.  Since a small $\Delta_0$ means a large difference between $f_t$ and $g_t$ for $t \in [i\delta-s,i\delta]$, it means a large difference between the risk score and $f_t$ for $t \in [i\delta, (i+1)\delta]$.
\item Thus a small $\Delta_0$ means a large cost in time period $[i\delta, (i+1)\delta]$ and a large $\Delta_0$ means a large cost in $[(i-1)\delta, i\delta]$. We show that there is a non-negligible cost accrued overall. When summed across all such time periods, this results in a $\Omega(N)$ contribution to overall cost.
\end{enumerate}
When using the naive update strategy, the costs during the time period $[i\delta,(i+1)\delta)$ depend on the similarity between $f_t$ and $g_t$ during the time period $[(i-1)\delta,i\delta)$. 
The risk score used during the time period $[(i-1)\delta,i\delta)$ under the naive update strategy is $\rho_g^{N,(i-1)\delta,s}$. We define
\begin{equation}
\Delta_0 \triangleq \xi_{(i-1)\delta - s}^2\left(\rho_g^{N,(i-1)\delta,s},f_{(i-1)\delta}\right) \, . \nonumber 
\end{equation}
Note that this is not an expectation; we will show a bound on the accrued costs which does not depend on $\Delta_0$. {Denote by $\alpha_2$ the Lipschitz constant of $\mu_t$ in assumption~\ref{asm:d_drift}. }
Given assumption~\ref{asm:ft_lipschitz}, we have, for $t \in [(i-1)\delta,i\delta)$ (using a tighter bound than Lemma~\ref{lem:lt_difference}):
\begin{align}
\xi_t^2 &\left(\rho_g^{N,(i-1)\delta,s},f_t\right) = \int \left(\rho_g^{N,(i-1)\delta,s}(x)-f_t(x)\right)^2 d\mu_t  \nonumber \\
&\leq \int \left(\left|\rho_g^{N,(i-1)\delta,s}(x)-f_{(i-1)\delta}(x)\right|+\left|f_{(i-1)\delta}(x)-f_t(x)\right|\right)^2 d\mu_t  \nonumber \\
&\leq \xi_t^2\left(\rho_g^{N,(i-1)\delta,s},f_{(i-1)\delta}\right) + 2 \int \left|f_{(i-1)\delta}(x)-f_t(x)\right| d\mu_t \nonumber \\
&\phantom{\geq} + \int \left(f_{(i-1)\delta}(x)-f_t(x)\right)^2 d\mu_t \nonumber \\
&\leq \xi_{(i-1)\delta-s}^2\left(\rho_g^{N,(i-1)\delta,s},f_{(i-1)\delta}\right) + \alpha_2(t - ((i-1)\delta-s)) \nonumber \\
&\phantom{\geq} + 2\alpha_1(t-(i-1)\delta) +\alpha_1^2 \delta^2 &&\text{Asm. \ref{asm:ft_lipschitz},\ref{asm:d_drift}} \nonumber \\
&\leq \Delta_0 + (\alpha_2+2\alpha_1)\delta + \alpha_1^2 \delta^2 +\alpha_2 s \nonumber\\
&= \Delta_0 + m_s\, . \label{eq:xit_upper_bound}
\end{align}
%
denoting, for brevity,
\begin{equation}
m_s = (\alpha_2+2\alpha_1)\delta + \alpha_1^2 \delta^2 +\alpha_2 s \, .
\end{equation}
Note that by choosing a large enough $N$ (and hence sufficiently small $s$) and a sufficiently small $\delta$ we  may ensure $m_s$ is arbitrarily small.

Using similar arguments, we have
\begin{align}
\xi_t^2 &\left(\rho_g^{N,(i-1)\delta,s},f_t\right) = \int \left(\rho_g^{N,(i-1)\delta,s}(x)-f_t(x)\right)^2 d\mu_t  \nonumber \\
&\geq \int \left(\left|\rho_g^{N,(i-1)\delta,s}(x)-f_{(i-1)\delta}(x)\right|-\left|f_{(i-1)\delta}(x)-f_t(x)\right|\right)^2 d\mu_t  \nonumber \\
&\geq \xi_t^2\left(\rho_g^{N,(i-1)\delta,s},f_{(i-1)\delta}\right) -2 \int \left|f_{(i-1)\delta}(x)-f_t(x)\right| d\mu_t \nonumber \\
&\geq \xi_{(i-1)\delta-s}^2\left(\rho_g^{N,(i-1)\delta,s},f_{(i-1)\delta}\right) - \alpha_2(t - ((i-1)\delta-s)) \nonumber \\
&\phantom{\geq} - 2\alpha_1(t-(i-1)\delta) &&\text{Asms. \ref{asm:ft_lipschitz},\ref{asm:d_drift}}\nonumber \\
&= \Delta_0 - (\alpha_2 + 2\alpha_1)(t-(i-1)\delta) - \alpha_2 s \nonumber \\
&\geq \Delta_0  - \delta(\alpha_2 + 2\alpha_1) - \alpha_2 s \nonumber \\
&\geq \Delta_0 - m_s \, . \label{eq:lt_delta0}
\end{align}
For any $t$ in the time period $[i\delta-s,i\delta)$, we have, by assumptions~\ref{asm:usefulness} and \ref{asm:costs}:
\begin{align}
k_1 - \mathbb{E}_{X \sim \mu_t} \left\{f_t(X)-g_t(X)\right\} &=  k_2 \xi_t^2\left(\rho_g^{N,(e-1)\delta,s},f_t\right) \, . \label{eq:cost_two_ways}
\end{align}
We now consider two cases. 

\begin{case} 
$\mathbb{E}_{X \sim \mu_{\tau}} \left\{f_{\tau}(X)-g_{\tau}(X)\right\}<0$ for some $\tau \in [i\delta-s,i\delta)$
\end{case}

Conceptually, in this case, use of the risk score $\rho_g^{N,(e-1)\delta,s}$ in fact makes the risk worse than would the use of no risk score at all at time $\tau$. In this case, we have, by assumption:
\begin{align}
\xi_{\tau}^2\left(\rho_g^{N,(e-1)\delta,s},f_{\tau}\right) > \frac{k_1}{k_2} \, , \nonumber
\end{align}
so for any $t \in [(i-1)\delta,i\delta)$:
\begin{align}
\xi_t^2\left(\rho_g^{N,(e-1)\delta,s},f_t\right) &\geq  \xi_t^2\left(\rho_g^{N,(e-1)\delta,s},f_{\tau}\right) - 2\sqrt{\xi_t^2\left(f_t,f_{\tau}\right)} \nonumber \\
&\geq \xi_{\tau}^2\left(\rho_g^{N,(e-1)\delta,s},f_{\tau}\right) - \alpha_2|t-\tau| - 2\alpha_1|t-\tau| \nonumber \\
&\geq \frac{k_1}{k_2} - (\alpha_2 + 2\alpha_1)\delta \, , \nonumber
\end{align}
and the total expected cost accrued over the time period $[(i-1)\delta,(i+1)\delta]$ (during which we encounter $\textrm{Pois}(2N\delta)$ samples) is:
\begin{align}
\mathbb{E}\left\{ C_{(n)} [(i-1)\delta,(i+1)\delta]\right\} &\geq \mathbb{E}\left\{ C_{(n)} [(i-1)\delta,i\delta]\right\} \label{eq:cost_expansion_naive_1} \\
&= \mathbb{E}\left\{\textrm{Pois}(N\delta)\right\} k_2 \mathbb{E}_{t \sim U((i-1)\delta,i\delta)}\left\{\xi_t^2\left(\rho_g^{N,(e-1)\delta,s},f_t\right)\right\} \nonumber \\
&\geq  N\delta \left(k_1 - k_2(\alpha_2 + 2\alpha_1)\delta\right) \, , \nonumber
\end{align}
and for any $\epsilon$ we may choose $\delta$ dependent only on $\alpha_2,\alpha_1,k_2$ sufficiently small that
\begin{align}
\mathbb{E}\left\{ C_{(n)} [(i-1)\delta,(i+1)\delta]\right\} &\geq   N\delta \left(k_1 - \epsilon \right) \, . \label{eq:cost_case1}
\end{align}

\begin{case}
$\mathbb{E}_{X \sim \mu_t} \left\{f_t(X)-g_t(X)\right\}>0$ for all $t \in [i\delta-s,i\delta)$.
\end{case}

The risk score used during the period $[i\delta,(i+1)\delta)$ is $\rho_g^{N,i\delta,s}$. Denote $G(x) = \mathbb{E}_{t\sim U(i\delta-s,i\delta)}\{g_t(x)\}$ (so $\rho_g^{N,i\delta,s}(x)$ estimates $G(x)$). We now have: 
\begin{align}
\xi_{i\delta}^2\left(f_{i\delta},G\right) &= \int \left(f_{i\delta}(x)-G(X)\right)^2 d \mu_{i\delta} \nonumber \\
&\geq \left(\int \left|f_{i\delta}(x)-G(x)\right| d \mu_{i\delta}\right)^2\nonumber \\
&\geq \left(\int \left(f_{i\delta}(x)-G(x)\right) d \mu_{i\delta}\right)^2\nonumber \\
&= \left(\int \frac{1}{s}\int_{i\delta-s}^{i\delta} (f_{i\delta}(x)-g_t(x))dt \, d \mu_{i\delta}\right)^2\nonumber \\
&= \left(\frac{1}{s}\int_{i\delta-s}^{i\delta} \left(\int (f_t(x)-g_t(x))d \mu_{i\delta} + \int (f_{i\delta}(x)-f_{t}(x))d \mu_{i\delta}\right) dt \right)^2\nonumber \\
&\geq \left(\frac{1}{s}\int_{i\delta-s}^{i\delta} \max\left(0,\mathbb{E}_{X \sim \mu_{i\delta}} \left\{f_t(x)-g_t(x)\right\} - \alpha_1\delta\right) dt \right)^2 \nonumber &&\text{Asm. \ref{asm:ft_lipschitz}}\\
&\geq \left(\frac{1}{s}\int_{i\delta-s}^{i\delta} \max\left(0,\mathbb{E}_{X \sim \mu_t} \left\{f_t(x)-g_t(x)\right\} -\alpha_2 s - \alpha_1\delta\right) dt \right)^2 \nonumber &&\text{Asm. \ref{asm:d_drift}}\\
&\geq \left(\frac{1}{s}\int_{i\delta-s}^{i\delta} \max\left(0,k_1 - k_2 \xi_t^2\left(\rho_g^{N,(e-1)\delta,s},f_t\right) - \alpha_2 s - \alpha_1\delta\right) dt \right)^2 \nonumber &&\text{Asm. \ref{asm:costs}}\\
&\geq \max\left(0,k_1 - k_2 \left(\Delta_0 + (\alpha_2+2\alpha_1)\delta + \alpha_1^2 \delta^2 +\alpha_2 s
\right) - \alpha_2 s - \alpha_1\delta\right)^2 &&\text{Ineq. \ref{eq:xit_upper_bound}}\nonumber\\
&\geq \max\left(0,k_1 - k_2 \left(\Delta_0 + m_s
\right) - m_s \right)^2\, . \label{eq:xibound_case2}
\end{align}
We consider the expectation over the data used to fit $\rho_g^{N,i\delta,s}$ to note that for $t \in [i\delta,(i+1)\delta)$, using Lemma~\ref{lem:lt_difference}:
\begin{align}
\mathbb{E}\left\{\xi_{t}^2(\rho_g^{N,i\delta,s},f_t)\right\} &\geq \xi_t^2(G,f_t) - 2\sqrt{\xi_t^2\left(\rho_g^{N,i\delta,s},G\right)} \nonumber \\
&\geq \xi_t^2(G,f_t) - O\left(N^{-\frac{1}{2}}\right)\nonumber \\
&\geq \xi_t^2(G,f_{i\delta}) - 2\sqrt{\xi_t^2(f_t,f_{i\delta})} - O\left(N^{-\frac{1}{2}}\right)\nonumber \\
&\geq \xi_t^2(G,f_{i\delta}) - 2\alpha_1(t-i\delta) - O\left(N^{-\frac{1}{2}}\right) &&\text{Asm. \ref{asm:ft_lipschitz}}\nonumber \\
&\geq \xi_{i\delta}^2(G,f_{i\delta}) - \alpha_2(t-i\delta) - 2\alpha_1(t-i\delta) - O\left(N^{-\frac{1}{2}}\right) &&\text{Asm. \ref{asm:d_drift}} \nonumber \\
&\geq \xi_{i\delta}^2(G,f_{i\delta}) - \delta(\alpha_2+2\alpha_1) - O\left(N^{-\frac{1}{2}}\right) \, . \nonumber
\end{align}
The expected total cost accrued during the period $[(i-1)\delta,i\delta)$ (during which we encounter $\textrm{Pois}(N\delta)$ samples) is thus, from expression~\ref{eq:lt_delta0}:
\begin{align}
\mathbb{E}\left\{ C_{(n)}[(i-1)\delta,i\delta] \right\} &= \mathbb{E}\left\{\textrm{Pois}(N\delta)\right\} k_2 \mathbb{E}_{t\sim U((i-1)\delta,i\delta)}\left\{\xi_t^2\left(\rho_g^{N,(i-1)\delta,s},f_t\right) \right\}\label{eq:cost_expansion_naive_case2_1} \\
&\geq  N\delta k_2 \left(\Delta_0  - \delta(\alpha_2 + 2\alpha_1) - \alpha_2 s\right)\nonumber \\
&\geq  N\delta k_2 \left(\Delta_0  - m_s \right) \, . \nonumber
\end{align}
The total cost accrued during the period $[i\delta,(i+1)\delta)$ is, from expression~\ref{eq:xibound_case2}
\begin{align}
\mathbb{E}\left\{ C_{(n)} [i\delta,(i+1)\delta]\right\} &= \mathbb{E}\left\{\textrm{Pois}(N\delta)\right\} k_2 \mathbb{E}_{t\sim U(i\delta,(i+1)\delta)}\left\{\xi_t^2\left(\rho_g^{N,i\delta,s},f_t\right) \right\} \label{eq:cost_expansion_naive_case2_2}\\
&\geq  N\delta k_2 \left(k_1-k_2(\Delta_0 - m_s)-m_s\right)^2 \, , \nonumber
\end{align}
and hence the total cost over both periods is
\begin{align}
\mathbb{E}\left\{ C_{(n)} [(i-1)\delta,(i+1)\delta]\right\} &\geq  N\delta k_2 \left(\Delta_0  - m_s + (k_1-k_2(\Delta_0 - m_s)-m_s)^2\right) \, . \nonumber
\end{align}
For any $\epsilon$ we may choose $\delta$ (dependent only on $k_1$, $k_2$, $\alpha_2$, $\alpha_1$) sufficiently small that for large enough $N$:
\begin{align}
\mathbb{E}\left\{ C_{(n)} [(i-1)\delta,(i+1)\delta]\right\} &\geq N\delta \left( k_2\min_{0\leq \Delta_0 \leq 1} \left(\Delta_0 + (k_1-k_2\Delta_0)^2\right) - \epsilon \right) \, . \nonumber 
\end{align}
Recalling expression~\ref{eq:cost_case1} for the earlier case, we denote 
\begin{equation}
c_{(n)}=\min\left(k_2 \min_{0\leq \Delta_0 \leq 1} \left(\Delta_0 + (k_1-k_2\Delta_0)^2\right), k_1\right) > 0 \, ,\label{eq:cndef}
\end{equation}
so, in either case:
\begin{align}
\mathbb{E}\left\{ C_{(n)} [(i-1)\delta,(i+1)\delta]\right\} &\geq N\delta \left( c_{(n)} - \epsilon \right) \, .\nonumber 
\end{align}
We now finally consider the cost accrued over the entire time period $[0,T]$, where $T>2\delta$. We have
\begin{align}
\mathbb{E}\left\{C_{(n)}[0,T]\right\} &= \mathbb{E}\left\{\sum_{i=0}^{T/\delta - 1} C_{(n)}[i\delta,(i+1)\delta]\right\} \nonumber \\
&\geq \frac{1}{2} \mathbb{E}\left\{\sum_{i=1}^{T/\delta - 1} C_{(n)}[(i-1)\delta,(i+1)\delta]\right\} \nonumber \\
&\geq \frac{1}{2} \left(\frac{T}{\delta}-1\right)N\delta(c_{(n)}-\epsilon) \nonumber \\
&=\Omega(N) \nonumber
\end{align}
as required.


\end{proof}

\clearpage

\subsection{Theorem~\ref{thm:oracle_comparison}}

\begin{reptheorem}{thm:oracle_comparison}
Consider use of each strategy in parallel with an `oracle' procedure. Under assumptions~\ref{asm:f0_difference}-\ref{asm:pop_size}, with sufficiently small fixed $\delta$, holdout set size $n_*=\Theta(N^{2/3})$, and $s<1$ of size $s=\Theta(N^{\epsilon-1/3})$ for some $\epsilon$ with $0<\epsilon<1/3$, we have for $(\textrm{strat}) \in \{(0),(n),(a)\}$:
\begin{equation}
\lim_{N \to \infty}\left( \frac{\mathbb{E}\left\{C_{(h)}[0,T]\right\}}{\mathbb{E}\left\{C_{(o)}[0,T]\right\}}\right) = 1\textrm{, and } 
\lim_{N \to \infty}\left(\frac{\mathbb{E}\left\{C_{(\textrm{strat})}[0,T]\right\}}{\mathbb{E}\left\{C_{(o)}[0,T]\right\}}\right) > 1 \hspace{5pt} \left(=\Omega(\delta^{-2})\right) \, .\nonumber
\end{equation}
\end{reptheorem}

\begin{proof}

%

We begin with the following lemma:

\begin{lemma}
\label{lemma:bigT}
For the $T$ in assumption~\ref{asm:f0_difference}:
\begin{equation}
I(T) \triangleq \sum_{i=0}^{T/\delta - 1} \mathbb{E}_{t \sim U(i\delta,(i+1)\delta)}\left\{\xi_t^2(f_{i\delta},f_t)\right\} > 0 \, .
\end{equation}
\end{lemma}

\begin{proof}
{We employ assumptions~\ref{asm:f0_difference}-~\ref{asm:ft_lipschitz}. As above, we denote $\alpha_2$ as the Lipschitz constant of $\mu_t$. 
We have $\int_0^T \xi_t^2(f_t,f_0) dt > 0$. There must be some $u<T$ with $j\delta < u < (j+1) \delta$ such that }
\begin{equation}
d \triangleq \xi_{u}^2(f_{u},f_{j\delta}) dt > 0 \, . \nonumber
\end{equation}
Let $\tau$ be a number satisfying
\begin{equation}
0 < \tau < \min\left(\frac{d}{\alpha_2+2\alpha_1},\frac{u}{j\delta}\right)
\end{equation}
so $d - (\alpha_2 + 2\alpha_1) \tau > 0$ and $u-\tau>j\delta$. We now have
\begin{align}
I(T) &\geq \mathbb{E}_{t \sim U(j\delta,(j+1)\delta)}\left\{\xi_t^2(f_{j\delta},f_t)\right\} \nonumber \\
&= \int_{j\delta}^{(j+1)\delta} \frac{1}{\delta} \xi_t^2(f_{j\delta},f_t) dt \nonumber \\
&\geq \frac{1}{\delta} \int_{u-\tau}^{u}  \xi_t^2(f_{j\delta},f_t) dt \nonumber \\
&= \frac{1}{\delta} \int_{u-\tau}^{u}  \int \left(f_t(x)-f_{j\delta}(x)\right)^2 d\mu_t dt \nonumber \\
&\geq \frac{1}{\delta} \int_{u-\tau}^{u}  \left(\int \left(|f_u(x) - f_{j\delta}(x)| - |f_t(x)-f_u(x)|\right)^2 d\mu_u - \alpha_2(u-t)\right) dt \nonumber \\
&\geq \frac{1}{\delta} \int_{u-\tau}^{u}  \left( \xi_t^2(f_u,f_{j\delta}) + \xi_u^2(f_t,f_u) - 2\sup_x(|f_t(x)-f_u(x)|) \right) - \alpha_2(u-t) dt \nonumber \\
&\geq \frac{1}{\delta} \int_{u-\tau}^{u}  \left( \xi_u^2(f_u,f_{j\delta})  - \alpha_2(u-t)  - 2\alpha_1(u-t)\right) dt \nonumber \\
&\geq \frac{1}{\delta}(d-2\alpha_2\tau - \alpha_1\tau) \nonumber \\
&>0 \, .\nonumber
\end{align}

\end{proof}

We firstly consider $C_{(h)}[0,T]$. 
We have, 
by Lemma~\ref{lemma:bigT}, and recalling expressions~\ref{eq:poisson} and~\ref{eq:holdout_cost_expansion}:
\begin{align}
 &\lim_{N \to \infty} \left(\frac{\mathbb{E}\left\{C_{(h)}[0,T]\right\}}{\mathbb{E}\left\{C_{(o)}[0,T]\right\}}\right) =  \lim_{N \to \infty} \left( \frac{\sum_{i=0}^{T/\delta - 1} C_{(h)}[i\delta,(i+1)\delta]}{\sum_{i=0}^{T/\delta - 1} C_{(o)}[i\delta,(i+1)\delta]} \right) \nonumber \\
 &= \lim_{N \to \infty} \left( \frac{\sum_{i=0}^{T/\delta - 1} \left(k_1 n_* + N\delta \left(p_* +  (1-p_*)k_2\mathbb{E}_{t \sim U(i\delta,(i+1)\delta)}\left\{\xi_t^2(\rho_f^{n_*,i\delta,s},f_t)\right\}\right)\right) }{\sum_{i=0}^{T/\delta - 1} \left(N\delta k_2\mathbb{E}_{t \sim U(i\delta,(i+1)\delta)}\left\{\xi_t^2(\rho_f^{N,i\delta,s},f_t)\right\}\right)} \right) \nonumber \\
 &= \lim_{N \to \infty} \left( \frac{\frac{T}{\delta} k_1 \Theta \left(N^{\frac{1}{3}}\right) + N\delta k_2  \sum_{i=0}^{T/\delta - 1} \left(  \mathbb{E}_{t \sim U(i\delta,(i+1)\delta)}\left\{\xi_t^2(f_{i\delta},f_t)\right\}\right) }{N\delta k_2 \sum_{i=0}^{T/\delta - 1} \mathbb{E}_{t \sim U(i\delta,(i+1)\delta)}\left\{\xi_t^2(f_{i\delta},f_t)\right\}} \right) \nonumber \\
 &= \lim_{N \to \infty} \left( \frac{\frac{T}{\delta} k_1 \Theta \left(N^{\frac{1}{3}}\right) + N\delta k_2  I(T) }{N\delta k_2 I(T)}\right) \nonumber \\
 &= 1 \, .\label{eq:oracle_holdout_derivation}
\end{align}

We next consider $C_{(n)}[0,T]$ and $C_{(a)}[0,T]$, recalling from Lemma~\ref{lem:0isa} that we need only consider the latter. We note that
\begin{align}
I(T) &= \sum_{i=0}^{T/\delta - 1} \int_{i\delta}^{(i+1)\delta} \xi_t^2(f_{i\delta},f_t) \frac{1}{\delta} dt \nonumber \\
&\leq \frac{1}{\delta} \sum_{i=0}^{T/\delta - 1} \int_{i\delta}^{(i+1)\delta} \int (f_{t}-f_{i\delta})^2 d\mu_t  dt  \nonumber \\
&\leq \frac{1}{\delta} \sum_{i=0}^{T/\delta - 1} \int_{i\delta}^{(i+1)\delta} \int (\alpha_1\delta)^2 d\mu_t  dt  \nonumber \\
&=T\delta \alpha_1^2 \, . \label{eq:it_upper_bound}
\end{align}
Recalling our definition of $b$ as
\begin{equation}
 \lim_{N \to \infty} \left(\mathbb{E}_{t \sim U[0,T], D}\left\{\xi_t^2(\rho_t,f_t)\right\}\right) = b > 0 \, , \nonumber 
\end{equation}
we choose any $\epsilon$ with $0<\epsilon<b$, so for large enough $N$, we have for small enough $\delta < \frac{\sqrt{b}}{\alpha_1}$:
\begin{align}
 \lim_{N \to \infty} \left(\frac{\mathbb{E}\left\{C_{(a)}[0,T]\right\}}{\mathbb{E}\left\{C_{(o)}[0,T]\right\}}\right) &=\lim_{N \to \infty} \left( \frac{N T k_2 \mathbb{E}_{t \sim U(0,T)} \left\{\xi_t^2(\rho_t,f_t)\right\}}{N \delta k_2 I(T)} \right) \nonumber \\
 &\geq\lim_{N \to \infty} \left( \frac{ T (b-\epsilon)}{T \delta^2 \alpha_1^2} \right) \nonumber \\
&=\Omega(\delta^{-2}) \nonumber \\
&>1 \, . \nonumber
\end{align}
We next consider $C_{(n)}[0,T]$. We recall from expression~\ref{eq:cndef} in the proof of Theorem~\ref{thm:alternative_asymptotic} that 
for any $\epsilon>0$ we may choose $\delta$ (dependent only on $k_1$, $k_2$, $\alpha_2$, $\alpha_1$) sufficiently small that for large enough $N$ we have:
\begin{align}
\mathbb{E}\left\{ C_{(n)} [(i-1)\delta,(i+1)\delta]\right\} &\geq N\delta \left( c_{(n)} - \epsilon \right) \, ,  \nonumber
\end{align}
where $c_{(n)}$ does not depend on $\delta$. Thus, for sufficiently small $\delta$:
\begin{align}
 &\lim_{N \to \infty} \left(\frac{\mathbb{E}\left\{C_{(n)}[0,T]\right\}}{\mathbb{E}\left\{C_{(o)}[0,T]\right\}}\right) \geq \lim_{N \to \infty} \left( \frac{\frac{1}{2}\sum_{i=1}^{T/\delta - 1} C_{(n)}[(i-1)\delta,(i+1)\delta]}{\sum_{i=0}^{T/\delta - 1} C_{(o)}[i\delta,(i+1)\delta]} \right) \label{eq:cost_oracle_expansion_naive} \\
 &\geq \lim_{N \to \infty} \left(\frac{ \frac{1}{2} N \delta \left(\frac{T}{\delta}-1\right)(c_{(n)}-\epsilon)}{N\delta k_2 I(T)} \right) \nonumber \\
 &\geq \lim_{N \to \infty} \left(\frac{ (T-\delta)(c_{(n)}-\epsilon)}{ 2 k_2 T \delta^2 \alpha_1^2} \right) \nonumber \\
&=\Omega(\delta^{-2}) \nonumber \\
&>1 \nonumber
\end{align}
as required.

\end{proof}

\clearpage

\subsection{Relaxation of assumption~\ref{asm:usefulness}}
\label{supp_sec:relaxation_usefulness}

Our assumption~\ref{asm:usefulness} is prescriptive, necessitating a strict relationship between a particular conception of `accuracy' of the risk score (as measured by $\xi_t^2$) and the amount by which we can reduce $f_t$ in expectation (assumption~\ref{asm:costs}). 
{
Assumption~\ref{asm:usefulness} allows some flexibility in the form of $g_t$: given $f_t$, an explicit example of a function $g_t(x,\rho)$ satisfying assumptions~\ref{asm:costs} and~\ref{asm:usefulness} can be constructed by considering constants $p_1,p_2>0$ and functions $q_1(x)$, $q_2(x)$ satisfying $\mathbb{E}_{X \sim \mu_t}\{q_1(X)\}=\mathbb{E}_{X \sim \mu_t}\{q_2(X)\}=0$, and considering any form of $g_t(x,\rho)$ satisfying:
\begin{equation}
g_t(x,\rho)=\underbrace{f_t(x)}_{\textrm{T1}} - \underbrace{p_1 + q_1(x)}_{\textrm{T2}} + \underbrace{ \xi_t^2(\rho,f_t)(p_2+q_2(x))}_{\textrm{T3}} \, , \nonumber
\end{equation}
where term T1 is the underlying risk (in the absence of a risk score), term T2 is the maximum potential reduction in risk (covariate-dependent, through $q_1(\cdot)$) and term T3 arises due to imperfect prediction: for instance, representing under-treatment or over-treatment due to $\rho(\cdot)$ differing from $f_t(\cdot)$, in a covariate-dependent manner through $q_2(\cdot)$. This can be seen to satisfy assumption~\ref{asm:costs} as:
\begin{align}
k_1 &= \mathbb{E}_{X \sim \mu_t} \left\{f_t(X) - g_t(X,f_t)\right\} \nonumber \\
&= \mathbb{E}_{X \sim \mu_t} \left\{p_1 - q_1(X) + \xi_t^2(f_t,f_t)(p_2+q_2(x))\right\} \nonumber \\
&= p_1 > 0 \, , \nonumber
\end{align}
since $\xi_t^2(f_t,f_t)=0$, and since:
\begin{align}
c_t(\rho_t) &= \mathbb{E}_{X \sim \mu_t} \left\{g_t(X,\rho_t) - g_t(X,f_t)\right\} \nonumber \\
&= \mathbb{E}_{X \sim \mu_t} \left\{\xi^2(\rho_t,f_t)\left(p_2+q_2(X)\right) - \xi^2(f_t,f_t)\left(p_2+q_2(X)\right) \right\} \nonumber \\
&= \xi^2(\rho_t,f_t)\mathbb{E}_{X \sim \mu_t} \left\{p_2+q_2(X) \right\} \nonumber \\
&= p_2 \xi^2(\rho_t,f_t) \, ,  \nonumber
\end{align}
the function $g_t$ satisfies assumption~\ref{asm:usefulness} with $k_2=p_2$.
}

We may substantially weaken assumption~\ref{asm:usefulness} and maintain our results, although the statement of Theorem~\ref{thm:holdout_asymptotic} becomes a little more complex. We consider the alternative assumption (using the conventions $u,l$ for `upper' and `lower'):
\begin{assumption}{A6-alt}\label{asm:usefulness_alt}
Suppose a risk score $\rho_t$ is in use at time $t$, and denote $\xi = \xi_t^2(\rho_t,f_t)$. We have
\begin{equation}
k_2^l\xi^{b_l} \leq c_t \leq k_2^u \xi^{b_u} \, , \nonumber
\end{equation}
for some constants $k_2^l, k_2^u > 0$, and constants $b_l$, $b_u$ satisfying $0< b_u \leq 1 \leq b_l$. 
\end{assumption}
along with a slight strengthening we will need for Theorem~\ref{thm:oracle_comparison} (stating that cost must be determined by a function of $\xi$, not merely bounded)
\begin{assumption}{A6-add}\label{asm:usefulness_add}
With $\rho_t$, $\xi = \xi_t^2(\rho_t,f_t)$ as per assumption~\ref{asm:usefulness_alt} and $c_t$ as per assumption~\ref{asm:costs}, we have $c_t = k_2(\xi)$, where $k_2$ is a continuous function.
\end{assumption}
and note the following:
\begin{remark}
Suppose we replace~\ref{asm:usefulness} with~\ref{asm:usefulness_alt} in the statement of Theorems~\ref{thm:holdout_asymptotic}, \ref{thm:alternative_asymptotic} and \ref{thm:oracle_comparison}. Then, replacing the statement of Theorem~\ref{thm:holdout_asymptotic} with
\begin{align}
\mathbb{E}\left\{C_{(h)}[0,T]\right\} &= \delta^{b_u} N T k_2^u (2\alpha_1 + \alpha_1^2)^{b_u} +   \delta^{-1} O(N^a) + O(N^{1 + b_u(a + \epsilon-1)}) +  O(N^{1-b_u a}) \, , \nonumber 
\end{align}
Theorems~\ref{thm:holdout_asymptotic} and \ref{thm:alternative_asymptotic} still hold. If we additionally make assumption~\ref{asm:usefulness_add}, then the statement of Theorem~\ref{thm:oracle_comparison} still holds, with an asymptotic growth rate of $\Omega(\delta^{-2b_u})$ rather than $\Omega(\delta^{-2})$ for strategies $(0)$, $(a)$, $(n)$.
\end{remark}

Before proceeding to the proof, we note that in analogy to the main manuscript, we may specify an asymptotic optimal holdout size and corresponding cost in terms of $N$, although this time we must allow dependency on $b_u$. 

For a fixed $\delta$, a holdout set size of $n_*=\Theta\left(N^{\frac{1}{2}}\right)$ (regardless of $b_u$) will give an optimal asymptotic cost of 
\begin{align}
\mathbb{E}\left\{C_{(h)}[0,T]\right\} &= \delta^{b_u} N k_2^u (2\alpha_1 + \alpha_1^2)^{b_u} +   O\left(N^{1 + \left(\epsilon-\frac{1}{2}\right)b_u}\right) \, , \nonumber 
\end{align}
and if we choose 
\begin{equation}
\delta=O\left(N^{\frac{1}{2+b_u}}\right), \hspace{10pt} n_*=\Theta\left(N^{\frac{1}{2+b_u}}\right), \hspace{10pt} s=\Theta\left(N^{\frac{1}{2+b_u} + \epsilon - 1}\right) \nonumber
\end{equation}
we may achieve an optimal growth rate (for sufficiently small $\epsilon$) of:
\begin{align}
\mathbb{E}\left\{C_{(h)}[0,T]\right\} &= O\left(N^{\frac{2}{2+b_u}}\right) \, . \nonumber 
\end{align}

\begin{proof}
\textbf{Theorem~\ref{thm:holdout_asymptotic}}. For Theorem~\ref{thm:holdout_asymptotic} we proceed in the same way until expression~\ref{eq:holdout_cost_expansion}, where we instead have (employing the observations that $\mathbb{E}\{\xi^{b_u}\} \leq \mathbb{E}\{\xi\}^{b_u}$ and $(a+b)^{b_u} \leq a^{b_u} + b^{b_u}$ for $a,b,\xi>0$, since $b_u\leq 1$)
\begin{align}
\mathbb{E}\left\{C_{(h)}[e-\delta,e]\right\} &\leq \underbrace{k_1 n_*}_{\substack{\text{Cost for held} \\ \text{-out samples}}} + \mathbb{E}\{n\} \bigg( \underbrace{(1-p_*) k_1}_{\substack{\text{Cost if $<n_*$}\\ \text{samples in} \\ \text{$(e-1-s,e]$}}} +  \underbrace{p_* \mathbb{E} \left\{ k_2^u \left(\xi_t^2(\rho_f^{n_*,e,s},f_t)\right)^{b_u}\right\}}_{\substack{\text{Cost for non-held-} \\ \text{out samples otherwise}}}\bigg) \nonumber \\
&\leq k_1 n_* + N\delta \left(O(N^{-2}) +  k_2^u\mathbb{E}\left\{\xi_t^2(\rho_f^{n_*,e,s},f_t)\right\}^{b_u}\right) \nonumber \\
&\leq k_1 n_* + N\delta\left(O(n_*^{-1}) + (2\alpha_1 + \alpha_1^2)(\delta+s)\right)^{b_u} \nonumber \\
&\leq k_1 n_* + N\delta O(n_*^{-1})^{b_u} + N\delta (2\alpha_1 + \alpha_1^2)^b_u(\delta^{b_u}+s^{b_u}) \nonumber \\
&\leq \Theta(N^{a}) + \delta O(N^{1-b_u a}) + \delta \Theta(N s^{b_u}) + \delta^{1+b_u} N k_2^u (2\alpha_1 + \alpha_1^2) \nonumber \\
&= O(N^{a}) + \delta O(N^{1-b_u a}) + \delta O(N^{1 + b_u(a + \epsilon - 1)}) + \delta^{1+b_u} N k_2^u (2\alpha_1 + \alpha_1^2) \, , \label{eq:holdout_assumption_weaker_cost}
\end{align}
with a total cost per unit time accrued over the $T/\delta$ total epochs is thus
\begin{equation}
\mathbb{E}\left\{C_{(h)}[0,T]\right\} = \delta^{-1} O(N^a) + O(N^{1 + b_u(a + \epsilon-1)}) +  O(N^{1-b_u a}) + \delta^{b_u} N T k_2^u (2\alpha_1 + \alpha_1^2)^{b_u}) \, . \nonumber
\end{equation}

\textbf{Theorem~\ref{thm:alternative_asymptotic}}. To establish Theorem~\ref{thm:alternative_asymptotic} for $(\textrm{strat}) \in \{(0),(a)\}$ we pick up at equation~\ref{eq:alternative_expansion_1} to instead note 
\begin{align}
\mathbb{E}\left\{C_{(a)}[0,T]\right\} &\geq \mathbb{E}\{\textrm{Pois}(NT)\}\cdot \mathbb{E}_{t \sim U(0,T),D} \left\{\mathbb{E}\left\{k_2^l \left(\xi_t^2(\rho_t,f_t)\right)^{b_l}\right\}\right\} \nonumber \\
&\geq N T k_2^l \mathbb{E}\left\{ \xi_t^2(\rho_t,f_t)\right\}^{b_l} \nonumber \\
&\geq N T k_2^l (b-\epsilon)^{b_l} \nonumber \\
&= \Omega(N) \, . \nonumber 
\end{align}
For $(\textrm{strat}) = (n)$ we begin by rephrasing equation~\ref{eq:cost_two_ways} as
\begin{align}
k_2^u \xi_t^2\left(\rho_g^{N,(e-1)\delta,s},f_t\right)^{b_u} \geq k_1 - \mathbb{E}_{X \sim \mu_t} \left\{f_t(X)-g_t(X)\right\} &\geq  k_2^l \xi_t^2\left(\rho_g^{N,(e-1)\delta,s},f_t\right)^{b_l} \, . \nonumber
\end{align}
We must rework case 1; we now have, by assumption:
\begin{align}
k_1 - k_2^u \xi_u^2\left(\rho_g^{N,(e-1)\delta,s},f_u\right)^{b_u} &< 0 \nonumber \\
\Leftrightarrow \hspace{10pt} 
\xi_{\tau}^2\left(\rho_g^{N,(e-1)\delta,s},f_u\right) > \left(\frac{k_1}{k_2^u}\right)^{\frac{1}{b_u}} \, , \nonumber
\end{align}
so for any $t \in [(i-1)\delta,i\delta)$:
\begin{align}
\xi_t^2\left(\rho_g^{N,(e-1)\delta,s},f_t\right) &\geq  \xi_t^2\left(\rho_g^{N,(e-1)\delta,s},f_{\tau}\right) - 2\sqrt{\xi_t^2\left(f_t,f_{\tau}\right)} \nonumber \\
&\geq \xi_u^2\left(\rho_g^{N,(e-1)\delta,s},f_{\tau}\right) - \alpha_2|t-\tau| - 2\alpha_1|t-\tau| \nonumber \\
&\geq \left(\frac{k_1}{k_2^u}\right)^{\frac{1}{b_u}}  - (\alpha_2 + 2\alpha_1)\delta \, , \nonumber
\end{align}
and the total expected cost accrued over the time period $[(i-1)\delta,(i+1)\delta]$ (during which we encounter $\textrm{Pois}(2N\delta)$ samples) is:
\begin{align}
\mathbb{E}\left\{ C_{(n)} [(i-1)\delta,(i+1)\delta]\right\} &\geq \mathbb{E}\left\{ C_{(n)} [(i-1)\delta,i\delta]\right\} \nonumber \\
&\geq \mathbb{E}\left\{\textrm{Pois}(N\delta)\right\} k_2^l \mathbb{E}_{t \sim U((i-1)\delta,i\delta)}\left\{\xi_t^2\left(\rho_g^{N,(e-1)\delta,s},f_t\right)^{b_l}\right\} \nonumber \\
&\geq  N\delta k_2^l \left( \left(\frac{k_1}{k_2^u}\right)^{\frac{1}{b_u}}  - (\alpha_2 + 2\alpha_1)\delta \right)^{b_l} \, , \nonumber
\end{align}
hence we may choose $\delta$ dependent only on $\alpha_2,\alpha_1,k_2$ sufficiently small that
\begin{align}
\mathbb{E}\left\{ C_{(n)} [(i-1)\delta,(i+1)\delta]\right\} &\geq N\delta m_1 \label{eq:cost_case1_alternative}
\end{align}
for some positive constant $m_1$.

For case 2, we again appeal to expression~\ref{eq:lt_delta0} to rewrite expression~\ref{eq:cost_expansion_naive_case2_1} (using the fact that $\mathbb{E}\{\xi^{b_l}\}\geq \mathbb{E}\{\xi\}^{b_l}$ since $b_l>0$)
\begin{align}
\mathbb{E}\left\{ C_{(n)}[(i-1)\delta,i\delta] \right\} &\geq \mathbb{E}\left\{\textrm{Pois}(N\delta)\right\} k_2^l \mathbb{E}_{t\sim U((i-1)\delta,i\delta)}\left\{\xi_t^2\left(\rho_g^{N,(i-1)\delta,s},f_t\right)^{b_l} \right\}\nonumber \\
&\geq N\delta  k_2^l \mathbb{E}_{t\sim U((i-1)\delta,i\delta)}\left\{\xi_t^2\left(\rho_g^{N,(i-1)\delta,s},f_t\right)\right\}^{b_l} \nonumber \\
&\geq  N\delta k_2^l \left(\Delta_0  - \delta(\alpha_2 + 2\alpha_1) - \alpha_2 s\right)^{b_l}\nonumber \\
&\geq  N\delta k_2^l \left(\Delta_0  - m_s \right)^{b_l} \, . \nonumber
\end{align}
We may recalculate the final three lines of derivation~\ref{eq:xibound_case2} as:
\begin{align}
\xi_{i\delta}\left(f_{i\delta},G\right) &\geq \left(\frac{1}{s}\int_{i\delta-s}^{i\delta} \max\left(0,\mathbb{E}_{X \sim \mu_t} \left\{f_t(x)-g_t(x)\right\} - \alpha_1\delta\right)\right)^2 \nonumber \\
&\geq \left(\frac{1}{s}\int_{i\delta-s}^{i\delta} \max\left(0,k_1 - k_2^u \xi_t^2\left(\rho_g^{N,(e-1)\delta,s},f_t\right)^{b_u} - \alpha_1\delta\right)\right)^2 \nonumber \\
&\geq \max\left(0,k_1 - k_2^u \left(\Delta_0 + (\alpha_2+2\alpha_1)\delta + \alpha_1^2 \delta^2 +\alpha_2 s
\right)^{b_u} - \alpha_1\delta\right)^2 \label{eq:xibound_case2_alternative}
\end{align}
and hence rewrite  expression~\ref{eq:cost_expansion_naive_case2_2} as
\begin{align}
\mathbb{E}\left\{ C_{(n)} [i\delta,(i+1)\delta]\right\} &\geq \mathbb{E}\left\{\textrm{Pois}(N\delta)\right\} k_2^l \mathbb{E}_{t\sim U(i\delta,(i+1)\delta)}\left\{\xi_t^2\left(\rho_g^{N,i\delta,s},f_t\right)^{b_l} \right\} \nonumber \\
&\geq N\delta k_2^l \mathbb{E}_{t\sim U(i\delta,(i+1)\delta)}\left\{\xi_t^2\left(\rho_g^{N,i\delta,s},f_t\right)\right\}^{b_l}  \nonumber \\
&\geq  N\delta k_2^l \left(k_1-k_2^u(\Delta_0 - m_s)^{b_u}-m_s\right)^{2b_l} \, , \nonumber
\end{align}
and hence the total cost over both periods is
\begin{align}
\mathbb{E}\left\{ C_{(n)} [(i-1)\delta,(i+1)\delta]\right\} &\geq  N\delta k_2 \left( (\Delta_0  - m_s)^{b_l} + (k_1-k_2^u(\Delta_0 - m_s)^{b_u}-m_s)^{2b_l}\right) \, . \nonumber
\end{align}
Analogous to the proof of Theorem~\ref{thm:alternative_asymptotic}, given $\epsilon$, we may choose $\delta$ (dependent only on $k_1$, $k_2$, $\alpha_2$, $\alpha_1$) sufficiently small that for large enough $N$:
\begin{align}
\mathbb{E}\left\{ C_{(n)} [(i-1)\delta,(i+1)\delta]\right\} &\geq N\delta \left( k_2\min_{0\leq \Delta_0 \leq 1} \left(\Delta_0^{b_l} + (k_1-k_2^u\Delta_0^{b_u})^{2b_l}\right) - \epsilon \right) \, . \nonumber 
\end{align}
Recalling expression~\ref{eq:cost_case1_alternative} for the earlier case, we denote 
\begin{equation}
c_{(n)}=\min\left(k_2^l \min_{0\leq \Delta_0 \leq 1} \left(\Delta_0^{b_l} + (k_1-k_2^u\Delta_0^{b_u})^{2b_l}\right), k_1\right) > 0 \, , \label{eq:cndef2}
\end{equation}
so, in either case
\begin{align}
\mathbb{E}\left\{ C_{(n)} [(i-1)\delta,(i+1)\delta]\right\} &\geq N\delta \left( c_{(n)} - \epsilon \right) \, ,\nonumber 
\end{align}
and the proof proceeds as above.

\textbf{Theorem~\ref{thm:oracle_comparison}}. For Theorem~\ref{thm:oracle_comparison}, we must additionally make assumption~\ref{asm:usefulness_add}, since without it there is no guarantee that costs of various strategies will be similar even if the associated risk scores are equally similar to $f_t$.

We require a slightly modified form of Lemma~\ref{lemma:bigT}; namely that for sufficiently large $T$ we have 
\begin{equation}
I(T) \triangleq \sum_{i=0}^{T/\delta - 1} \mathbb{E}_{t \sim U(i\delta,(i+1)\delta)}\left\{k_2\left(\xi_t^2(f_{i\delta},f_t)\right)\right\} > 0 \, , \nonumber
\end{equation}
which can proved in essentially the same way as the original lemma by first noting that since $k_2(\xi) \geq k_2^l \xi^{b_l}$, we have
\begin{equation}
I_2(T) \geq k_2^l \sum_{i=0}^{T/\delta - 1} \mathbb{E}_{t \sim U(i\delta,(i+1)\delta)}\left\{\xi_t^2(f_{i\delta},f_t)^{b_l}\right\} \, . \nonumber
\end{equation}
With this, we may trudge back to derivation~\ref{eq:oracle_holdout_derivation}, delete any occurrences of the \emph{constant} $k_2$, replace $I(\cdot)$ with $I_2(\cdot)$, and replace any instances of $\xi_t^2(f_{i\delta},f_t)$ with $k_2\left(\xi_t^2(f_{i\delta},f_t)\right)$, and see that the derivation holds. 

Given that $k_2(\xi)\leq k_2^u \xi^{b_u}$ we may amend expression~\ref{eq:it_upper_bound} to note:
\begin{equation}
I_2(T) \leq T k_2^u \alpha_1^{2b_u} \delta^{2b_u - 1} \, , \nonumber
\end{equation}
so for sufficiently small $\delta$:
\begin{align}
 \lim_{N \to \infty} \left(\frac{\mathbb{E}\left\{C_{(a)}[0,T]\right\}}{\mathbb{E}\left\{C_{(o)}[0,T]\right\}}\right) &=\lim_{N \to \infty} \left( \frac{N T k_2 \mathbb{E}_{t \sim U(0,T)} \left\{k_2\left(\xi_t^2(\rho_t,f_t)\right\}\right)}{N \delta I_2(T)} \right) \nonumber \\
  &\geq \lim_{N \to \infty} \left( \frac{ T k_2^l \mathbb{E}_{t \sim U(0,T)} \left\{\xi_t^2(\rho_t,f_t)^{b_l}\right\}}{ T \delta k_2^u \alpha_1^{b_u} \delta^{2b_u - 1}} \right) \nonumber \\
 &\geq \lim_{N \to \infty} \left( \frac{ k_2^l \mathbb{E}_{t \sim U(0,T)} \left\{\xi_t^2(\rho_t,f_t)\right\}^{b_l}}{ k_2^u \alpha_1^{b_u} \delta^{2b_u}} \right) \nonumber \\
 &\geq\lim_{N \to \infty} \left( \frac{k_2^l(b-\epsilon)^{b_l}}{k_2^u \delta^{2b_u} \alpha_1^{2b_u}} \right) \nonumber \\
&=\Omega(\delta^{-{2b_u}}) \nonumber \\
&>1 \, .\nonumber
\end{align}
For $(\textrm{strat})=(n)$, allowing for the new definition of $c_{(n)}$ in expression~\ref{eq:cndef2}, we have (picking up at derivation~\ref{eq:cost_oracle_expansion_naive})
\begin{align}
 &\lim_{N \to \infty} \left(\frac{\mathbb{E}\left\{C_{(n)}[0,T]\right\}}{\mathbb{E}\left\{C_{(o)}[0,T]\right\}}\right) \geq \lim_{N \to \infty} \left( \frac{\frac{1}{2}\sum_{i=1}^{T/\delta - 1} C_{(n)}[(i-1)\delta,(i+1)\delta]}{\sum_{i=0}^{T/\delta - 1} C_{(o)}[i\delta,(i+1)\delta]} \right) \nonumber \\
 &\geq \lim_{N \to \infty} \left(\frac{ \frac{1}{2} N \delta \left(\frac{T}{\delta}-1\right)(c_{(n)}-\epsilon)}{N\delta I_2(T)} \right) \nonumber \\
 &\geq \lim_{N \to \infty} \left(\frac{ (T-\delta)(c_{(n)}-\epsilon)}{ 2 k_2^u T \delta^{2b_u} \alpha_1^{2b_u}} \right) \nonumber \\
&=\Omega(\delta^{-2b_u}) \nonumber \\
&>1 \, , \nonumber
\end{align}
as required.
\end{proof}

\clearpage

\subsection{Updating in the absence of drift}
\label{supp_sec:f0diff}

{
We briefly consider performance of updating strategies (section~\ref{sec:general_setup}) when no drift occurs. In some such settings researchers may still want to update risk scores: for instance, in order to make use of a new or better training algorithm. Essentially we claim that \emph{either} the risk score should not be updated, or a holdout set should be used to update it, even in this case: the naive-update strategy will still lead to suboptimal costs.
}

We firstly consider a fixed time period $[0,T]$ for which there is no $T_1 \leq T$ for which assumption~\ref{asm:f0_difference} holds. We note that the proof of Theorem~\ref{thm:alternative_asymptotic} for $(\textrm{strat})=(n)$ makes no use of assumption~\ref{asm:f0_difference}, so if naive updating is used, cost accrues at a rate $\Omega(N)$. Essentially, the costs accrued due to use of naive updating arise from the tension that a difference between $f_t$ and $g_t$ both indicates that the risk score is a `success' in reducing incidence of adverse events, but means that it is inevitably inaccurate once updated. 

The upper bound in Theorem~\ref{thm:holdout_asymptotic} similarly does not require assumption~\ref{asm:f0_difference}, so we may attain $O(N)$ cost for fixed $\delta$ or $O(N^{2/3})$ cost if we may choose $\delta$. The costs accrued using the no-update strategy ($(\textrm{strat})=(0)$) are 0 if we assume full knowledge of $f_0$, since $f_t=f_0$ almost everywhere with respect to $\mu_t$ for all $t\leq T$. 

We now turn our attention to an asymptotic regime with an infinite time horizon. Since $\sup_{x,\rho}|f_t(x)-g_t(\rho,x)| \leq 1$, we have, for fixed $N$:
\begin{equation}
\int_0^T \xi_t^2(f_tf_0)dt = O(T) \, .\nonumber
\end{equation}
If we had:
\begin{equation}
\int_0^T \xi_t^2(f_tf_0)dt = o(T) \, ,\nonumber
\end{equation}
then the mean deviation of $f_t$ from $f_0$ would converge to 0 as $T \to \infty$, as would the time-averaged total cost if the no-update strategy is used with fixed $N$:
\begin{equation}
\frac{1}{T}C_{(0)}[0,T] =  \frac{1}{T} k_2 \int_0^T \xi_t^2(f_t,f_0)dt = \frac{o(T)}{T} \to 0 \, . \nonumber
\end{equation}
If either the naive-update or holdout-set update strategies are followed with a fixed update frequency $\delta$, the time-averaged cost as above will not converge to 0, since a new cost is added each time the model is updated. We note that over any finite time horizon over which some drift occurs, our usual formulations of Theorems~\ref{thm:holdout_asymptotic}, \ref{thm:alternative_asymptotic} and~\ref{thm:oracle_comparison} govern the relative costs of each strategy.

If we have 
\begin{equation}
\int_0^T \xi_t^2(f_tf_0)dt = \Theta(T) \, , \nonumber
\end{equation}
then for all sufficiently large finite time horizons $T$, assumption~\ref{asm:f0_difference} will hold, and Theorems~\ref{thm:holdout_asymptotic}, \ref{thm:alternative_asymptotic} and~\ref{thm:oracle_comparison} describe the performance of each strategy.

\clearpage

\section{Notes on ethics and alternative options}

\subsection{Non-identifiability of $f_t$ from $g_t$ and $\rho_t$}
\label{supp_sec:nonidentifiability}

As in the main paper, suppose that we are a researcher at time $t$, at which a risk score $\rho_t$ is in place, and we have at hand \emph{only} a set of sample covariate values $x$ and values $g_t(x)=g_t(x,\rho_t)$ (or rather, samples from $\textrm{Bern}(g_t(x,\rho_t))$). We show here that we cannot infer the function $f_t$ from only this information: note that, in particular, we do not directly observe $f_t(x)$ or values $\textrm{Bern}(f_t(x))$. We will demonstrate this with a simple counterexample.

Intuitively, we have \emph{some} knowledge of $f_t$; it cannot be an arbitrary function. In particular $f_t$ will be similar to $g_t$ if we have little confidence in the risk score (and hence take little risk-score guided action) and $f_t$ will be different from $g_t$ if we are more confident that the risk score resembles $\rho_t$. These intuitions are quantified by  assumptions~\ref{asm:costs} and \ref{asm:usefulness}, in that we have (all expectations over $X \sim \mu_t$):
\begin{align}
c_t(\rho_t) &= k_2\xi_t^2(f_t,\rho_t) &&\text{Assumption~\ref{asm:usefulness}} \nonumber \\
&= \mathbb{E}\left\{g_t(X,\rho_t) - g_t(X,f_t)\right\} &&\text{Assumption~\ref{asm:costs}} \nonumber \\
&= \mathbb{E}\left\{g_t(X,\rho_t)-f_t(X) + f_t(X)- g_t(X,f_t)\right\} \nonumber \\
&= \mathbb{E}\left\{g_t(X,\rho_t)-f_t(X)\right\} + \mathbb{E}\left\{f_t(X)- g_t(X,f_t)\right\} \nonumber \\
&= \mathbb{E}\left\{g_t(X,\rho_t)-f_t(X)\right\} + k_1 &&\text{Assumption~\ref{asm:costs}} \nonumber \\
&= k_1 - \mathbb{E}\left\{f_t(X) - g_t(X,\rho_t)\right\} \nonumber
\end{align}
Therefore, as stated in the main paper,
\begin{equation}
k_2\xi_t^2(f_t,\rho_t) = k_1 - \mathbb{E}_{X \sim \mu_t}\{f_t(X)-g_t(X, \rho_t)\} \, . \label{eq:ft_all_we_know}
\end{equation}
This comprises the only constraint we have on $f_t$, given knowledge of the functions $\rho_t$ and $g_t(\cdot,\rho_t)$.

We show in this section that knowledge of the value of this is (unsurprisingly) not enough to identify $f_t$. Our construction is elementary, but somewhat laborious. Let $h_t(x)$ be a function and consider the function $f_t^h(x)=f_t(x) + h_t(x) - h$ as follows:
\begin{align}
&k_2\xi(f_t^h,\rho_t) - \left(k_1-\mathbb{E}\{f_t^h(X)-g_t(X, \rho_t)\}\right) \nonumber \\
&= k_2\mathbb{E}\{(f_t(X) + (h_t(X)-h) -\rho_t(X))^2\} - \left(k_1 - \mathbb{E}\{f_t(X)+h_t(X)-h-g_t(X, \rho_t)\}\right) \nonumber \\
&= k_2\xi_t^2(f_t,\rho_t) + k_2\mathbb{E}\{(h_t(X)-h)(f_t(X) -\rho_t(X))\} + 2k_2\mathbb{E}\{(h_t(X)-h)^2\} \nonumber \\
&\phantom{=} - \left(k_1-\mathbb{E}\{f_t(X)-g_t(X, \rho_t)\}\right) + \mathbb{E}\{h_t(X) - h\} \nonumber \\
&= \mathbb{E}\{(h_t(X)-h)(k_2(f_t(X) -\rho_t(X) + h_t(X)-h) + 1)\} \nonumber \\
&= \mathbb{E}\{(h_t(X)-h)(k_2(f_t(X) -\rho_t(X) + h_t(X)-h) + 1)\} \nonumber \\
&= A h^2 + \mathbb{E}\{B\} h + \mathbb{E}\{C\} \, , \label{eq:banana}
\end{align}
where 
\begin{align}
A &= k_2 \nonumber \\
B &= -\left(2k_2(f_t(X) + h_t(X) - \rho_t(X)) + 1\right) \nonumber \\
C &= \left(k_2 \left( h_t(X)^2 + 2h_t(X)(f_t(X)-\rho_t(X))\right) + h_t(X)\right) \, , \nonumber
\end{align}
and all expectations are over $X \sim \mu_t$. We note that 
\begin{equation}
\mathbb{E}\{B^2 - 4AC\} = \mathbb{E}\{B^2\} - 4A\mathbb{E}\{C\} = \left(2 k_2 (f_t(X)-\rho_t(X)) + 1\right)^2 > 0 \, .
\end{equation}
If we choose $h_t(X)$ close to $\rho_t(X)-f_t(X)$, we can ensure $\textrm{var}(B)$ is small enough that 
\begin{equation}
\textrm{var}(B) = \mathbb{E}\{B^2\} - \mathbb{E}\{B\}^2  < \mathbb{E}\{B^2 - 4AC\} \, , \nonumber
\end{equation}
and hence $\mathbb{E}\{B\}^2 - 4A\mathbb\{C\}>0$ and equation~\ref{eq:banana} has a root $h$ (call it $h_t^0$). Thus, for \emph{arbitrary} $h_t(X)$ close to $\rho_t(X)-f_t(X)$, we can find an $h_t^0$ such that $f_t^h(x)=f_t(x) + h_t(x) - h_t^0$ also satisfies
\begin{equation}
k_2\xi_t^2(f_t^h,\rho_t) = k_1 - \mathbb{E}_{X \sim \mu_t}\{f_t^h(X)-g_t(X)\} \, , \nonumber
\end{equation}
that is, given $\rho_t$ and $g_t(\cdot,\rho_t)$ we have identified a class of distinct functions $f^h_t$ which are all consistent with identity~\ref{eq:ft_all_we_know}. Hence $f_t$ is not distinguishable from $f_t^h$, even with knowledge of $k_1$, $k_2$, and hence is not identifiable from $g_t,\mu_t, \rho_t$ alone.

\clearpage

\subsection{Use of natural hold-out sets,  recorded interventions, and maximum-of-two updating}
\label{supp_sec:natural_holdout}

Ethical objections to the use of holdout sets are in a sense due to the need to \emph{actively} with-hold some samples from access to a risk score. A brief discussion is warranted on potential use of `natural' hold-out sets, in which some proportion of a population may be assumed to behave like a holdout set without actively requiring withholding of scores. We will use our motivating example (section~\ref{sec:motivating_example}) to illustrate this idea.

An obvious setting in which a natural holdout set would be appropriate is in the case where an effectively random subset of samples already do not have access to a risk score. In our example, this may comprise a set of individuals under the care of a medical practitioner who chose not to use the ASPRE score. As long as the distribution of covariates of such individuals is typical of the population distribution (that is, $\mu_t$ is conserved) and the behaviour of such practitioners is typical of the general behaviour of practitioners in the absence of a risk score (that is, $f_t$ is conserved) then such a subset of individuals can be used as a holdout set.

\subsubsection{Recorded interventions}


{
Risk scores are generally used to simplify information in complex settings, including medicine or finance. This complexity is generally also present in the range and effect of intervention: it is difficult to determine the optimal intervention in a given context. We generally wish to make use of expertise of domain experts (for instance, doctors or financial analysts) in making interventions, and hence do not generally consider direct recommendations of intervention to be practical in this circumstance. Moreover, even \emph{recording} of interventions may be difficult: a doctor responding to a high disease risk may see a patient more often, reduce thresholds for further investigation, or consider new treatments, all of which may be difficult to record. }

{However, we consider here options for a setting in which the intervention is simple and recorded (for the case of PRE, whether an individual was given aspirin).} The set of individuals for whom no intervention was recorded (that is, did not get prescribed aspirin) do not constitute a `natural' holdout set in our sense: the function $f_t$ measures the probability of an event \emph{under normal care without a risk score}, and patients without a risk score may still get prescribed aspirin. 

Suppose we are working at a fixed time $t$ (that is, disregard dependence on time) and denote by $X$ a set of covariates included in a risk score, $A$ an indicator for whether aspirin was prescribed or not, $L$ a `latent' covariate (taking values in finite set $\mathscr{L}$ for simplicity) representing patient characteristics visible to a medical practitioner but which are not amongst covariates $X$, $Y \in \{0,1\}$ the event in question (in this case, the incidence of \textsc{pre}) and $G$ the event of whether the practitioner had access to a risk score ($G=1$ for yes, $G=0$ for no). 

We firstly note that recording $A$ does not necessarily help in meaningfully `updating' the risk score, essentially because the same considerations as in Figure~\ref{fig:epoch} continue to hold. In particular, the treatment decision $P(A|X)$ depends on the risk score currently in place. Should we estimate $\rho_{e+1}=P(Y|A,X)$ while a risk score $\rho_e$ is in use, and replace $\rho_e$ with $\rho_{e+1}$, then $\rho_{e+1}$ will now not generally estimate $P(Y|A,X)$, since $P(A|X)$ has changed. 

{It is worthwhile to consider a more general estimation problem. Faced with the problem of whether to allocate treatment to a particular patient, given covariates $x$, we ideally wish to estimate the counterfactuals 
\begin{equation}
P_{A\gets 0}(Y|X=x) \hspace{10pt} \textrm{ and } \hspace{10pt} P_{A\gets 1}(Y|X=x) \, ; \label{eq:estimable_counterfactuals}
\end{equation}
or potentially the counterfactuals:
\begin{equation}
P_{A\gets 0}(Y|X=x,L=\lambda) \hspace{10pt} \textrm{ and } \hspace{10pt} P_{A\gets 1}(Y|X=x,L=\lambda) \, ; \label{eq:estimable_counterfactuals2}
\end{equation}
that is, the probability of the outcome on a patient with $X=x$ if we forcibly assign a treatment $A=0$ or $A=1$. The counterfactual quantities are computed assuming that random variables $X,G,Y,A,L$ have the following causal structure:}
\begin{center}
\vspace{10pt}
\begin{tikzpicture}

    \node[state] (x) at (0,4) {$X$};
    \node[state] (y) at (0,0) {$Y$};
    \node[state] (a) at (2,2) {$A$};
    \node[state] (g) at (0.5,2.5) {$G$};
    \node[state] (l) at (4,2.5) {$L$};

    \path (x) edge (y);
    \path (x) edge (l);
    \path (g) edge (a);
    \path (x) edge (a);
    \path (a) edge (y);
    \path (l) edge (a);
    \path (l) edge (y);

\end{tikzpicture}
\vspace{10pt}
\end{center}
in which $L$ are partly dependent on $X$, the decision $A$ to prescribe aspirin is made based on $X$ and $L$ with the decision rule modulated by $G$, and $Y$ is dependent on $A$, $X$ and $L$. To reconcile $L$ with our earlier notation we have
\begin{align}
f_t(x) &=P(Y=1|X=x,G=0) =\sum_{\lambda \in \mathscr{L}} P(Y=1|X=x,G=0,L=\lambda)P(L=\lambda|X=x) \nonumber \\
g_t(x) &=P(Y=1|X=x,G=1) =\sum_{\lambda \in \mathscr{L}} P(Y=1|X=x,G=1,L=\lambda)P(L=\lambda|X=x) \, , \nonumber 
\end{align}
taking $L$ and $G$ as conditionally independent given $X$. 

The data we may use to estimate these counterfactuals are direct estimates of the conditionals 
\begin{equation}
P(Y|X=x,A=1) \hspace{10pt} \textrm{ and } \hspace{10pt} P(Y|X=x,A=0) \, , \label{eq:estimable_conditionals}
\end{equation}
whose values are dependent on the risk score through $P(A|X=x)$. If we make a \emph{deterministic} decision on $A$ on the basis of $x$, so $P(A|X=x) \in \{0,1\}$, then for each value of $x$ we can estimate one of the values~\ref{eq:estimable_conditionals} and it is equal to the corresponding counterfactual~\ref{eq:estimable_counterfactuals}, but we lose all ability to estimate the other. 

In more generality, we consider two circumstances:
\begin{enumerate}
\item We make a deterministic decision on the basis of $X,L$. 
\item We have a degree of randomness, so some patients with $X=x,L=\lambda$ are treated and some are not.
\end{enumerate}

In the \textbf{first circumstance}, this leaves us with much the same problem as when we use a deterministic decision based on $x$: faced with a patient with $(X,L)=(x,\lambda)$, we can estimate only one of the counterfactuals~\ref{eq:estimable_counterfactuals2} and at most one of the counterfactuals~\ref{eq:estimable_counterfactuals}. In other words, we can say what will happen to this patient if we take a specific, forced course of action, but not if we take the alternative. We claim that knowledge of $f_t$ (that is, the probability of the outcome \emph{under normal clinical care without a risk score}) is more clinically useful. 

In the \textbf{second circumstance}, as long as we allocate patients with $X=x$ to $A=1$ or $A=0$ with positive probabilities, we can estimate the counterfactuals~\ref{eq:estimable_counterfactuals2}, and thereby decide on the best option. However, we note that since $L$ represents everything that a clinician can see, this scenario constitutes \emph{randomising patients to a treatment}. With this in mind:
\begin{itemize}
\item We expect the counterfactual values to change with time, so we will need to re-estimate them, and keep allocating patients to both $A=0$ and $A=1$. 
\item Given that consent to randomisation may not be independent of $X$ and $L$, it may be difficult to make unbiased estimates of the counterfactual values
\item While randomisation of treatment is considered ethically reasonable in a clinical trial while in a position of equipoise over the effectiveness of the treatment, it is less ethically permissible when the treatment is known to be effective (as we are assuming). In a sense, this is \emph{less ethical than using a holdout set}, in that treatment is compelled, in a sense, rather than left to the clinician's best discretion in the absence of a risk score. 
\end{itemize}
In summary, we claim that recording interventions may enable estimation of counterfactual quantities~\ref{eq:estimable_counterfactuals}, but that this is not straightforward, that estimates would be unlikely to be consistent or unbiased, and that this may be ethically less acceptable than use of a holdout set. Nonetheless, this remains a potential avenue for future work.

\subsection{Maximum-of-two updating}
\label{supp_sec:best_of_two}

{
In a parallel applied work~\cite{liley21medRxiv} we recognised the potential problems with naive updating as applied to a real-world risk score. The score in question predicted yearly individual risk of emergency hospital admission for the majority of the Scottish population, with risk scores deployed monthly to Scottish general practitioners for the patients in their care. We developed the fourth version of the risk score, the third version being currently in use. 
}

{
In the absence of a holdout set (and the lack of precedent) we chose to first re-fit the risk score naively: in the notation of section~\ref{sec:holdout_set_motivation}, taking $t=1,2,3,4$ as the times when the first through fourth versions of the risk score were fitted, if $\rho_3$ is the third version of the risk score, our naive estimate $\rho_4'$ was 
\begin{equation}
\rho_4'(x) \approx g_4(x,\rho_3) \, . \nonumber
\end{equation}
We then proposed to designate the updated risk score $\rho_4$ as:
\begin{equation}
\rho_4 = \max\left (\rho_3(x),\rho_4'(x)\right) \, . \nonumber
\end{equation}
Our intent in doing so was to try and ensure that $\rho_4(x) \geq f_4(x)$, given that we could not aim for $\rho_4(x) = f_4(x)$. We consider that this was the best option available to us at the time. However, this approach has many drawbacks compared to a holdout set approach, and we do not consider that it is viable in general or in the long term. Our reasons for this are as follows:
}
{
\begin{enumerate}
\item Overestimation of risk, while possibly less costly than underestimation of risk, is not without cost. In practical terms, overestimation of risk could lead to overtreatment, or diversion of care away from patients in greater need. 
\item There is no clear extension of this approach to updating the model a second time (say, to $\rho_5$). Given $\rho_5'(x) \approx g_t(x,\rho_4)$, we could take 
\begin{equation}
\rho_5(x)=\max\left(\rho_4,\rho_5'\right) = \max\left(\rho_3,\rho_4',\rho_5'\right) \, , \nonumber 
\end{equation}
but this will lead to more severe overestimation of $f_5$, ultimately lowering confidence in the risk score as above. It is also more susceptible to drift than our $\rho_3 \to \rho_4$ update, since we are relying on an estimate made two epochs ago rather than just one. This is not avoided if we take $\rho_5(x)=\max\left(\rho_4',\rho_5'\right)$, since $\rho_3$ has a causal influence on $\rho_4'$ as per Figure~\ref{fig:epoch}. 
\item As per Section~\ref{sec:cost_specification}, should we use a holdout set, we would generally try and ensure the holdout set is as close as possible to $t=4$, and thus reduce the effect of drift between $t=3$ and $t=4$. Our approach in this case is more susceptible to drift, because we use the risk score $\rho_3$ directly in our estimate of $\rho_4$. 
\end{enumerate}
}
\clearpage

\section{Proofs of Theorems~\ref{thm:ohs_exists}, \ref{thm:ohs_exists_weak}, and \ref{thm:ohs_cor}}
\label{apx:thm1proof}

\begin{reptheorem}{thm:ohs_exists}
Suppose assumptions~\ref{item:k1_indep_pi}-\ref{item:k2_2nd_der} in the main manuscript hold. 
Then there exists a $N_* \in (0,N)$ with $N \in \mathbb{N}$, which we call the optimal holdout set size, such that:
\begin{align}
    \ell(i) \geq \ell(j) &\textrm{ for } 0 < i < j < N_* \nonumber \\
    \ell(i) \leq \ell(j) &\textrm{ for } N_* < i < j < N \, . \nonumber
\end{align}
\end{reptheorem}

\begin{proof}
%

As discussed in the main manuscript, 
we may impose that
\begin{align}
\frac{\partial}{\partial n} k_2(n) &< 0  \, .   \label{eq:k_diffs}
\end{align}

Since both $k_2(n)$ and $(N - n)$ are positive and monotonically decreasing in $n$, so is $k_2(n)(N-n)$. 
Now 
\begin{align}
    \ell'(n) 
    &= \frac{\partial}{\partial n} \left(k_1 n + k_2\left(n\right) \left(N-n\right)\right)  \\
    &= k_1 + k_2'\left(n\right)\left(N-n\right) - k_2\left(n\right) \\
    &= \left(k_1-k_2\left(n\right)\right) + k_2'\left(n\right)\left(N-n\right) \, .   \label{eq:ldash}
\end{align}
By assumption~\ref{item:k1_l_k2} in the main manuscript, $k_1 < k_2(0)$, and, from equation~\ref{eq:k_diffs}, $k_2'(0) < 0$, so both terms in equation~\ref{eq:ldash} are negative when $n=0$ and $\ell'(0) < 0$. When $n=N$, the second term vanishes while the first one is positive, as $k_1 > k_2(N)$ by assumption~\ref{item:k1_l_k2} in the main manuscript. We thus have $\ell'(N) > 0$. By assumption, $\ell$ is smooth, so 
by Bolzano's Theorem, there must exist at least one point $n_*$ for which $\ell'\left(n_*\right)=0$, which is an extremum of $\ell$.

We now prove that this extremum is unique and a minimum. First, by assumption~\ref{item:k2_2nd_der} in the main manuscript, we may impose that
\begin{equation}
    \frac{\partial^2}{\partial n^2} k_2(n) > 0 \, .      \label{eq:k2_2nd_der}
\end{equation}
Taking the second derivative of $\ell$: 
\begin{align}
    \frac{\partial^2}{\partial n^2} \ell(n) 
    &= \frac{\partial^2}{\partial n^2}\left( k_1 n + k_2(n)(N - n) \right) \\
    &= k_2''(n) (N - n) - 2k_2'(n) \, ;
\end{align}
and using equations~\ref{eq:k_diffs} and \ref{eq:k2_2nd_der}, we see that $\ell''(n)$ is strictly positive, and, as a consequence, $\ell'(n)$ is monotonically increasing. Therefore, the extremum of $\ell(n)$ at $n_*$ we found earlier is unique and, as $\ell''(n) > 0$, it is a minimum.

If $n_* \in 1..(N-1)$
, let $N_*=n_*$. If $n_* \notin \mathbb{N}$, let $N_*$ be the closest natural number to either side of $n_*$. From assumption~\ref{item:k1_l_k2} in the main manuscript, $N_*$ cannot be $0$ or $N$. In both scenarios, this completes the proof.

\end{proof}



For the proofs of Theorems~\ref{thm:ohs_exists_weak} and \ref{thm:ohs_cor}, we will use the following lemma:

\begin{lemma} \label{lem:lMlN}
Suppose assumption 1 holds and there exists $0<M<N$ such that $k_2(M) < k_1$. Then $\ell(M) < \ell(N)$.
\end{lemma}

\begin{proof}
\begin{align}
    k_2(M) < k_1 &\implies (N-M)k_2(M) < (N-M)k_1 \nonumber \\
    &\implies (N-M)k_2(M) + Mk_1 < Nk_1 \nonumber \\
    &\implies \ell(M) < \ell(N) \, . \nonumber
\end{align}
\end{proof}

We now have

\begin{reptheorem}{thm:ohs_exists_weak} 
Suppose assumptions \ref{item:k1_indep_pi} and
\ref{item:frac_imp} hold, and $k_2(0) \leq k_1$. Then there exists an $N_* \in \{1, \dots, N-1\}$ such that:
$\ell(i) \geq \ell(N_*) \textrm{ for } i \in \{1,\dots,N-1\}$ and
$\ell(i) > \ell(N_*) \textrm{ for } i \in \{0,N\}$
\end{reptheorem}

\begin{proof}
All that is needed to show that there exists a holdout set size $M$ where $\ell(M) < \ell(0)$ and $\ell(M) < \ell(N)$. This will be the $M$ in assumption~\ref{item:frac_imp}.

It immediately follows from assumption~\ref{item:frac_imp} that $k_2(M)<k_1$, so by Lemma \ref{lem:lMlN} $\ell(M) < \ell(N)$.
If $k_1 = k_2(0)$ then we are done as $\ell(N) = \ell(0)$. If $k_1 > k_2(0)$ then from assumption~\ref{item:frac_imp} we have 
\begin{align}
    k_1 - k_2(0) < \frac{N-M}{N}(k_1 - k_2(M)) &\implies N( k_1 - k_2(0)) < (N-M)(k_1 - k_2(M)) \nonumber \\
    &\implies Nk_1 - Nk_2(0) < Nk_1 -(N-M)k_2(M) - Mk_1 \nonumber \\
    &\implies \ell(N) - \ell(0) < \ell(N) - \ell(M) \nonumber \\
    &\implies \ell(M) < \ell(0) \, , \nonumber
\end{align}
as needed.
\end{proof}

\begin{reptheorem}{thm:ohs_cor}
Suppose assumption~\ref{item:k1_indep_pi} holds, $k_1 < k_2(0)$ and there exists $0<M<N$ such that $k_2(M) < k_1$. Then there exists an $N_* \in \{1, \dots, N-1\}$ such that:
$\ell(i) \geq \ell(N_*) \textrm{ for } i \in \{1,\dots,N-1\}$ and
$\ell(i) > \ell(N_*) \textrm{ for } i \in \{0,N\}$.
\end{reptheorem}

\begin{proof}

Note $$k_1 < k_2(0) \implies Nk_1 < Nk_2(0) \implies \ell(N) < \ell(0)$$ so all that is needed is $\ell(M) < \ell(N)$, which is given by Lemma \ref{lem:lMlN}.
\end{proof}

\clearpage

\section{Estimation of $k_2(n)$}
\label{supp_sec:k2_estimation}

\subsection{Justification of approximations of $k_2(n)$ as a mean-square error}

We argue in this section that $k_2(n)$ can generally be a modelled as a linear function of expected mean squared error of the risk score, where `expected' refers to expectation over the data used to fit the risk score. 


Suppose that we are interested in making predictions at a time $t$. We consider a risk score which, at the value $x$,  is an inexact approximation of $f_t(x)$ (recalling the definition of $f_t(x)$ from section~\ref{sec:holdout_set_motivation} as the probability of $Y=1$ given $X=x$ at time $t$ when no risk score is used). 

We will denote by $\varrho = \rho(x)$ the value of the risk score $\rho$ at $x$, where $\rho$ is taken to depend on the samples $d_n$. We write $c_2(x,\varrho)=\mathbb{E}\{C_2(x;d_n)\}$, where $\mathbb{E}_{C_2}$ indicates expectation over any additional randomness after fixing $x$ and $d_n$. 

If it is reasonable to assume that $c_2(x,\varrho)$ has a straightforward form in terms of $\varrho$, $f_t$, then a corresponding form for $k_2$ may be immediate. For instance, if one of 
\begin{align}
c_2(x,\varrho) &= c^0 + c^1 |\varrho - f_t(x)| \nonumber \\
c_2(x,\varrho) &= c^0 + c^1 (\varrho - f_t(x))^2 \nonumber \\
c_2(x,\varrho) &= c^0 + c^1 f_t(x)(\varrho - f_t(x))^2 \nonumber
\end{align}
holds, for some constants $c^0$, $c^1$, then $k_2$ will be linear in
\begin{align}
\mathbb{E}_{D_n}\left\{\mathbb{E}_X \left[|\rho(X)-f_t(X)|\right]\right\} \nonumber \\
\mathbb{E}_{D_n}\left\{\mathbb{E}_X \left[(\rho(X)-f_t(X))^2\right]\right\} \nonumber \\ \mathbb{E}_{D_n}\left\{\mathbb{E}_X \left[f_t(X)(\rho(X)-f_t(X))^2\right]\right\}
\end{align}
respectively. As discussed in the main manuscript, this reduces the estimation of $k_2(n)$ to estimating the `learning curve' of a risk score, and expectations of risk score accuracy measures such as those above over $X$ and $D_n$ can be readily estimated for small $n$ given training samples $X,Y$. 

In more general cases where simple forms of $c_2(x,\varrho)$ cannot be assumed, we claim that if $c_2(x,\varrho)$ is smooth in $\varrho$, we should generally expect $k_2(n)$ to be approximately linear in the expected mean-square error of the risk score; that is, $k_2(n) \propto \mathbb{E}_{D_n}\{MSE(\rho)\} = \mathbb{E}_{D_n} \{\xi_t^2(\rho,f_t)\}$.

We work from the following heuristic:
\begin{displayquote}
For any given sample and a range of possible risk scores for that sample, one of which is equal to $f_t$, the intervention taken will minimise the expected cost for the risk score which is equal to $f_t$.
\end{displayquote}
This is equivalent to
\begin{equation}
\arg \min_{\varrho} c_2(x,\varrho) = f_t(x) \nonumber
\end{equation}
for all $x$ in the domain of $X$. We suppose firstly that $c_2(x,\varrho)$ is smooth in $\varrho$, and write
\begin{align}
c_2(x,\varrho) &= c_2\left(x,f_t(x)\right) + \frac{1}{2}\frac{\partial^2 c_2}{\partial \varrho^2}\left(x,f_t(x)\right)\left(\varrho-f_t(x)\right)^2 + O\left((\varrho-f_t(x))^3\right) \, , \nonumber
\end{align}
noting that $\frac{\partial c_2}{\partial \varrho}\left(x,f_t(x)\right)=0$ by the heuristic above.

The value $\frac{\partial^2 c_2}{\partial \varrho^2}\left(x,f_t(x)\right)$ represents the curvature with respect to $\varrho$ of the function $c_2(x,\varrho)$ about $\varrho=f_t(x)$. Practically, this corresponds to the tolerance or robustness of the intervention: the amount of cost incurred due to a given deviation of the risk score from $f_t(x)$. We claim that this quantity will thus have relatively low variation across values of $x$, as the degree of robustness should be roughly constant. This means that:
\begin{equation}
\frac{\partial^2 c_2}{\partial \varrho^2}\left(x,f_t(x)\right) \approx \iota_2 \nonumber
\end{equation}
where $\iota_2$ does not depend on $x$. Given this, we have
\begin{align}
\mathbb{E}_X \left[c_2(X;\varrho)\right] &= \mathbb{E}_X\bigg[ c_2\left(x,f_t(x)\right) \nonumber \\
&\phantom{=} + \frac{\partial^2 c_2}{\partial \varrho^2}\left(x,f_t(x)\right)(\varrho(X)-f_t(X))^2 \nonumber \\ 
&\phantom{=}\left. + O\left(\sup_x(\varrho(x)-f_t(x))^3\right)\right] \nonumber \\
&\approx \mathbb{E}_X \left[c_2\left(x,f_t(x)\right)\right] \nonumber \\
&\phantom{=} + \frac{\partial^2 c_2}{\partial \varrho^2}\left(x,f_t(x)\right)\mathbb{E}_X\left[(\varrho(X)-f_t(X))^2\right] \nonumber \\ 
&\phantom{=} + O\left(\sup_x(\varrho(x)-f_t(x))^3\right) \nonumber \\
&= \iota_1 + \iota_2 \, \textrm{MSE}(\varrho) + O\left(\sup_x(\varrho(x)-f_t(x))^3\right) \, , \nonumber 
\end{align}
where $\varrho$ is fixed and $\iota_1$, $\iota_2$ do not depend on $\varrho$ and are hence independent of $D_n$, and $\textrm{MSE}(\varrho)$ is the standard mean-square error of $\varrho$. Hence
\begin{align}
k_2(n) &= \mathbb{E}_{D_n} \left\{\mathbb{E}_X \left[c_2(X;\varrho)\right]\right\} \nonumber \\
&\approx \iota_0 + \iota_2 \mathbb{E}_{D_n} \left\{MSE(\varrho)\right\} \, , \nonumber
\end{align}
so $k_2(n)$ is approximately linear in $\mathbb{E}_{D_n} \left\{MSE(\varrho)\right\}$.

\subsection{Estimation of $k_2(n)$ when intervention is governed by a threshold on the risk score}

We briefly describe a second presumably common setting in which we have a single intervention which we may use, which has a proportional effect on the risk of $Y=1$ (that is, $g_t(x)=(1-\alpha) f_t(x)$, with $\alpha<1$). We intervene on a sample if their risk score exceeds a particular threshold $\rho_0$. The cost function $c_2(x,\varrho)$ is now discontinuous in $\varrho$, and we no longer expect that $k_2(n)$ is proportional to the expected mean-square error o $\rho$, but we may derive the form of $k_2$ in this case nonetheless.

We assume that the intervention has a fixed cost $\gamma_i$, and that an event $Y=1$ has a fixed cost $\gamma_y$. Then
\begin{align}
c_2(x,\varrho) &= \mathbb{E}_{C_2}\left\{(\textrm{Cost of an event}) + (\textrm{Cost of intervening}) \right\} \nonumber \\
&= \begin{cases} 
\gamma_y f_t(x) + 0 &\textrm{if  } \varrho < \rho_0 \\
\gamma_y \alpha f_t(x) + \gamma_i &\textrm{if  } \varrho \geq \rho_0
\end{cases} \, , \nonumber
\end{align}
disregarding potential baseline costs common to all samples. We may now apply the heuristic above more directly, by presuming that the threshold $\rho_0$ is chosen so as to minimise the expectation of $c_2(X,\rho)$ over $X$ under the assumption that $\rho(X)=f_t(X)$. In other words, we choose the threshold that gives us the best outcome assuming the risk score is correct, as in the previous subsection.

This implies that for $f_t(x)<\rho_0$, we have $\gamma_y f_t(x) < \gamma_y \alpha f_t(x) + \gamma_i$ (if true risk is below the threshold, it is cheaper not to intervene) and for $f_t(x) \geq \rho_0$, we have $\gamma_y f_t(x) \geq \gamma_y \alpha f_t(x) + \gamma_i$ (if true risk is above the threshold, it is cheaper to intervene). 

We now have:
\begin{align}
c_2(x,\varrho) - c_2(x,f_t(x)) &= \begin{cases}
\gamma_y f_t(x) &\textrm{if  } \varrho < \rho_0, f_t(x) < \rho_0 \\
\gamma_y \alpha f_t(x) + \gamma_i - \gamma_y f_t(x) &\textrm{if  } \varrho \geq \rho_0, f_t(x) < \rho_0 \\
\gamma_y f_t(x) - (\gamma_y \alpha f_t(x) + \gamma_i) &\textrm{if  } \varrho < \rho_0, f_t(x) \geq \rho_0 \\
\gamma_y \alpha f_t(x) + \gamma_i &\textrm{if  } \varrho \geq \rho_0, f_t(x) \geq \rho_0 
\end{cases} \nonumber \\
&= 1_{(\varrho<\rho_0)\, XOR\, (f_t(x)<\rho_0)}|\gamma_y (\alpha-1) f_t(x) + \gamma_i| \, , \nonumber
\end{align}
and, denoting $\delta_{\varrho}(x)=(\varrho(x)<\rho_0)\, XOR \, (f_t(x)<\rho_0)$ and presuming $\lim_{n \to \infty} k_2(n)$ exists and is finite, we have (since as $n \to \infty$ we have $\rho \to f_t$):
\begin{align}
k_2(n) &= k_2(n) - \lim_{n \to \infty} k_2(n) + \lim_{n \to \infty} k_2(n) \nonumber \\
&= \lim_{n \to \infty} k_2(n) + \mathbb{E}_{D_n} \left\{\mathbb{E}_X \left[c_2(X,\rho) - c_2(X,f_t)\right]\right\} \nonumber \\
&= iota_0 + \mathbb{E}_{D_n} \left\{\mathbb{E}_X \left[1_{\delta_{\rho}(X)}|\iota_1 f_t(x) + \iota_2|\right]\right\} \, , \nonumber \\
\end{align}
where $\iota_0$, $\iota_1$, $\iota_2$ are constant, and potentially estimable from several observations of $k_2$. 

This form is unsurprising: if a risk score $\rho(X)$ is such that the sign of $\rho(X)-\rho_0$ agrees with the sign of $f_t(X)-\rho_0$, it will have identical cost to a risk score $\rho(X)$ which agrees with $f_t(X)$ everywhere. That is, the cost associated with a risk score achieves its minimum for risk scores not almost-everywhere equal to $f_t$.

\clearpage

\section{Parametric OHS estimation}
\label{apx:parametric}

In this section, we describe estimation of optimal holdout set sizes by explicit parametrisation of the function $k_2(n)$. As in section~\ref{sec:parametric} in the main paper, we take $k_2(n)=k_2(n; \theta)$. We will take $k_2'$, $k_2''$, $\ell'$ to mean partial derivatives with respect to $n$, and the shorthand $\Theta=(N,k_1,\theta)$ and $\Theta_0=\mathbb{E}(\Theta)$. We will also write $n_*=n_*(\Theta)$, $\ell(n_*)=\ell\{n_*(\Theta);\Theta\}$, $n_0=n_*(\Theta_0)$ and $\ell(n_0)=\ell\{n_*(\Theta_0),\Theta_0\}$ for brevity. We presume that $\Theta$ is an unbiased estimate of $\Theta_0$, so $\Theta_0$ corresponds to `true' parameter values. 

We firstly develop asymptotic confidence intervals for parametric OHS estimates to link error in parameter estimates to error in optimal size. The sample-size $m$ used in the following denotes a proxy for effort expended in estimating $\Theta_0$. 

\begin{theorem}
\label{thm:parametric_robustness}
Assume that $k_2''(n;\theta)$, $k_2'(n;\theta)$ and $\nabla_{\theta} k_2(n;\theta)$ are continuous in $n$ and $\theta$ in some neighbourhood of $(n_0,\Theta_0)$, and that $\Theta_0$ parametrizes a setting satisfying assumptions~\ref{item:k1_indep_pi}-\ref{item:k2_2nd_der}. Suppose that $\Theta$ behaves as a mean of of $m$ appropriately-distributed samples in satisfying $\sqrt{m}(\Theta-\Theta_0) \to N\left(0,\Sigma\right)$ in distribution where $\Theta_0$ does not depend on $m$, that an estimate $\widehat{\Sigma}$ of $\Sigma$ is available which is independent of $\Theta$ and satisfies $||\widehat{\Sigma} - \Sigma||_2 \to 0 $ in distribution, and that $n_0$ is finite and unique as above. Then denoting 
\begin{equation}
\beta_{\Theta} =\frac{\partial^2 \ell}{\partial n \partial \Theta_i} \bigg/ \frac{\partial^2 \ell}{\partial n^2}, \hspace{30pt}
\gamma_{\Theta} = \frac{\partial \ell}{\partial \Theta_i} \nonumber
\end{equation}
we may uniquely define $n_0=\left\{n:\ell'(n;\Theta_0)=0\right\}$ and we have
\begin{equation}
\sqrt{m}(n_*-n_0) \to N\left(0,\beta_{\Theta_0}^t  \Sigma \beta_{\Theta_0} \right), 
\hspace{30pt}
\sqrt{m}\left\{ \ell(n_*)-\ell(n_0)\right\} \to N\left(0,\gamma_{\Theta_0}^t  \Sigma \gamma_{\Theta_0} \right) \nonumber 
\end{equation}
in distribution, and denoting $z_{\alpha}=\Phi^{-1}\left(1-\alpha/2\right)$, the confidence intervals
\begin{equation}
I_{\alpha}(\Theta,\hat{\Sigma})=\left[n_*(\Theta) \pm z_{\alpha}\sqrt{\frac{\beta_{\Theta}^t\widehat{\Sigma}\beta_{\Theta}}{m}}   \hspace{5pt} \right],  \hspace{20pt}
J_{\alpha}(\Theta,\hat{\Sigma})=\left[\ell(n_*) \pm z_{\alpha}\sqrt{\frac{\gamma_{\Theta}^t\widehat{\Sigma}\gamma_{\Theta}}{m}}  \hspace{5pt} \right] \nonumber
\end{equation}
%
satisfy $P\left\{ n_0 \in I_{\alpha}(\Theta,\hat{\Sigma})\right\} \to 1-\alpha$ and $P\left\{ \ell(n_0) \in J_{\alpha}(\Theta,\hat{\Sigma})\right\} \to 1-\alpha$ as $m \to \infty$. 
\end{theorem}

The proof is given in Supplement~\ref{supp_sec:parametric_proof} below. A consequence is that for sufficiently accurately estimated costs, the OHS will be a non-trivial size:

\begin{corollary}
Under assumptions of Theorems~\ref{thm:ohs_exists},~\ref{thm:parametric_robustness}, 
$P\{ 1< n_*(\Theta) < N \} \to 1$ 
%
as $m \to \infty$. 
\end{corollary}

In light of the proportionality assumption in Section~\ref{sec:practicalities}, and the tendency of the accuracy of a risk score with number of training samples (`learning curve') to follow a power-law form \citep{viering21}, we recommend considering such a parametric form for $k_2$ (i.e.~$k_2(n; \theta)=a n^{-b} + c$ with $\theta=(a,b,c)$), and provide explicit asymptotic confidence intervals for this setting in Supplement~\ref{supp_sec:power_law_explicit}. Examples of variation in $n_*$ and $\ell(n_*)$ with a power-law form for $k_2$, are shown in Supplementary Figures~\ref{supp_fig:partials_nstar}, \ref{supp_fig:partials_mincost}.

Note that confidence intervals must be interpreted with with care: if the sampling distributions for $k_1$ and $\theta$ admit the possibility that assumptions of Theorem~\ref{thm:ohs_exists} are violated such that
$P\left[k_1<\lim \inf_{n\to\infty}\{k_2(n,\theta)\}\right]>0$  
%
then the standard error of $n_*$ does not exist, as $n_*$ can be undefined. Finite-sample confidence intervals may be constructed by bootstrapping (see function \texttt{ci\_ohs()} in our R package \texttt{OptHoldoutSize}).

\sloppy
Our parametric algorithm assumes $\Theta$ is estimated from a multiset $\mathbf{n}$ of values in $\{1, \dots, N\}$ and estimates $\mathbf{d}$ of $k_2(n)$ for each $n \in \mathbf{n}$ with known finite sampling variances $\boldsymbol{\sigma}^2$. For certain multisets $\mathbf{n}$, estimates of $\Theta$ will not converge; for instance, if $\mathbf{n}$ contains only a single value repeated. The value $m$ in Theorem~\ref{thm:parametric_robustness} should be interpreted as an `effective' population size, such that $\sqrt{m}\left\{\Theta(\mathbf{n})-\Theta_0 \right\} \to N\left(0,\Sigma\right)$ in distribution. 

\sloppy
Given that our eventual aim to estimate the OHS with minimal error, we suggest the following way to iteratively select a new value $\tilde{n}$ at which an estimate $\hat{k_2}(\tilde{n})$ of $k_2(\tilde{n})$ should be made, given a set $\mathbf{n}$ of points at which estimates $\mathbf{k_2}$ of $k_2(\mathbf{n})$ have been made already. We denote by $\Theta(\mathbf{n},\mathbf{k_2},\boldsymbol{\sigma})$, $\hat{\Sigma}(\mathbf{n},\mathbf{k_2},\boldsymbol{\sigma})$ and $I_{\alpha}(\mathbf{n},\mathbf{k_2},\boldsymbol{\sigma})$ respectively the estimates of $\Theta_0$, $\lim_{m \to \infty} \textrm{var}\left[\sqrt{m}\left\{ \Theta(\mathbf{n},\mathbf{k_2},\boldsymbol{\sigma}) - \Theta_0\right\} \right]$ and the width of the confidence interval $I_{\alpha}\left\{\Theta(\mathbf{n},\mathbf{k_2},\boldsymbol{\sigma}),\hat{\Sigma}(\mathbf{n},\mathbf{k_2},\boldsymbol{\sigma})\right\}$. Suppose we have the option of estimating $d(n)$ for one value of $n \in \{1,\dots,N\}$ with known variance $\textrm{var}\left\{d(n)\right\}=\sigma^2$. We select $\tilde{n}$ as:
\begin{equation}
\tilde{n}=\arg \min_n \mathbb{E}_{d(n) \sim N\left[k_2\left\{n,\Theta(\mathbf{n},\mathbf{k_2},\boldsymbol{\sigma})\right\},\sigma^2\right]}\left[ I_{\alpha}\left\{\mathbf{n}\cup n,\mathbf{k_2} \cup \hat{k_2}(n),\boldsymbol{\sigma} \cup \sigma\right\}\right] \, , \label{eq:parametric_nextn}
\end{equation}
that is, `select the $\tilde{n}$ which will minimize the expected OHS confidence interval width if added to our set $\mathbf{n}$, with expectation computed with respect to our current parameter estimates'. If no minimum exists, $\tilde{n}$ is selected uniformly from ${1,\dots,N}$. Algorithm~\ref{alg:parametric}, an expanded version of algorithm~\ref{alg:parametric_approximate} in the main manuscript, shows our full parametric estimation procedure.


\begin{algorithm}[h]
\begin{algorithmic}[1]
\State $\mathbf{n},\mathbf{k_2}, \boldsymbol{\sigma}^2 \gets$ some initial values $\mathbf{n}$ with $(\boldsymbol{k_2})_i=\hat{k_2}(n_i) \approx k_2(n_i)$, $(\boldsymbol{\sigma}^2)_i = \textrm{var}(\hat{k_2}(n_i))$\;
\While{$|\mathbf{n}| < n_{add}$}
\State Find best new value $\tilde{n}$ to add to $\mathbf{n}$ as per formula~\ref{eq:parametric_nextn} \;
\State Estimate $\hat{k_2}(\tilde{n}) \approx k_2(\tilde{n})$ \;
\State $\mathbf{n} \gets (\mathbf{n} \cup \tilde{n})$, $\mathbf{k_2} \gets \left\{ \mathbf{k_2} \cup \hat{k_2}(\tilde{n})\right\}$, $\boldsymbol{\sigma}^2 \gets \boldsymbol{\sigma}^2 \cup \textrm{var}\left\{\hat{k_2}(\tilde{n})\right\}$ \;
\EndWhile
\State Re-estimate OHS $n_*^{final}=n_*\left\{\Theta\left(\mathbf{n},\mathbf{k_2},\boldsymbol{\sigma} \right)\right\}$ \;
\State \Return $n_*^{final}$
\caption{Parametric OHS estimation; with $n_{add}$ estimates of $k_2(\cdot)$}
\label{alg:parametric}
\end{algorithmic}
\end{algorithm}

\clearpage

\subsection{Explicit partial derivatives for $n^*$, $\ell$ with power-law parametrisation}
\label{supp_sec:power_law_explicit}

If we assume a power-law form of $k_2$, parametrised by $\theta=(a,b,c,k_1,N)$;
\begin{equation}\label{powerlaw}
k_2(n;\theta)=a n^{-b} + c \, , \nonumber
\end{equation}
then we have
\begin{align}
\frac{\partial n_*}{\partial a} &= \frac{1}{a}\left(\frac{b N n_* - (b-1)n_*^2}{b(b+1)N - b(b-1)n_*}\right) \nonumber \\
\frac{\partial n_*}{\partial b} &= 
\frac{N n_*(b \log(n_*)-1) - n_*^2 \left((b-1)\log(n_*)-1\right)}{b(b+1)N - b(b-1)n_*}
\nonumber \\
\frac{\partial n_*}{\partial c} &= 
\frac{1}{a}\left(\frac{n_*^{b+2}}{b(b+1)N - b(b-1)n_*}\right)  \nonumber \\
\frac{\partial n_*}{\partial k_1} &= 
\frac{1}{a}\left(\frac{-n_*^{b+2}}{b(b+1)N - b(b-1)n_*}\right)  \nonumber \\
\frac{\partial n_*}{\partial N} &= 
\frac{b n_*}{b(b+1)N - b(b-1) n_*} \, , \nonumber
\end{align}
and, more simply
\begin{align}
\frac{\partial }{\partial a}\ell(n_*;\theta) &= (N-n_*)n_*^{-b} \nonumber \\
\frac{\partial }{\partial b}\ell(n_*;\theta) &= -\log(n_*) (N-n_*) a n_*^{-b} \nonumber \\
\frac{\partial }{\partial c}\ell(n_*;\theta) &= N-n_* \nonumber \\
\frac{\partial }{\partial k_1}\ell(n_*;\theta) &= n_* \nonumber \\
\frac{\partial }{\partial N}\ell(n_*;\theta) &= a n_*^{-b} + c \, . \nonumber
\end{align}

\clearpage

\subsection{Proof of Theorem~\ref{thm:parametric_robustness}}
\label{supp_sec:parametric_proof}

\begin{reptheorem}{thm:parametric_robustness}
Assume that $k_2''(n;\theta)$, $k_2'(n;\theta)$ and $\nabla_{\theta} k_2(n;\theta)$ are continuous in $n$ and $\theta$ in some neighbourhood of $(n_0,\Theta_0)$, and that $\Theta_0$ parametrizes a setting satisfying assumptions~\ref{item:k1_indep_pi}-\ref{item:k2_2nd_der}. Suppose that $\Theta$ behaves as a mean of of $m$ appropriately-distributed samples in satisfying $\sqrt{m}(\Theta-\Theta_0) \to N\left(0,\Sigma\right)$ in distribution where $\Theta_0$ does not depend on $m$, that an estimate $\widehat{\Sigma}$ of $\Sigma$ is available which is independent of $\Theta$ and satisfies $||\widehat{\Sigma} - \Sigma||_2 \to 0 $ in distribution, and that $n_0$ is finite and unique as above. Then denoting 
\begin{equation}
\beta_{\Theta} =\frac{\partial^2 \ell}{\partial n \partial \Theta_i} \bigg/ \frac{\partial^2 \ell}{\partial n^2}, \hspace{30pt}
\gamma_{\Theta} = \frac{\partial \ell}{\partial \Theta_i} \, , \nonumber
\end{equation}
we may uniquely define $n_0=\left\{n:\ell'(n;\Theta_0)=0\right\}$ and we have
\begin{equation}
\sqrt{m}(n_*-n_0) \to N\left(0,\beta_{\Theta_0}^t  \Sigma \beta_{\Theta_0} \right), 
\hspace{30pt}
\sqrt{m}\left\{ \ell(n_*)-\ell(n_0)\right\} \to N\left(0,\gamma_{\Theta_0}^t  \Sigma \gamma_{\Theta_0} \right) \label{eq:nstar_loss_asymptotic_normality}
\end{equation}
in distribution, and denoting $z_{\alpha}=\Phi^{-1}\left(1-\alpha/2\right)$, the confidence intervals
\begin{equation}
I_{\alpha}(\Theta,\hat{\Sigma})=\left[n_*(\Theta) \pm z_{\alpha}\sqrt{\frac{\beta_{\Theta}^t\widehat{\Sigma}\beta_{\Theta}}{m}}   \hspace{5pt} \right],  \hspace{20pt}
J_{\alpha}(\Theta,\hat{\Sigma})=\left[\ell(n_*) \pm z_{\alpha}\sqrt{\frac{\gamma_{\Theta}^t\widehat{\Sigma}\gamma_{\Theta}}{m}}  \hspace{5pt} \right] \nonumber
\end{equation}
satisfy $P\left\{ n_0 \in I_{\alpha}(\Theta,\hat{\Sigma})\right\} \to 1-\alpha$ and $P\left\{ \ell(n_0) \in J_{\alpha}(\Theta,\hat{\Sigma})\right\} \to 1-\alpha$ as $m \to \infty$. \end{reptheorem}

As above, we consider $n_*$ as a function of parameters $\Theta=(N, k_1,\theta)$ (where $k_2(\cdot)=k_2(\cdot; \theta)$), write $n_*=n_*(\Theta)$, set $\Theta_0=E(\Theta)$, $n_0=n_*(\Theta_0)$ and $\ell(n_0)=\ell(n_0;\Theta_0)$ and  $\ell(n_*)=\ell(n_*(\Theta),\Theta)$. As discussed above, $n_*$ and $\ell(n_*)$ do not generally have means or standard errors.

\begin{proof}
From 
$\ell(n)=k_1 n + k_2(n;\theta)(N-n)$ and 
$n_*=\{n:\ell'(n;\Theta)=0\}$, where such $n_*$ is unique, we have (as per section~\ref{sec:parametric})
\begin{align}
(\nabla n_*)_i = \frac{\partial n_*}{\partial \Theta_i} = \frac{\frac{\partial^2 l}{\partial n \partial \Theta_i}}{\frac{\partial^2 \ell}{\partial n^2}}  = (\beta_{\Theta_0})_i\nonumber \\
\left(\nabla \ell(n_*)\right)_i = \frac{\partial \ell}{\partial \Theta_i} = (\gamma_{\Theta_0})_i\nonumber
\end{align}
for all components $\Theta_i$ of $\Theta$. Thus partial derivatives of $n_*$ exist as long as 
\begin{equation}
\frac{\partial^2 \ell}{\partial n^2}>0 \, . \nonumber
\end{equation}
By assumption, $\ell(\cdot;\Theta_0)$ has a minimum at $n_0$. Since
\begin{equation}
\frac{\partial^2 \ell}{\partial n^2}=\frac{\partial^2 }{\partial n^2} k_2(n;\theta) - 2\frac{\partial}{\partial n} k_2(n;\theta) \, , \nonumber
\end{equation}
where both terms are continuous in a neighbourhood of $n_0$, $\Theta_0$ by assumption, the value of $\frac{\partial^2 \ell}{\partial n^2}$ must be positive in some (possibly smaller) neighbourhood $R_{\delta}$ of $(n_0,\Theta_0)$ of width $2\delta$, and hence all partial derivatives of $n_*$ and $\ell(n_*)$ are defined (and indeed continuous) in $R_{\delta}$. 
Within $R_{\delta}$ we have
\begin{align}
n_*(\Theta) &= n_*(\Theta_0) + (\nabla n_*|_{\Theta=\Theta_0})\cdot \left( \Theta - \Theta_0\right) + O\left(||\Theta-\Theta_0||_2\right) \nonumber\\
&= n_0 + \beta_{\Theta_0}^t\cdot \left( \Theta - \Theta_0\right) + O\left(||\Theta-\Theta_0||_2\right) \label{eq:nstarfirstorder} \\
\ell(n_*) &= \ell(n_*(\Theta_0);\Theta_0) + (\nabla \ell(n_*)|_{\Theta=\Theta_0})\cdot \left( \Theta - \Theta_0\right) + O\left(||\Theta-\Theta_0||_2\right) \nonumber\\
&= \ell(n_0) + \gamma_{\Theta_0}^t\cdot \left( \Theta - \Theta_0\right) + O\left(||\Theta-\Theta_0||_2\right) \, , \label{eq:lossfirstorder}
\end{align}
from which, given the assumption of asymptotic normality of $\Theta$, assertions~\ref{eq:nstar_loss_asymptotic_normality} 
follow. We note that despite this convergence in distribution, $n_*$ and $\ell(n_*)$ do not generally have first or second moments for finite $m$.

We now have
\begin{align}
P\left(n_0 \geq n_*(\Theta) + z_{\alpha}\sqrt{\frac{\beta_{\Theta}\widehat{\Sigma}\beta_{\Theta}^t}{m}} \right) 
&= P\left(\frac{\sqrt{m}}{z_{\alpha}}(n_0 - n_*(\Theta)) \geq \sqrt{\beta_{\Theta}^t\widehat{\Sigma}\beta_{\Theta}}\right) \nonumber \\
&= P\left(\frac{\sqrt{m}}{z_{\alpha}}(n_0 - n_*(\Theta)) \geq \left(\beta_{\Theta_0}^t\Sigma\beta_{\Theta_0} + \right.\right. \nonumber \\
&\phantom{=} \phantom{P((n_0 - n_*(\Theta)) \geq} 
\beta_{\Theta}^t\left(\widehat{\Sigma}-\Sigma\right)\beta_{\Theta} +  \nonumber \\
&\phantom{=} \left.\left.\phantom{P((n_0 - n_*(\Theta)) \geq} 
(\beta_{\Theta}-\beta_{\Theta_0})^t \Sigma \left(\beta_{\Theta}+\beta_{\Theta_0}\right)\right)^{\frac{1}{2}}\right) \nonumber \\
&\to  P\left(\frac{\sqrt{m}}{z_{\alpha}}(n_0 - n_*(\Theta)) \geq \sqrt{\beta_{\Theta_0}^t\Sigma\beta_{\Theta_0}} \right) \nonumber \\
&=\frac{\alpha}{2} \, , \nonumber
\end{align}
since, by the assumption of convergence of $\widehat{\Sigma}$
\begin{align}
\left|\beta_{\Theta}^t\left(\Sigma-\widehat{\Sigma}\right)\beta_{\Theta}\right| &\leq ||\beta_{\Theta}||_2||\Sigma-\widehat{\Sigma}||_2 \nonumber \\
&\to_p 0 \, , \nonumber
\end{align} 
and, since $P(\Theta\in R_{\delta}) \to 1$ by the asymptotic normality of $\Theta$, we have from~\ref{eq:nstarfirstorder}
\begin{align}
\left|(\beta_{\Theta}-\beta_{\Theta_0})^t \Sigma \left(\beta_{\Theta}+\beta_{\Theta_0}\right)\right| &= O\left(||\beta_{\Theta}-\beta_{\Theta_0}||_2\right) \nonumber \\
&\to_p 0 \, . \nonumber
\end{align}
Thus, combining with the corresponding limit for the lower end of $I_{\alpha}(\Theta,\hat{\Sigma})$:
\begin{equation}
P(n_0 \in I_{\alpha}(\Theta,\hat{\Sigma})) \to 1-\alpha \nonumber
\end{equation}
as required. An identical argument holds for $J_{\alpha}(\Theta,\hat{\Sigma})$.

\end{proof}

\clearpage

\section{Estimation of OHS by Bayesian Emulation}
\label{apx:emulation}

Our second algorithm for estimation of optimal holdout sizes uses Bayesian emulation~\citep{brochu2010tutorial}. In many cases, it may be difficult or unrealistic to provide a precise parametric form for the function $k_2(n)$. The function depends both on the `learning curve' of the risk score, which may be complex~\citep{viering21}, and the relationship of the risk score accuracy to the accrued cost, which may be nonlinear. Here, we propose a second algorithm which is less reliant on assuming a particular parametric form for $k_2(n)$.

As in the main manuscript, we approximate the cost function $\ell$ as an `emulator' modelled as a Gaussian process, and take the minimum of its posterior mean over $n$ as our OHS estimate. It is worth noting that whilst gaining an accurate approximation of the cost function is important, the main goal is to ascertain the minimum of this function, not provide a universally effective approximation at all points. Therefore, we aim to choose the location of design points $\mathbf{n}$
in order to efficiently obtain the minimum of the cost function, and hence the OHS.

First we must construct an emulator which approximates the cost function. We begin with an initial set of design points  $\mathbf{n}$ and their corresponding observed noisy cost estimates $\mathbf{d}$.
The prior for our emulator is, following~\cite{vernon2018bayesian}, 
    $\ell(n) = m(n,\Theta) + u(n)$
%
with mean function $m(n,\Theta)=k_{1}n + k_2(n;\theta) (N - n)$, given some initial estimate of $\Theta=(N,k_1,\theta)$, and $u(n)$ a zero-mean Gaussian process 
\begin{equation}
    u(n) \sim \mathcal{GP}\left\{0,k(n,n')\right\} \hspace{20pt} k(n,n') = \sigma_u^2\exp{\left\{-\left(\frac{n - n'}{\zeta}\right)^2\right\}} \nonumber
\end{equation}
%
where $k$ is chosen to enforce smoothness in $\ell(n)$, though other covariance functions having varying degrees of smoothness could be used. 
The hyperparameters $\theta$, $\sigma_u$ and $\zeta$ are problem-specific and must be specified; however, we will show that for sufficiently large $|\mathbf{n}|$ mis-specification of $\theta$, $\sigma_u$ and $\theta$ is overcome.

Since $\mathbf{n}$ may be a multiset, we take $\mathbf{n}^1$ as the set of unique values in $\mathbf{n}$, with $\mathbf{d}^1$, $\boldsymbol{\sigma}^1$ defined correspondingly with $d^1_i$ as 
the mean of $\{d_j:n_j =n^1_i\}$ with sample 
standard error $\sigma^1_i$,
%
noting that $\sigma^1_i$ may change with $i$. We state 
 $d^1_i = \ell(n^1_i) + \mathfrak{e}$ 
%
where $\mathfrak{e} \sim N\left\{0,(\sigma^1_i)^2 = \frac{(\tilde{\sigma}^1_i)^2}{|j:n_j=n^1_i|}\right\}$ and $(\tilde{\sigma}^1_i)^2$ is the sample variance of a single evaluation at $n^1_i$. Alternatively, we may account for the variation in $\mathbf{d}$ through `inactive' variables and opt to use a 'nugget' term; this approach is described in detail in Supplement~\ref{supp_sec:cost_emulation}.

Now with input $n$, with an unevaluated loss value, our emulator specifies that the joint distribution of $\ell(n)$ and our observed output values $\mathbf{d}^1$ is:
\begin{equation}
    \begin{bmatrix}
    \ell(n)\\
    \mathbf{d^1}
    \end{bmatrix} \sim \mathcal{N}\left(\begin{bmatrix}
    m(n,\Theta)\\
    m(\mathbf{n^1},\Theta)
    \end{bmatrix},\begin{bmatrix}
    k(n,n) & k(n,\mathbf{n^1}) \\
    k(\mathbf{n^1},n) & k(\mathbf{n^1},\mathbf{n^1}) + \textrm{diag}\{(\boldsymbol{\sigma^1})^2\}
    \end{bmatrix}\right) \, , \nonumber
\end{equation}
where $m(\mathbf{n^1},\Theta)^T_i=m(n^1_i,\Theta)$, $k(n,\mathbf{n^1})_i = k(\mathbf{n^1},n)^T_i = k(n,n^1_i)$, $k(\mathbf{n},\mathbf{n})_{ij} = k(n^1_i,n^1_j)$, $\textrm{diag}\{(\boldsymbol{\sigma^1})^2\}_{ij} = (\sigma^1_i)^2 1_{i=j}$. 
By obtaining the conditional posterior distribution $\pi_{\mathbf{n}} = \pi\{\ell(n) \mid n,\mathbf{n^1},\mathbf{d^1},\boldsymbol{\sigma^1}\}$ and taking the expectation and variance we gain the Bayes linear update equations \citep{vernon2018bayesian}:
\begin{align}
    \mu(n) &= \mathbb{E}_{\pi_{\mathbf{n}}}\{\ell(n)\} = m(n,\Theta) +  k(n,\mathbf{n^1})\left[k(\mathbf{n^1},\mathbf{n^1}) + \textrm{diag}\{(\boldsymbol{\sigma^1})^2\}\right]^{-1}\{\mathbf{d^1} - m(\mathbf{n^1},\Theta)\} \nonumber \\
    \Psi(n) &= \textrm{var}_{\pi_{\mathbf{n}}}\{\ell(n)\} = k(n,n) - k(n,\mathbf{n^1})[k(\mathbf{n^1},\mathbf{n^1}) + \textrm{diag}\{(\boldsymbol{\sigma^1})^2\}]^{-1}k(\mathbf{n^1},n) \, . \nonumber
\end{align}
In algorithm~\ref{alg:parametric}, selection of new design points should generally favour well-spaced points across $\{1,\dots,N\}$ for both exploration and exploitation. Here, since we wish both to estimate the OHS accurately but also locally approximate $\ell$ well, we choose the next $n$ in a way which predominantly but not completely favours exploitation. We use the `expected improvement', which measures discrepancy between the emulator at a certain design point and the known minimum $EI(\cdot)$~\citep{brochu2010tutorial}:
\begin{equation}\label{EI}
    EI(n) = \{d^{-} - \mu(n)\}\Phi\left(\frac{d^{-} - \mu(n)}{\sqrt{\Psi(n)}}\right) + \sqrt{\Psi(n)}\phi\left(\frac{d^{-} - \mu(n)}{\sqrt{\Psi(n)}}\right) \, , \nonumber
\end{equation}
where $d^{-} = \min_i \{\mathbf{d^1}_i\}$, and 
\begin{equation}
\tilde{n} = \arg \max_{n \in \{1,\dots,N\}} EI(n) 
= \arg \max \left(\mathbb{E}_{\pi_{\mathbf{n}}}\left[\max\{0,d^{-}-\ell(n)\}\right]\right) \, . \nonumber
\end{equation}
%


We see that by formulating the problem in terms of $EI(\cdot)$, there is a natural stopping criterion on the size of $\mathbf{n}$: setting a threshold $EI(\tilde{n})>\tau$ allows us to specify that for each iteration that we expect total cost to improve by at least $\tau$ over our current known minimum $d^{-}$. Examples of $\mu(\cdot),\Psi(\cdot), EI(\cdot)$ are shown in Supplementary Figure~\ref{supp_fig:emulator}. 
This leads to algorithm~\ref{alg:emulation} for OHS estimation by Bayesian Emulation (a more precise version of algorithm~\ref{alg:emulation_approximate} in the main manuscript).


\begin{algorithm}[h]
\begin{algorithmic}
\State $\mathbf{n},\mathbf{d} \gets$ some initial values $\mathbf{n}$ with $(\boldsymbol{d})_i=d(n_i) \approx \ell(n)$ \;
\State Coalesce $\mathbf{n}, \mathbf{d}$ into $\mathbf{n^1}, \mathbf{d^1}$ and obtain $\boldsymbol{\sigma^1}$ as above \;
\State Estimate functions $\mu(n)$, $\Psi(n)$, $EI(n)$, with $\Theta = \Theta(\mathbf{n^1},\mathbf{d^1},\boldsymbol{\sigma^1})$ \;
\While{$\max_{n \in \{1,\dots, N\}}\{EI(n)\} > \tau$}
\State $\tilde{n} \gets \arg \max_{n \in \{1,\dots, N\}} EI(n)$ \;
\State Estimate $d(\tilde{n}) \approx k_2(\tilde{n})$ \;
\State $\mathbf{n} \gets (\mathbf{n} \cup \tilde{n})$; $\mathbf{d} \gets (\mathbf{d} \cup d\{\tilde{n})\}$ \;
\State Coalesce $\mathbf{n}, \mathbf{d}$ into $\mathbf{n^1}, \mathbf{d^1}$ and obtain $\boldsymbol{\sigma^1}$ \;
\State Re-estimate functions $\mu(n)$, $\Psi(n)$, $EI(n)$, with $\Theta = \Theta(\mathbf{n^1},\mathbf{d^1},\boldsymbol{\sigma^1})$ \;
\EndWhile
\State \Return $n_*^{final} = \arg \min_{n_i \in \mathbf{n}^1} \left\{d^1_i\right\}$
\caption{Emulation OHS estimation; minimum cost improvement $\tau$ }
\label{alg:emulation}
\end{algorithmic}
\end{algorithm}

Various results on the consistency of the expected improvement algorithm have been proved, albeit in differing settings; either with noiseless observations $\mathbf{d}$~\citep{locatelli1997bayesian,vazquez2010convergence,bull2011convergence} or with noisy observations with known variance~\citep{ryzhov2016convergence}. We prove the following consistency results specifically for the setting of this work in Supplement~\ref{apx:consistency_emulation}.

\begin{theorem}
\label{thm:fixed_points_var}
If $\ell(n)$, $\boldsymbol{\tilde{\sigma}^1}$, and $m(n,\Theta)$ are almost surely bounded and $d^1_i = \ell(n^1_i) + \mathfrak{e}$ where $\mathfrak{e} \sim N\left\{0,\frac{(\tilde{\sigma}^1_i)^2}{|j:n_j=n^1_i|}\right\}$ then for every $n \in \{1,\ldots,N\}$, as the multiplicity of $n$ in $\mathbf{n}$ tends to $\infty$ we have $\mu(n) \longrightarrow \ell(n)$ and $\Psi(n) \to 0$ almost surely with respect to variation in $\mathbf{d}$.
\end{theorem}
also noting the following simple result, proved in Supplement~\ref{supp_sec:proof_ei_lim}: 

\begin{theorem}
\label{thm:ei_lim}
Given the conditions of Theorem \ref{thm:fixed_points_var}, for every $n \in \{1,\ldots,N\}$, as the multiplicity of $n$ in $\mathbf{n}$ tends to $\infty$, 
\begin{equation*}
    EI(n) \longrightarrow 0 \nonumber
\end{equation*}
almost surely with respect to randomness in $\mathbf{d}$
\end{theorem}

These results assert that $\mu(n)$ can eventually approximate any loss function sufficiently well given enough estimates of $\ell$ at all values of $n$. It is not obvious that this is guaranteed by algorithm~\ref{alg:emulation}, although we show that this generally does occur in the following, the proof of which is given in Supplement~\ref{supp_sec:proof_emulation_consistency}:

\begin{theorem}
\label{thm:emulation_consistency}
If $\ell(n)$, $\boldsymbol{\tilde{\sigma}^1}$, and $m(n,\Theta)$ are almost surely bounded and $d^1_i = \ell(n^1_i) + \mathfrak{e}$ where $\mathfrak{e} \sim N\left\{0,\frac{(\tilde{\sigma}^1_i)^2}{|j:n_j=n^1_i|}\right\}$ then under algorithm~\ref{alg:emulation} with $\tau=0$, the value $\mu(\tilde{n})$ converges almost surely to $\ell(\tilde{n})$ for every $\tilde{n} \in \{1,\ldots,N\}$.
\end{theorem}

We characterize the error in $n_*$ using `the number of values of $n$ for which the probability of the true cost at holdout set size $n$ is less than the estimated minimum cost exceeds $1-\alpha$', or formally:
$\left\{n:pr_{\textbf{n}_{\pi}} \left\{ \ell(n) < \mu(n_*)\right\} \geq 1- \alpha  \right\}$, 
%
although this should not be interpreted as a credible set for $n_*$. This is implemented in our R package \texttt{OptHoldoutSize}, available on CRAN.

\subsection{Emulation of cost function with nugget term}
\label{supp_sec:cost_emulation}

Rather than explaining the variation of values in $\mathbf{d}$ corresponding to a design point in $\mathbf{n^1}$ as approximation error of a deterministic loss function, we can explain this variation as the result of not including active variables, being the data $(X,Y)$. Note that as a consequence we are now not emulating a deterministic function $\ell(n)$ as we are not generalising the loss through expectations, we are generalising the loss through omission of the data which generated $\mathbf{d}$. To clarify this distinction we replace the loss function $\ell(n)$ with the stochastic function $\mathcal{E}(n)$.

Now we may specify variation in $\mathbf{d}$ using a `nugget' term $w(n)$, following~\cite{bower2010galaxy}: 
\begin{equation}
    \mathcal{E}(n) = m(n) + u(n) + w(n) \, , \nonumber
\end{equation}
where $m(n)$ and $u(n)$ are as before but now $w(n)$ represents our nugget term, which we again specify as a Gaussian process:
\begin{equation}
    w(n) \sim \mathcal{GP}(0,\kappa(n,n')) \, , \nonumber
\end{equation}
with \begin{equation}
    \kappa(n,n') = \begin{cases}
        \kappa(n) &\mbox{if } n = n' \\
        0 & \mbox{otherwise. }
    \end{cases}
\end{equation}
Since there is less variance in risk scores fitted to larger datasets, we expect less variance in $\mathcal{E}(n)$ for larger $n$, so we specify $\kappa(n)$ as 
a monotonically decreasing function in $n$. 

The joint distribution between $\mathcal{E}(n)$ and $\mathbf{d^1}$ is now:
\begin{equation}
    \begin{bmatrix}
    \mathcal{E}(n)\\
    \mathbf{d^1}
    \end{bmatrix} \sim \mathcal{N}\left(\begin{bmatrix}
    m(n)\\
    m(\mathbf{n^1})
    \end{bmatrix},\begin{bmatrix}
    k(n,n) + \kappa(n) & k(n,\mathbf{n^1}) \\
    k(\mathbf{n^1},n) & k(\mathbf{n^1},\mathbf{n^1}) + \textrm{diag}(\kappa(\mathbf{n^1}))
    \end{bmatrix}\right) \, . \nonumber
\end{equation}

\sloppy
This then gives our Bayes linear update equations in terms of $\pi_{\mathbf{n}}=\pi(\mathcal{E}(n) | n,\mathbf{n^1},\mathbf{d^1})$ as
\begin{align}
   \mu(n) &= \mathbb{E}_{\pi_{\mathbf{n^1}}}(\mathcal{E}(n)) \nonumber \\ 
    &= m(n) +  k(n,\mathbf{n^1})[k(\mathbf{n^1},\mathbf{n^1}) + \textrm{diag}(\kappa(\mathbf{n^1}))]^{-1}(\mathbf{d^1} - m(\mathbf{n^1}))\\
    \Psi(n) &= \textrm{var}_{\pi_{\mathbf{n^1}}}(\mathcal{E}(n)) \nonumber \\
    &= k(n,n) + \kappa(n) - k(n,\mathbf{n})[k(\mathbf{n^1},\mathbf{n^1}) + \textrm{diag}(\kappa(\mathbf{n^1}))]^{-1}k(\mathbf{n^1},n) \, . \nonumber
\end{align}
Note that this differs only slightly from the emulator constructed in section~\ref{sec:emulation}, with the main difference being we now attribute uncertainty in the loss values as an inherent behaviour of our emulator and not in the procedure to obtain these loss values. As a result $\kappa(n)$ does not decrease as the multiplicity of elements of $\mathbf{n}$ increases, which represents a major disadvantage to the uncertainty representation in section~\ref{sec:emulation}. 

One may then be sceptical of the benefit of duplicating design points for this method, and whilst it is possible to use this method without duplication (i.e $\mathbf{n} = \mathbf{n^1}$), the consequence of this would be that we are heavily reliant on a singular sample to locate the minimum which could be misleading. Averaging various samples at the same design point mitigates this potential problem, as does replacing $d^-$ with $\mu^{-} = \min_i \{\mu(\mathbf{n^1}_i)\}$ as detailed in ~\cite{brochu2010tutorial}. Taking the median of samples instead of a weighted mean is more appropriate here as we are not seeking to accurately approximate an expectation, instead we only wish to avoid extreme samples misleading our search for the minimum.

\subsection{Proof of Theorem~\ref{thm:fixed_points_var}}
\label{apx:consistency_emulation}

\begin{reptheorem}{thm:fixed_points_var}
If $\ell(n)$, $\boldsymbol{\tilde{\sigma}^1}$, and $m(n,\Theta)$ are almost surely bounded and $d^1_i = \ell(n^1_i) + \mathfrak{e}$ where $\mathfrak{e} \sim N\left\{0,\frac{(\tilde{\sigma}^1_i)^2}{|j:n_j=n^1_i|}\right\}$ then for every $n \in \{1,\ldots,N\}$, as the multiplicity of $n$ in $\mathbf{n}$ tends to $\infty$ we have $\mu(n) \longrightarrow \ell(n)$ and $\Psi(n) \to 0$ almost surely with respect to variation in $\mathbf{d}$.
\end{reptheorem}

\begin{proof}
Assume W.L.O.G that $(\mathbf{n}^1)_1=n$. Since $\boldsymbol{\tilde{\sigma}^1}$ is bounded, we have (from the distributional assumption of noise in evaluations) $\textrm{var}\left((\mathbf{d}^1)_1\right)=(\boldsymbol{\sigma^1})^2_1 \to 0$, so $(\mathbf{d}^1)_1 \longrightarrow \ell(n)$ almost surely. We now prove that $k(n,\mathbf{n^1})[k(\mathbf{n^1},\mathbf{n^1}) + \textrm{diag}((\boldsymbol{\sigma^1})^2)]^{-1} = (1,0,\ldots,0)$ when $(\boldsymbol{\sigma^1})_1=0$. Now:
\begin{align}
k(n,\mathbf{n^1})[k(\mathbf{n^1},\mathbf{n^1}) + \textrm{diag}((\boldsymbol{\sigma^1})^2)]^{-1} &= (1,0,\ldots,0) \nonumber \\
\Leftrightarrow \hspace{20pt} k(n,\mathbf{n^1}) &= (1,0,\ldots,0)*[k(\mathbf{n^1},\mathbf{n^1}) + \textrm{diag}((\boldsymbol{\sigma^1})^2)] \, , \nonumber
\end{align}
and $k(n,\mathbf{n^1}) = (1,0,\ldots,0)*k(\mathbf{n^1},\mathbf{n^1} + \textrm{diag}((\boldsymbol{\sigma^1})^2)_{-1})$ is true by definition as the first row of $k(\mathbf{n^1},\mathbf{n^1}) + \textrm{diag}((\boldsymbol{\sigma^1})^2)$ is $k(n,\mathbf{n^1})$. Therefore, 
\begin{equation}
    \mu(n) = m(n,\Theta) + (1,0,\ldots,0)(\mathbf{d^1} - m(\mathbf{n^1},\Theta)) = m(n,\Theta) + d(n) - m(n,\Theta) = d(n) = \ell(n) \nonumber
\end{equation}
almost surely, and 
\begin{equation}
    \Psi(n) = k(n,n) - (1,0,\ldots,0)k(\mathbf{n^1},n) = k(n,n) - k(n,n) = 0 \nonumber
\end{equation}
in the limit.

\end{proof}

\subsection{Proof of Theorem~\ref{thm:ei_lim}}
\label{supp_sec:proof_ei_lim}

\begin{reptheorem}{thm:ei_lim}
Given the conditions of Theorem \ref{thm:fixed_points_var}, for every $n \in \{1,\ldots,N\}$, as the multiplicity of $n$ in $\mathbf{n}$ tends to $\infty$, 
\begin{equation}
    EI(n) \longrightarrow 0 \nonumber
\end{equation}
almost surely with respect to randomness in $\mathbf{d}$
\end{reptheorem}

\begin{proof}
From Theorem \ref{thm:fixed_points_var} we have that $\mu(n) \longrightarrow \ell(n) < \infty$ and $d^1_i \longrightarrow \ell(n^1_i)$, so therefore in the limit we can state $pr_{\mathbf{d}}(-\infty <d^{-} - \mu(n) \leq 0)=1$. Indeed, let $j$ be the index such that $n^1_j=n$. If in the limit $d^{-}> \mu(n) = d(n)$ then this implies that  $d^1_j < \min_i \{d^1_i\}$ which is a contradiction. Also note from Theorem \ref{thm:fixed_points_var} that $\Psi(n) \longrightarrow 0$ and that $\Phi\left(\cdot\right) \in (0,1)$, $\phi\left(\cdot\right) \in (0, (2\pi)^{-1/2}]$. As a result the following two scenarios have joint probability 1: 
\begin{itemize}
    \item $d^{-} - \mu(n) = 0$ in the limit: As $\Phi\left(\cdot\right)$, $\phi\left(\cdot\right)$ are bounded and $\Psi(n) = 0$ in the limit, we also have $EI(n) = 0$ in the limit.
    \item $\infty < d^{-} - \mu(n) < 0$ in the limit: As $\Psi(n) = 0$ in the limit, $\Phi\left(\frac{d^{-} - \mu(n)}{\sqrt{\Psi(n)}}\right) = 0$ in the limit. As $\phi\left(\cdot\right)$ is bounded we have that $EI(n) = 0$ in the limit.
\end{itemize}
which proves the corollary.

\end{proof}

\subsection{Proof of Theorem~\ref{thm:emulation_consistency}}
\label{supp_sec:proof_emulation_consistency}

\begin{reptheorem}{thm:emulation_consistency}
If $\ell(n)$, $\boldsymbol{\tilde{\sigma}^1}$, and $m(n,\Theta)$ are almost surely bounded and $d^1_i = \ell(n^1_i) + \mathfrak{e}$ where $\mathfrak{e} \sim N\left\{0,\frac{(\tilde{\sigma}^1_i)^2}{|j:n_j=n^1_i|}\right\}$ then under algorithm~\ref{alg:emulation} with $\tau=0$, the value $\mu(\tilde{n})$ converges almost surely to $\ell(\tilde{n})$ for every $\tilde{n} \in \{1,\ldots,N\}$.
\end{reptheorem}

\begin{proof}

Our overall argument is to show that algorithm~\ref{alg:emulation} leads to the multiplicity of $\tilde{n}$ in $\mathbf{n}$ tending to infinity, from which the result follows from Theorem~\ref{thm:fixed_points_var}. 

To do this, we begin with the following two lemmas, the second of which describes the limiting behaviour of $EI(n)$ according to how often $n$ occurs in $\mathbf{n}$: namely that if the multiplicity of $n$ in $\mathbf{n}$ diverges, the value of $EI(n)$ converges to 0; otherwise, it remains positive. We introduce the index $EI_{\mathbf{n}}(n)$ to indicate the dependence of $EI(n)$ on $\mathbf{n}$ and assume that the function $\ell(n)$ is fixed. For a multiset $\mathbf{n}_i$, we denote $\textrm{mult}_{\mathbf{n_i}}(n)$ as the multiplicity of $n$ in $\mathbf{n_i}$. 

\begin{lemma}\label{lem:matrix_identity}

Suppose $m \times m$ matrix $A$ is symmetric. Denote by $I^1$ the $m \times m$ matrix with $I^1_{ij}=1_{i=j=1}$. Let $x$ be a vector of length $m$ and denote by $A_x$ the matrix $A$ with its top row replaced by $x$. Then for $p$ in any interval containing 0 on which  $A + p I^1$ is invertible we have
\begin{equation}
\frac{\partial}{\partial p} \left(x^T (A+ pI^1)^{-1} x\right) = 
-\frac{\left|A_x \right|^2}
{\left|A+ pI^1\right|^2} \, .
\end{equation}
\end{lemma}

\begin{proof}
If $M(p)$ is invertible in a neighbourhood of $p$ we have $\frac{\partial M^{-1}}{\partial p}=-M^{-1}\frac{\partial M}{\partial p} M^{-1}$, and if $M$ is symmetric with dimensions $m \times m$ and first row $M_1$, then $M I^1 M = M_1 M_1^T$. Since $(A+pI)$ and $A$ differ only in the top row, we have $\textrm{adj}(A+pI)_1=\textrm{adj}(A)_1$, where $\textrm{adj}(\cdot)$ indicates the adjugate matrix and $\cdot_1$ the top row. We now have
\begin{align}
\frac{\partial}{\partial p} \left(x^T (A+ pI^1)^{-1} x\right) &= -x^T (A+ pI^1)^{-1} \frac{\partial (A+ pI^1)}{\partial p}  (A+ pI^1)^{-1} x  \nonumber \\
&= -x^T (A+ pI^1)^{-1} I^1 (A+ pI^1)^{-1} x  \nonumber \\
&= \frac{x^T \textrm{adj}(A+ pI^1) I^1 \textrm{adj}(A+ pI^1) x }{|A+pI^1|^2}  \nonumber \\
&= -\frac{x^T \textrm{adj}(A+ pI^1)_1 \textrm{adj}(A+ pI^1)_1^T x }{|A+pI^1|^2}  \nonumber \\
&= -\frac{x^T \textrm{adj}(A)_1 \textrm{adj}(A)_1^T x }{|A+pI^1|^2}  \nonumber \\
&= -\frac{|A_x|^2}{|A+pI^1|^2}  \nonumber
\end{align}
as required.
\end{proof}

\begin{lemma}
\label{lem:lim_sup_ei}

Let $S_1$ and $S_2$ be disjoint subsets of $[N]=\{1,\dots,N\}$ with $S_1\cup S_2=[N]$. For a multiset $\mathbf{n}$ denote 
\begin{align}
q_1(\mathbf{n}) &= \max_{n \in S_1} \textrm{mult}_{\mathbf{n}}(n) \nonumber \\
q_2(\mathbf{n}) &= \min_{n \in S_2} \textrm{mult}_{\mathbf{n}}(n) \, .
\nonumber \\
\end{align}
Suppose we have infinite sequences $\mathbf{n}$, $\mathbf{d}$. Let $\mathbf{n_i}$, $\mathbf{d_i}$ denote the (multiset) first $i$ elements of each sequence, and let $\mathbf{n^1_i}$ be the unique values of $n$ in $\mathbf{n_i}$ and $\mathbf{d^1_i}$, $\boldsymbol{\tilde{\sigma}^1_i}$ be the associated mean and sample standard deviation, with $\boldsymbol{\tilde{\sigma}^1_i}$ upper bounded.  Suppose that 
$q_1(\mathbf{n_i}) \leq m_1$ for all $i$ and $q_2(\mathbf{n_i}) \to \infty$, and the set $\left\{k_2(n,\Theta_i)= k_2\left(n,\Theta(\mathbf{n_i},\mathbf{d_i},\boldsymbol{\tilde{\sigma}_i})\right): n \in \{1,\dots, N\}, i \in \mathbb{N} \right\}
$ is almost surely asymptotically bounded. Then for sufficiently large $\sigma_u$:
\begin{equation}
\lim \sup_{i \to \infty} EI_{\mathbf{n}_i}(n) = \begin{cases} e_n>0 &\textrm{if } n \in S_1 \\ 0 &\textrm{if } n \in S_2 \end{cases}
\end{equation}
almost surely. 

\end{lemma}

\begin{proof}
We will in fact show that even $\lim \inf EI_{\mathbf{n_i}}(n)>0$ for $n \in S_1$, but $\lim \sup$ will suffice for our purposes. We note that
\begin{equation}
    EI_{\mathbf{n}_i}(n) > 0 \Leftrightarrow  \sqrt{\Psi_{\mathbf{n}_i}(n)}\phi\left(\frac{d^{-}_{\mathbf{n}_i} - \mu_{\mathbf{n}_i}(n)}{\sqrt{\Psi_{\mathbf{n}_i}(n)}}\right) > (\mu_{\mathbf{n}_i}(n) - d^{-}_{\mathbf{n}_i})\Phi\left(\frac{d^{-}_{\mathbf{n}_i} - \mu_{\mathbf{n}_i}(n)}{\sqrt{\Psi_{\mathbf{n}_i}(n)}}\right) \, . \label{eq:ei_expanded}
\end{equation}
We will show that for all $n$, we have 
\begin{equation}
P\left(-\infty < \lim \inf_{i \to \infty} \left(d_{\mathbf{n_i}}^- - \mu_{\mathbf{n_i}}(n)\right)\right)  
=1 \, .\label{eq:lim_inf_dmu}
\end{equation}
\sloppy
By the argument in Theorem~\ref{thm:fixed_points_var} and corollary~\ref{thm:ei_lim} we have for $n \in S_2$ that 
$\lim_{i \to \infty} \Psi_{\mathbf{n_i}}(n)=0$, from which both sides of~\ref{eq:ei_expanded} converge to 0. For $n \in S_1$ we will show $\lim_{i \to \infty} \Psi_{\mathbf{n_i}}(n)>0$, in which case we may define
\begin{equation}
z_{\mathbf{n}_i}(n) =\frac{\mu_{\mathbf{n}_i}(n)-d^{-}_{\mathbf{n}_i}}{\sqrt{\Psi_{\mathbf{n}_i}(n)}} \, , \nonumber
\end{equation}
from which inequality~\ref{eq:ei_expanded} reduces to
\begin{equation}
\phi(z_{\mathbf{n}_i}(n)) > z_{\mathbf{n}_i}(n)\Phi(-z_{\mathbf{n}_i}(n)) \, , \nonumber
\end{equation}
which holds for all $-\infty \leq z_{\mathbf{n}_i}(n) < \infty$. Since $z_{\mathbf{n}_i}(n)$ is asymptotically bounded between positive values, the result follows.

Beginning with $d_{\mathbf{n_i}}^-$, we note that 
$d_{\mathbf{n_i}}^-$ is the minimum of 
\begin{enumerate}
\item Values of $\mathbf{d^1_i}$ corresponding to values of $\mathbf{n^1_i}$ in $S_1$; and \label{itm:n_in_s1}
\item Values of $\mathbf{d^1_i}$ corresponding to values of $\mathbf{n^1_i}$ in $S_2$ \label{itm:n_in_s2}
\end{enumerate}
For sufficiently large $s$, the sequence $\{n_j=(\mathbf{n})_j:j>s\}$ never contains any $n \in S_1$ again; hence, the minimum of item~\ref{itm:n_in_s1} is determined after finitely many $i$ and its limit is finite. Since  $\boldsymbol{\tilde{\sigma}^1_i}$ is upper-bounded, all values of $\mathbf{d^1_i}$ in item~\ref{itm:n_in_s2} converge to finite values in $\{\ell(n): n \in S_2\}$ almost surely. Hence $d_{\mathbf{n_i}}^-$ converges almost surely to a finite value. 

Since $\lim \sup_{i \to \infty}$ and $\lim \inf_{i \to \infty}$ of $m\left(n;\Theta \left(\mathbf{n_i},\mathbf{d_i},\boldsymbol{\tilde{\sigma}^1_i}\right)\right)$ are almost surely finite, all terms in $\mu(n)$ are asymptotically finite, from which equation~\ref{eq:lim_inf_dmu} follows. 

It remains to consider $\Psi_{\mathbf{n}_i}(n)$ for $n \in S_1$. Firstly take $n \in \mathbf{n^1}$ and suppose W.L.O.G that $\boldsymbol{n^1_i}_1=n$. 
Since $n \in S_1$ we have  $\lim_{i \to \infty} \textrm{mult}_{\mathbf{n_i}}(n)>0$ so $\lim_{i \to \infty} (\boldsymbol{\sigma^1_i})_1$ exists and is positive. Denoting $\boldsymbol{\sigma'}$ as $\boldsymbol{\sigma^1_i}$ with 0 substituted for the first element, we have
\begin{align}
\frac{\partial}{\partial (\boldsymbol{\sigma^1_i})^2_1} \Psi_{\mathbf{n_i}}(n) &=  \frac{\partial}{\partial (\boldsymbol{\sigma^1_i})^2_1} \left( k(n,n) - k(n,\mathbf{n^1_i})[k(\mathbf{n^1_i},\mathbf{n^1_i}) + \textrm{diag}((\boldsymbol{\sigma^1_i})^2)]^{-1}k(\mathbf{n^1_i},n) \right) \nonumber \\
&= \frac{\left|k(\mathbf{n^1_i},\mathbf{n^1_i}) + \textrm{diag}((\boldsymbol{\sigma'})^2) \right|^2}{\left|k(\mathbf{n^1_i},\mathbf{n^1_i}) + \textrm{diag}((\boldsymbol{\sigma^1_i})^2)\right|^2} \nonumber \\
&>0 \nonumber
\end{align}
by Lemma~\ref{lem:matrix_identity}; hence $\Psi_{\mathbf{n_i}}(n)$, considered as a function of $(\boldsymbol{\sigma^1_i})^2_1$, is increasing. Given that $\lim_{i \to \infty} (\boldsymbol{\sigma^1_i})_j$ is 0 for  $(\mathbf{n^1_i})_j \in S_2$ and is positive for $(\mathbf{n^1_i})_j \in S_1$, we conclude that $\lim_{i \to \infty} \Psi_{\mathbf{n_i}}(n)$ is positive when $n \in S_1$ and $n \in \mathbf{n^1}$. 

If $n \notin \mathbf{n^1}$, so $n$ never occurs in any $\mathbf{n_i}$, then we firstly note that since $k(n,n)<k(n,m)$ for any $m \neq n$, we have:
\begin{equation}
 k(n,n) - k(n,\mathbf{n^1_i})[k(\mathbf{n^1_i},\mathbf{n^1_i}))]^{-1}k(\mathbf{n^1_i},n)  > 0 \, . \nonumber
\end{equation}
This omits the term $\textrm{diag}((\boldsymbol{\sigma^1_i})^2$ from the expression for $\Psi_{\mathbf{n_i}}(n)$. However, if we denote $k'_j$ the matrix $k(\mathbf{n^1_i},\mathbf{n^1_i}) + \textrm{diag}((\boldsymbol{\sigma^1_i})^2)$ with the $j$th row replaced by $k(n,\mathbf{n^1_i})$, we have from Lemma~\ref{lem:matrix_identity}:
\begin{equation}
\frac{\partial}{\partial (\boldsymbol{\sigma^1_i})^2_j} \Psi_{\mathbf{n_i}}(n) = \frac{\left|k'_j\right|^2}{\left|k(\mathbf{n^1_i},\mathbf{n^1_i}) + \textrm{diag}((\boldsymbol{\sigma^1_i})^2)\right|^2}>0 \, , \nonumber
\end{equation}
for any element $(\boldsymbol{\sigma^1_i})^2_j$ of $(\boldsymbol{\sigma^1_i})^2$; hence $\Psi_{\mathbf{n_i}}(n)$ is increasing in any such element and its positivity follows. This completes the proof of the lemma.

\end{proof}

Now suppose that some $n \in \{1, \dots, N\}$ occurs only finitely often in $\mathbf{n^1}$. Then there must be some largest set $S_1$ of such $n$, with complement $S_2=\{1, \dots, N\} \setminus S_1$. 
Since every element in $S_1$ occurs in $\mathbf{n^1}$ with finite multiplicity there must be some $j$ such that no $n \in S_1$ occurs amongst the values $\{(\mathbf{n^1})_{j+1},(\mathbf{n^1})_{j+2},\dots\}$. But from Lemma~\ref{lem:lim_sup_ei}, there will almost surely eventually be some $J>j$ for which some value in $\{EI_{\mathbf{n_J}}(n):n \in S_1\}$ exceeds all values in  $\{EI_{\mathbf{n_J}}(n):n \in S_1\}$, and hence $(\mathbf{n^1})_{J+1} \in S_1$ (as long as $\tau$ is sufficiently small), contradicting the choice of $j$. So the event that an $n \in \{1,\dots, N\}$ occurs in $\mathbf{n^1}$ with finite multiplicity has probability 0. This completes the proof.

\end{proof}

\subsection{Repetitive Expected Improvement}

Typically expected improvement algorithms specify $\tau$ as a definitive stopping criterion when evaluations of the true function are noiseless \cite{brochu2010tutorial}. However, in the presence of noise the termination of algorithm~\ref{alg:emulation} may result in the selection of a hold-out set size whose cost evaluation has high sample variance. As shown by Theorem~\ref{thm:emulation_consistency} when $\tau=0$ this premature termination does not occur and we select the optimal hold-out set with certainty, but this is not practical as this would lead to an algorithm which does not terminate in finite time.

In order to mitigate this issue, we derive a further stopping criterion after the expected improvement algorithm has terminated. Namely, we set a threshold $\mathfrak{s}$ such that for $n^1_i$, 
\begin{equation}
    d^1_i - 3\sigma^1_i > d^{-} \cup \sigma^1_i < \mathfrak{s} \, . \nonumber
\end{equation} 
This ensures we have confidence in either the value of $d^1_i$ or confidence that further evaluation of $d^1_i$ will not result in $d^{-} = d^1_i$. If any $d^1_i \in \mathbf{d}^1$ do not meet this criteria, we evaluate these points again and restart the expected improvement algorithm. This process is detailed in algorithm~\ref{alg:meta_emulation}.

\begin{algorithm}[h]
\begin{algorithmic}
\State Run algorithm~\ref{alg:emulation} \;
\State Let $\mathfrak{n} = \{n^1_i \in \mathbf{n}^1: d^1_i - 3\sigma^1_i > d^{-} \cup \sigma^1_i < \mathfrak{s}\}$ \;
\If{$\mathfrak{n} \neq \emptyset$}
\For {$\tilde{n} \in \mathfrak{n}$}
\State Estimate $d(\tilde{n}) \approx k_2(\tilde{n})$ \;
\State $\mathbf{n} \gets (\mathbf{n} \cup \tilde{n})$; $\mathbf{d} \gets (\mathbf{d} \cup d\{\tilde{n})\}$ \;
\State Coalesce $\mathbf{n}, \mathbf{d}$ into $\mathbf{n^1}, \mathbf{d^1}$ and obtain $\boldsymbol{\sigma^1}$ \;
\EndFor
\State Re-estimate functions $\mu(n)$, $\Psi(n)$, $EI(n)$, with $\Theta = \Theta(\mathbf{n^1},\mathbf{d^1},\boldsymbol{\sigma^1})$ \;
\State Return to step 1 \;
\EndIf
\State \Return $n_*^{final} = \arg \min_{n_i \in \mathbf{n}^1} \left\{d^1_i\right\}$
\caption{Repetitive emulation OHS estimation; minimum cost improvement $\tau$ }
\label{alg:meta_emulation}
\end{algorithmic}
\end{algorithm}


\subsection{Extensions}

Various extensions of the emulator may improve our surrogate of the loss function, for example specifying priors on the parameters $\theta,\sigma^2_u, \zeta$ and using the likelihood provided by the Gaussian process to marginalize out these parameters. An explicit approach is given in~\cite{andrianakis2011parameter}, but under linearity assumptions which do not hold in our case, so analytic tractability would be lost. If we were able to cheaply estimate the derivative of the cost function at design points, this could be incorporated into our emulator~\citep{killeya2004thinking}, enabling greater posterior accuracy around these points. Direct estimation of gradients from only estimates of $\ell(n)$ usually requires double the number of evaluations as estimation of $\ell(n)$ values, and so has the potential to become a more costly procedure than the method presented in section~\ref{sec:emulation}.

\clearpage

\section{Simulations}

\subsection{Simulation of holdout, naive updating, and no-update strategies}
\label{supp_sec:holdout_dominance_simulation}

We simulated a population of $2 \times 10^5$ samples at 50 timepoints, with ten timepoints per epoch (time between updates). We considered a risk score on 22 `visible' covariates similar to those of the ASPRE score~\cite{rolnik17b} with true risk also depending on a `latent' covariate not included in the risk score. We designated the true risk function $f_t$ as a logistic model with coefficients varying continuously as a Gaussian process of $t${. We also used a logistic regression model for fitting all risk scores}. At each time point, we computed risk scores using each method, and made interventions on the 10\% of samples with highest predicted risk by reducing values of visible and latent covariates. We defined total cost as the sum of post-intervention risk across all samples. Hold-out sets were used in the final time-point of each epoch. 

{We included an `alternative' updating strategy in which we recorded a binary indicator of whether an individual sample underwent a risk-score guided intervention (which, as per Supplement~\ref{supp_sec:natural_holdout}, is not always possible). When making a treatment decision for an individual, we set the value of this treatment indicator as a constant value.}

\subsection{Optimal holdout set size arising from a simulated example}
\label{supp_sec:sim_example}

In this section, we analyse the dynamics of a roughly realistic, binary outcome system, subject to predictions from different families of risk models. Our main aim is to demonstrate the natural emergence of an optimal holdout set size from a reasonable setting. 

We generated datasets with a population size $N=5000$ with seven standard normally distributed covariates and outcomes $Y$ under a ground-truth logistic model, either with interaction terms (i.e., non-linear) or without (linear). We considered risk scores $\rho$ derived from either logistic regression models (not including interaction terms) or random forests. We designated cost functions $C_1$, $C_2$ to have value 0 for true-negatives, 0.5 for false- or true- positives, and 1 for false-negatives. 
%

Supplementary Figure~\ref{fig:example_sim} shows simulation results using either linear or logistic prediction models and linear or non-linear underlying models for $Y\mid X$.  We can observe that an optimal holdout set size can arise naturally from standard predictive models, since empirical $k_2$ curves for both a random forest and logistic regression satisfy assumptions~\ref{item:k2_decrease} and~\ref{item:k2_2nd_der} in the main manuscript. The optimal holdout set size occurs at a value $n$ smaller than that at which $k_2(n)$ is nearly `flat', indicating that unnecessarily large training sets are suboptimal. However, since $\ell(n)$ rises only linearly as $n$ increases, it is generally less costly to slightly overestimate rather than underestimate the optimal holdout set size. Finally, the rightmost panels illustrate that the optimal holdout set size is not necessarily smaller for a more accurate model: the random forest model (non-lin $\rho$) in the non-linear underlying case (right panels) leads to uniformly lower expected costs $k_2(n)$ at all potential holdout set sizes, although the optimal holdout set size is larger.

\clearpage

\section{Optimal holdout size in ASPRE}
\label{supp_sec:pre_estimate}

\subsection{Implementation}

We implemented the complete ASPRE model as described in~\cite{rolnik17}. We simulated a population of individuals with a similar distribution of ASPRE model covariates. We computed the ASPRE scores for our simulated individuals, and found a linear transformation of these scores such that, should the scores exactly specify the probability of PRE, the expected population prevalence and sensitivity of the score would match those reported in~\cite{rolnik17b}: prevalence $\pi_{PRE}$, and sensitivity amongst 10\% highest scores: 12.3\%. We then simulated PRE incidence according to these transformed scores.

We found that a generalised linear model with logistic link performed almost as well as the ASPRE score on our simulated data, so we used this model type to estimate the learning curve in the interests of simplicity. 

To choose values $\mathbf{n}$ and $\mathbf{k_2}/\mathbf{d}$, we initially chose a set $\mathbf{n}$ of 20 random values from $[500,30000]$. For each size $n$ in $\mathbf{n}$, we took a random sample of our data of size $n$, fitted a logistic model to that sample, and estimated corresponding expected costs per individual $\mathbf{k_2}$ as above. We fitted values $\theta=\theta(\mathbf{n},\mathbf{k_2})=(a,b,c)$ parametrising $k_2$ as the maximum-likelihood estimator of $\theta$ under the model
\begin{equation}
(\mathbf{k_2})_i \sim N\left(k_2((\mathbf{n})_i,\theta),\sigma^2\right) \sim N(a (\mathbf{n})_i^{-b} + c,\sigma^2) \nonumber
\end{equation}
for a fixed values $\sigma$, noting that the estimate of $\theta$ is independent of $\sigma$. For the parametric algorithm, we then set all values of $\boldsymbol{\sigma}$ to the same value, chosen empirically as the sample variance of
\begin{equation}
\mathbf{k_2} - k_2(\mathbf{n},\theta(\mathbf{n},\mathbf{k_2}))
\end{equation} \, . \label{eq:choose_sigma}
For the emulation algorithm, we set values $\mathbf{d}$ as 
\begin{equation}
\mathbf{d}_i=k_1 (\mathbf{n})_i + (\mathbf{k_2})_i(N-(\mathbf{n})_i) \, , \nonumber
\end{equation}
transforming values $\boldsymbol{\sigma}$ correspondingly for use in the emulation algorithm. We then sequentially chose 100 additional values $\mathbf{n}$ using both algorithm~\ref{alg:parametric} and~\ref{alg:emulation}, setting $\boldsymbol{\sigma}$ as the same value found in~\ref{eq:choose_sigma}. After choosing the 120 values of $\mathbf{n}$ using algorithm~\ref{alg:parametric}, we re-estimated $\mathbf{k_2}/\mathbf{d}$ for each of these values before estimating the OHS and confidence interval to avoid any potential regression-to-the mean effects from choosing next-values-of-$n$ so as to minimise estimated confidence interval width.

Our complete pipeline is available at~\url{https://github.com/jamesliley/OptHoldoutSize_pipelines}, and a comprehensive vignette is included in our R package \texttt{OptHoldoutSize} on CRAN and at~\url{https://github.com/jamesliley/OptHoldoutSize}.

\clearpage

\section{Supplementary figures}
\label{supp_sec:supplementary_figures}

\setcounter{figure}{0}
\renewcommand\thefigure{S10.\arabic{figure}}

\begin{figure}[!ht]
    \centering
    \includegraphics[width=\textwidth]{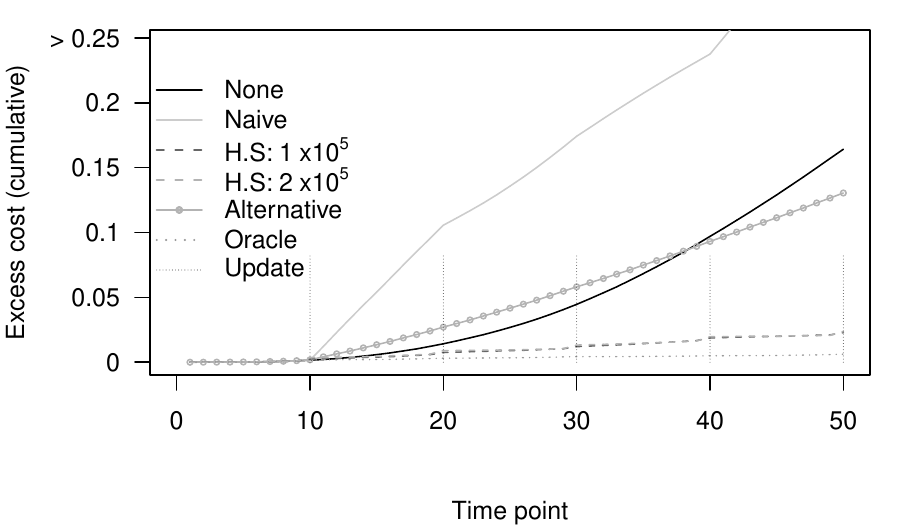}
    \caption{{Cumulative costs of updating strategies, defined analogously to figure~\ref{fig:holdout_dominance}}}
    \label{supp_fig:cumcost}
\end{figure}
\clearpage

\begin{figure}[!ht]
    \centering
    \includegraphics[width=\textwidth]{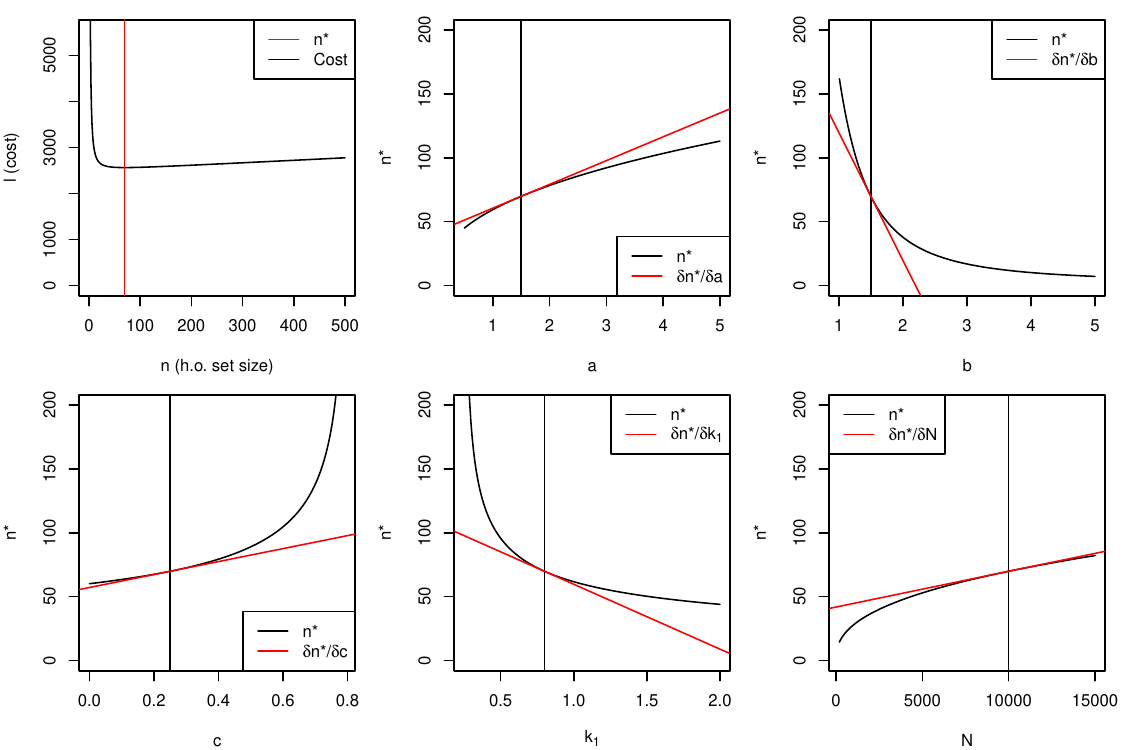}
    \caption{Dependence of optimal holdout set size on parameters of estimated learning curve ($a,b,c$, with $k_2(n;a,b,c)=a n^{-b}+c$), cost in intervention set $k_1$, and total number of samples $N$. Figures show change in optimal holdout set size $n_*$ while varying one parameter and holding others constant at $(a,b,c)=\left(\frac{3}{2},\frac{3}{2},\frac{1}{4}\right)$, $k_1=\frac{4}{5}$, $N=10^4$.}
    \label{supp_fig:partials_nstar}
\end{figure}
\clearpage

\begin{figure}[!ht]
    \centering
    \includegraphics[width=\textwidth]{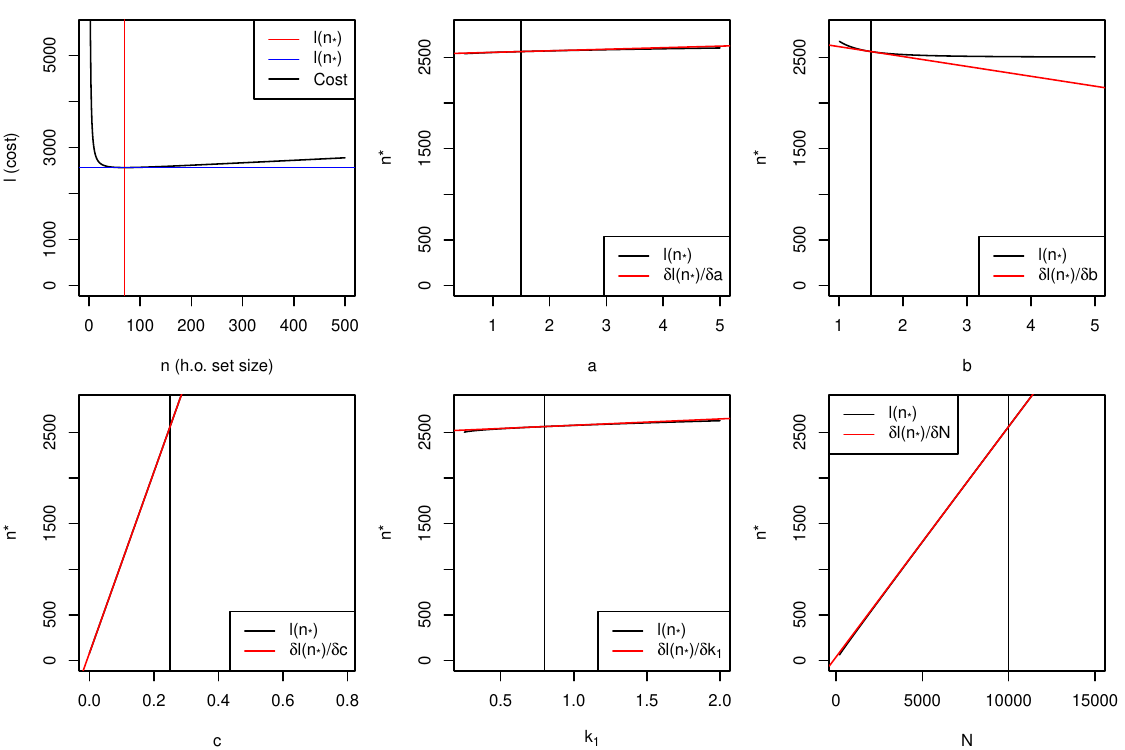}
    \caption{Dependence of minimum total cost on parameters of estimated learning curve ($a,b,c$, with $k_2(n;a,b,c)=a n^{-b}+c$), cost in intervention set $k_1$, and total number of samples $N$. Figures show change in minimal cost $\ell(n_*)$ while varying one parameter and holding others constant at $(a,b,c)=\left(\frac{3}{2},\frac{3}{2},\frac{1}{4}\right)$, $k_1=\frac{4}{5}$, $N=10^4$.}
    \label{supp_fig:partials_mincost}
\end{figure}

\clearpage

\begin{figure}[h] 
\centering
\begin{subfigure}{0.48\textwidth}
\includegraphics[width=\textwidth]{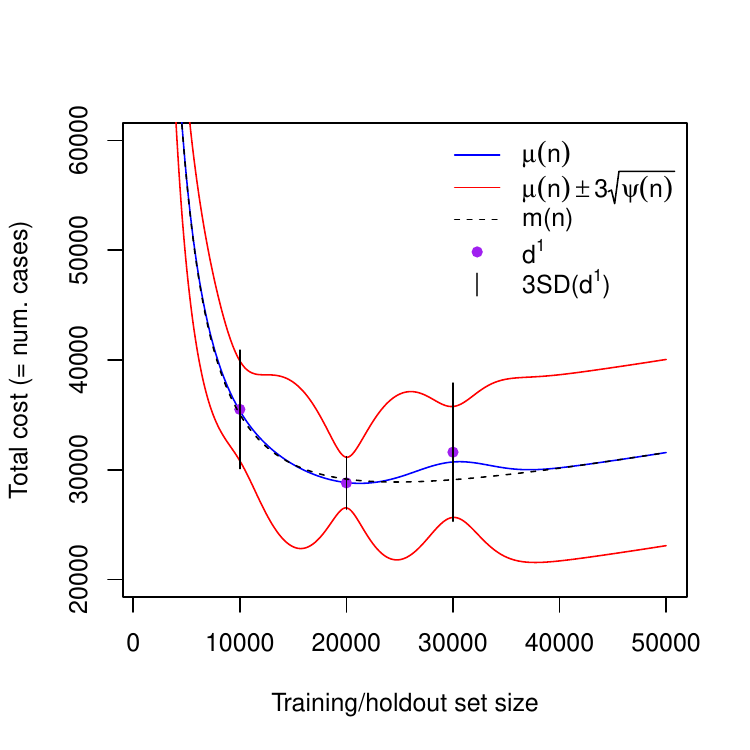}
\caption{}
\label{supp_fig:initialemulator}
\end{subfigure}
\begin{subfigure}{0.48\textwidth}
\includegraphics[width=\textwidth]{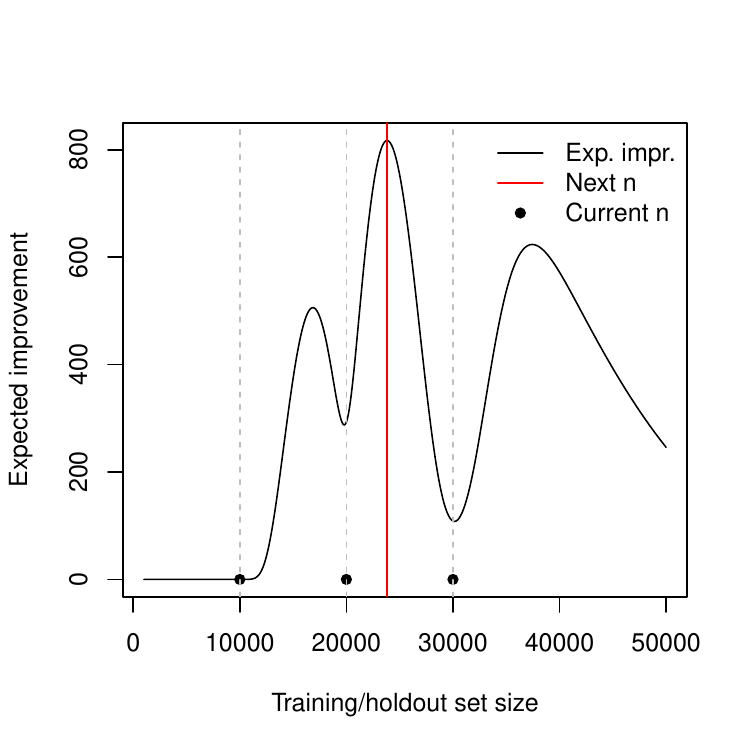}
\caption{}
\label{supp_fig:initialei}
\end{subfigure}
\caption{
Left panel shows emulator constructed using three $k_2()$ values (see pipelines). Function $m(n,\Theta)$ is constructed using $\theta$ derived from these three $k_2()$ estimates. Note reduced pointwise posterior variance at sample points. 
Rightmost panel shows expected improvement plot for the emulator constructed in panel\ref{supp_fig:initialemulator}. Note local minima at existing sample points.}
\label{supp_fig:emulator}
\end{figure}

\clearpage

\begin{figure} [h] 
\centering
  \begin{subfigure}{0.3\textwidth}
   \includegraphics[width=\textwidth]{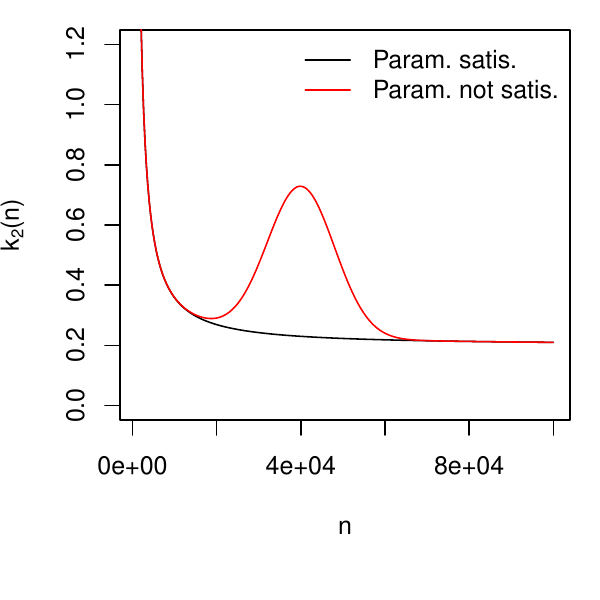}
   \caption{Versions of $k_2$}
   \label{supp_fig:k2_versions}
  \end{subfigure}
  \begin{subfigure}{0.3\textwidth}
   \includegraphics[width=\textwidth]{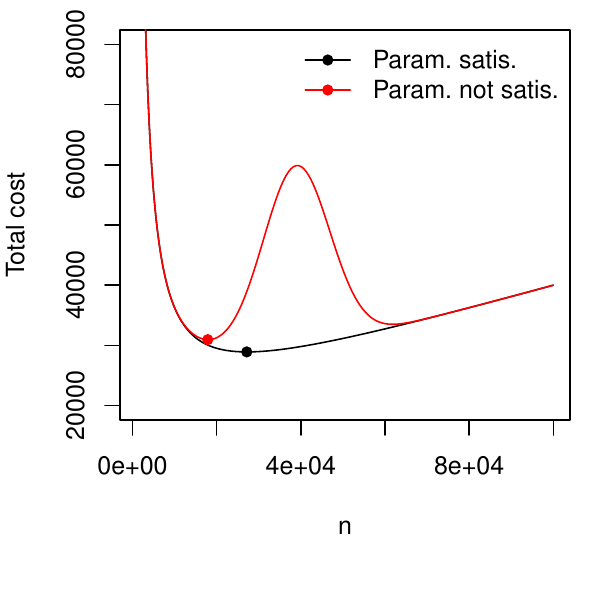}
   \caption{Costs associated with $k_2$}
   \label{supp_fig:cost_versions}
  \end{subfigure}
    \begin{subfigure}{0.3\textwidth}
   \includegraphics[width=\textwidth]{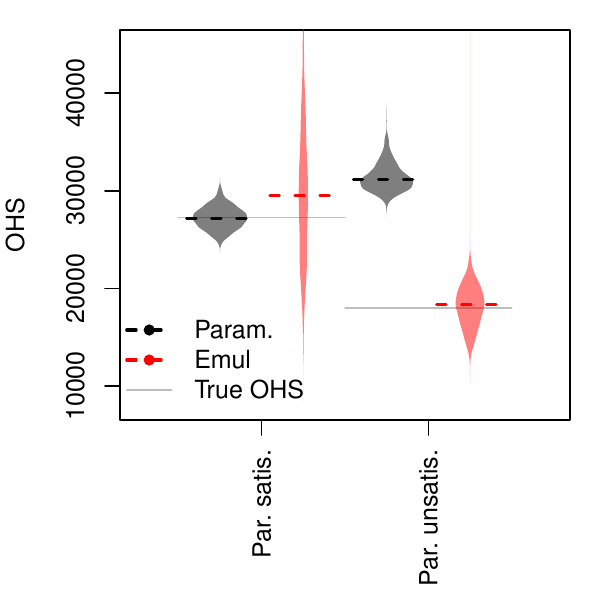}
   \caption{Distr/meds of est. optimal holdout set size}
   \label{supp_fig:param_emul_comp}
  \end{subfigure}
 \caption{Parametric and emulation algorithms with parametric assumptions satisfied or unsatisfied. OHS: optimal holdout set size}
\end{figure}
\clearpage

\begin{figure}[H] 
\centering
  \begin{subfigure}{\textwidth}
  \centering
  \begin{subfigure}{0.35\textwidth}
   \includegraphics[width=\textwidth]{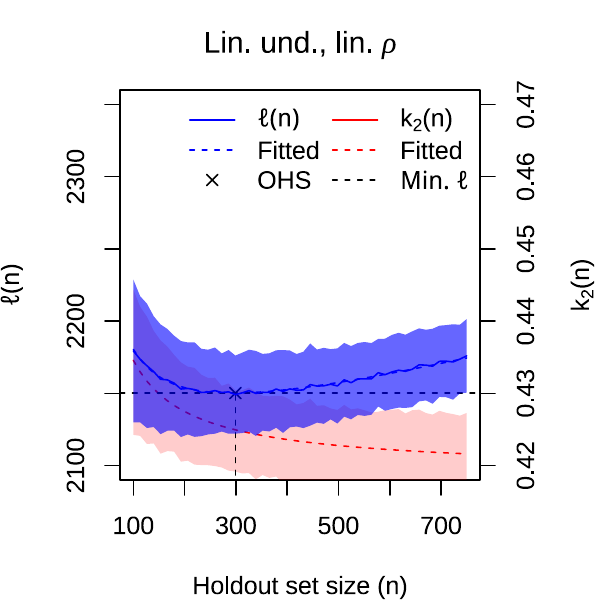}
  \end{subfigure}  
    \begin{subfigure}{0.35\textwidth}
   \includegraphics[width=\textwidth]{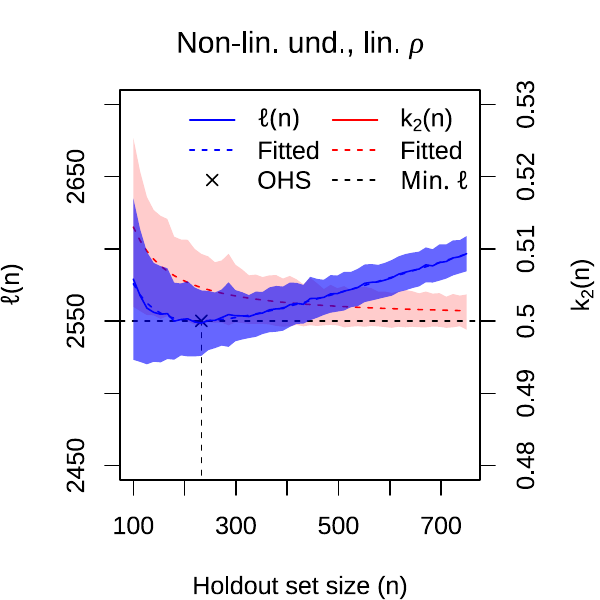}
  \end{subfigure}  
  \end{subfigure}
  \begin{subfigure}{\textwidth}
  \centering
    \begin{subfigure}{0.35\textwidth}
   \includegraphics[width=\textwidth]{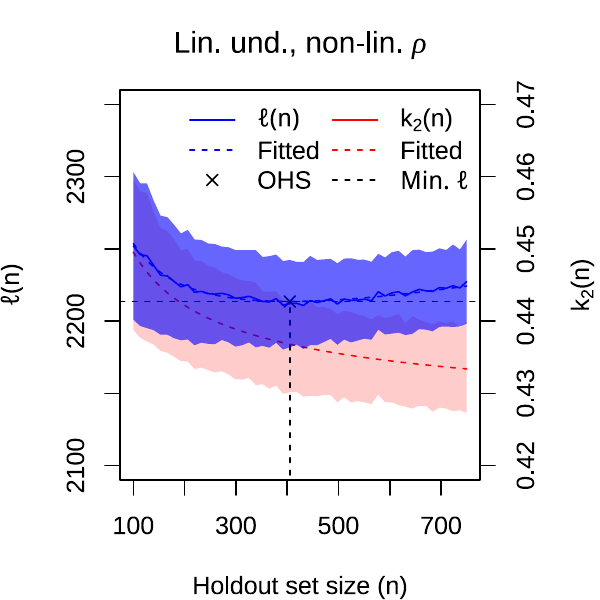}
  \end{subfigure}  
    \begin{subfigure}{0.35\textwidth}
   \includegraphics[width=\textwidth]{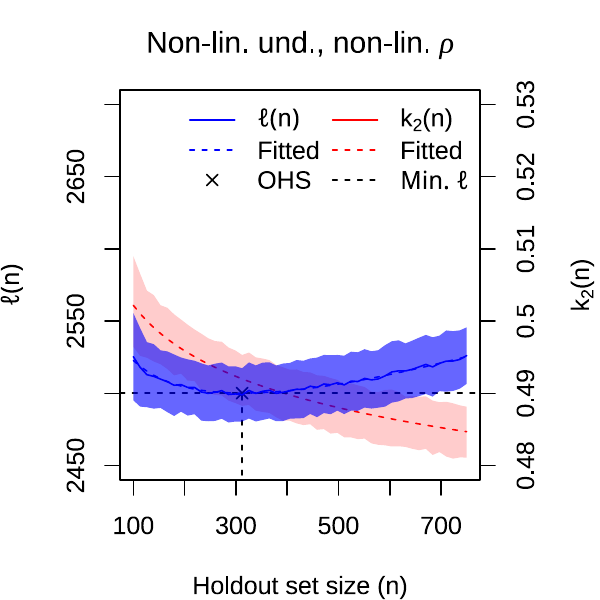}
  \end{subfigure}  
  \end{subfigure}
  
  \caption{Examples of cost functions as per Theorem~\ref{thm:ohs_exists} arising naturally from a basic risk score, with varying underlying model (und.), risk score type ($\rho$) and one point-wise standard deviation (shaded regions). The contributions of terms $k_1 n$ to $\ell(n)$ depend only on the underlying model and are the same in each column. OHS: optimal holdout set size}
  \label{fig:example_sim}
\end{figure}

\end{bibunit}

\end{document}